\pdfoutput=1
\documentclass{gOMS2earXiv}

\usepackage{setspace} 
\AtBeginDocument{%
  \addtolength\abovedisplayskip{-0.5\baselineskip}%
  \addtolength\belowdisplayskip{-0.5\baselineskip}%
}

\usepackage{graphicx}

\usepackage{subfigure}%

\theoremstyle{plain}%
\newtheorem{theorem}{Theorem}[section]
\newtheorem{corollary}[theorem]{Corollary}
\newtheorem{lemma}[theorem]{Lemma}

\theoremstyle{definition}

\theoremstyle{remark}
\newtheorem{remark}{Remark}

\usepackage{amsmath,amsfonts,amsthm,amssymb}
\usepackage{algpseudocode,algorithm}
 
\usepackage{hyperref}
\usepackage{xspace}

\usepackage[colorinlistoftodos,bordercolor=orange,backgroundcolor=orange!20,linecolor=orange,textsize=scriptsize]{todonotes}

\usepackage{subfigure}

\newcommand{\R}{{\mathbb R}}

\newcommand{\nBG}{G} %

\newcommand{\Gk}{\mathcal{G}_k}
\newcommand{\Gks}{\mathcal{G}^{\sigma'}_k\hspace{-0.08em}}

\newcommand{\Exp}{\mathbb{E}}

\newcommand{\eqdef}{:=}

\def\<#1,#2>{\left\langle #1, #2 \right\rangle}
\newcommand{\aggpar}{\nu}
\newcommand{\hk}{\vsub{\hv}{k}}
\newcommand{\Xk}{\vsub{X}{k}}
\newcommand{\alphak}{\vsub{\alphav}{k}}

 \newtheorem{assumption}[theorem]{Assumption}

\usepackage{color}

\definecolor{orange}{rgb}{1,0.5,0}

\newcommand{\mcx}[1]{{#1}}
\newcommand{\gs}[1]{{#1}}

\newcommand{\vsubset}[2]{#1_{[#2]}}
\newcommand{\vc}[2]{#1^{#2}}

\newcommand{\vsub}[2]{#1_{[#2]}}
\newcommand{\cocoa}{\textsc{CoCoA}\xspace} 
\newcommand{\cocoap}{\textsc{CoCoA$\!^{\bf \textbf{\footnotesize+}}$}\xspace}

\title{Distributed Optimization with Arbitrary Local Solvers}

\author{[names will come later]}

\newcommand{\hv}{\mathbf{h}}
\newcommand{\sv}{\mathbf{s}}
\newcommand{\uv}{\mathbf{u}}
\newcommand{\vv}{\mathbf{v}}
\newcommand{\wv}{\mathbf{w}}
\newcommand{\xv}{\mathbf{x}}
\newcommand{\alphav}{{\boldsymbol\alpha}}%

\newcommand\tagthis{\addtocounter{equation}{1}\tag{\theequation}}

\usepackage{rotating}

\newcommand{\bD}{D}

\newcommand{\bP}{\mathcal{P}}

\newcommand{\Ggk}{\mathcal{G}^{\sigma'}_k\hspace{-0.08em}}

\newcommand{\gap}{{Gap}}

\begin{document}

\title{\textit{Distributed Optimization with Arbitrary Local Solvers}  }

\author{
\name{Chenxin Ma\textsuperscript{a},
Jakub Kone\v{c}n\'y\textsuperscript{b},
Martin Jaggi\textsuperscript{c},
Virginia Smith\textsuperscript{d},
Michael I. Jordan\textsuperscript{d},\\
Peter Richt\'arik\textsuperscript{b}
and Martin Tak\'a\v{c}\textsuperscript{a}$^{\ast}$\thanks{$^\ast$Corresponding author. Email: \href{mailto:takac.mt@gmail.com}{takac.mt@gmail.com}.
This paper presents a substantial extension of an earlier conference article which appeared in ICML 2015 -- Proceedings of the 32th International Conference on Machine Learning, 2015 \cite{Ma:2015ti}. 
A conference article \cite{Smith:2015ua} exploring non-strongly convex regularizers is currently under review.
}}
\affil{
\textsuperscript{a}Lehigh University, 27 Memorial Dr W, Bethlehem, PA 18015, USA;
\textsuperscript{b}University of Edinburgh, Old College, South Bridge, Edinburgh EH8 9YL, United Kingdom;  
\textsuperscript{c}EPFL, Station 14, CH-1015 Lausanne, Switzerland;
\textsuperscript{d}UC Berkeley, 253 Cory Hall, Berkeley, CA 94720-1770, USA
}
\received{v2.0 released \today}
}

\maketitle

\begin{abstract} 
With the growth of data and necessity for distributed optimization methods, solvers that work well on a single machine must be re-designed to leverage distributed computation. Recent work in this area has been limited by focusing heavily on developing highly specific methods for the distributed environment. These special-purpose methods are often unable to fully leverage the competitive performance of their well-tuned and customized single machine counterparts. Further, they are unable to easily integrate improvements that continue to be made to single machine methods. To this end, we present a framework for distributed optimization that both allows the flexibility of arbitrary solvers to be used on each (single) machine locally, and yet maintains competitive performance against other state-of-the-art special-purpose distributed methods. We give strong primal-dual convergence rate guarantees for our framework that hold for arbitrary local solvers. We demonstrate the impact of local solver selection both 
theoretically and in an extensive experimental comparison. Finally, we provide thorough implementation details 
for our framework, highlighting areas for practical performance gains.\end{abstract}

\section{Motivation}
\label{sec:motivation}
Regression and classification techniques, represented in the general class of regularized loss minimization problems \cite{vapnik1998statistical}, are among the most central tools in modern big data analysis, machine learning, and signal processing.
For these tasks, much effort from both industry and academia has gone into the development of highly tuned and customized solvers.
However, with the massive growth of available datasets, major roadblocks still persist in the distributed setting, where data no longer fits in the memory of a single computer, and computation must be split across multiple machines in a network \cite{bottou2010large,marecek2014distributed,Ma:2015tt,
Balcan:2012tc,richtarik2013distributed,
Duchi:2013te,takavc2015distributed,
fercoq2014fast,chen2014large,ma2015linear,
Shamir:2014tp,DANE,ALPHA, 
zhang2015communication, konecny2015federated}.

On typical real-world systems, communicating data between machines is several orders of magnitude slower than reading data from main memory, e.g., when leveraging commodity hardware.
Therefore when trying to translate the highly tuned existing single machine solvers to the distributed setting, great care must be taken to avoid this significant communication bottleneck \cite{Jaggi:cocoa,Yang:2013vl}.

While several distributed solvers for the problems of interest have been recently developed, they are often unable to fully leverage the competitive performance of their highly tuned and customized single machine counterparts, which have already received much more research attention. 
More importantly, it is unfortunate that distributed solvers cannot automatically benefit from improvements made to the single machine solvers, and therefore are forced to lag behind the most recent developments.

In this paper we make a step towards resolving these issues by proposing a general communication-efficient distributed framework that can employ arbitrary single machine local solvers and thus directly leverage their benefits and problem-specific improvements.
Our framework works in rounds, where in each round the local solvers on each machine find a (possibly weak) solution to a specified subproblem of the same structure as the original master problem. On completion of each round, the partial updates between the machines are efficiently combined by leveraging the primal-dual structure of the global problem \cite{Yang:2013vl,Jaggi:cocoa,Ma:2015ti}. The framework therefore completely decouples the local solvers from the distributed communication. Through this decoupling, it is possible to balance communication and computation in the distributed setting, by controlling the desired accuracy and thus computational effort spent to determine each subproblem solution. Our framework holds with this abstraction even if the user wishes to use a different local solver on each machine.

\vspace{-1em}
\subsection{Contributions}

\vspace{0em}
 
\paragraph*{\it Reusability of Existing Local Solvers.}
The proposed framework allows for distributed optimization with the use of arbitrary local solvers on each machine. This abstraction makes the resulting framework highly flexible, and means that it can easily leverage the benefits of well-studied, problem-specific single machine solvers. In addition to increased flexibility and ease-of-use, this can result in large performance gains, as single machine solvers for the problems of interest have typically been thoroughly tuned for optimal performance. Moreover, any performance improvements that continue to be made to these local solvers will be automatically translated by the framework into the distributed setting.

\vspace{-1em}
\paragraph*{\it Adaptivity to Communication Cost.}
On real-world compute systems, the cost of communication versus computation typical varies by many orders of magnitude, from high-performance computing environments to very slow disk-based distributed workflow systems such as MapReduce/Hadoop. 
For optimization algorithms, it is thus essential to accommodate varying amounts of work performed locally per round, while still providing convergence guarantees.  Our framework provides exactly such control. 

\vspace{-1em}
\paragraph*{\it Strong Theoretical Guarantees.}
In this paper we extend and improve upon the \cocoa~\cite{Jaggi:cocoa} method. Our theoretical convergence rates apply to both smooth and non-smooth losses, and for both \cocoa as well as the more general \cocoap framework here. 
 The framework also exhibits favorable \emph{strong scaling} properties for the class of problems considered, as the number of machines~$K$ increases and the data size is kept fixed. More precisely, while the convergence rate of \cocoa degrades as $K$ is increased, the stronger theoretical convergence rate here is---in the worst case---\emph{independent} of $K$. 
Since the number of communicated vectors is only one per round and worker, 
this favorable scaling might be surprising. Indeed, for existing methods, splitting data among more machines generally
increases communication requirements \cite{Shamir:2014tp,Arjevani:2015vka}, which
can severely affect overall runtime.

\vspace{-1em}
\paragraph*{\it Primal-Dual Convergence.} We additionally strengthen the rates by showing stronger 
primal-dual convergence for both algorithmic frameworks, which are almost 
tight to their dual-only (or primal-only) counterparts.
Primal-dual rates for \cocoa had not previously been analyzed in the general convex case.
Our primal-dual rates allow efficient and practical certificates for the 
optimization quality, e.g., for stopping criteria. %

\vspace{-1em}
\paragraph*{\it Experimental Results.}
Finally, we give an extensive experimental comparison highlighting the impact of using various arbitrary solvers locally on each machine on several real-world, distributed datasets. We compare the performance of \cocoa and \cocoap across these datasets and choices of solvers, in particular illustrating the performance on a massive 280 GB dataset. Our code is available in an open source C\texttt{++} library, at: \url{https://github.com/optml/CoCoA}.

\vspace{-1em}
 
\subsection{Outline} The rest of the paper is organized as follows. Section~\ref{sec:problem} provides context and states the problem of interest, including necessary assumptions and their consequences. In Section~\ref{sec:algorithm} we formulate the algorithm in detail and explain how to implement it efficiently in practice. The main theoretical results are presented in Section~\ref{sec:result},
followed by discussion of relevant related work in Section~\ref{sec:relatedWork}.
 Practical experiments demonstrating the strength of the proposed framework are given in Section~\ref{sec:experiments}.
 Finally, we prove the main results in the appendix, in Section~\ref{sec:analysis}.

\vspace{-1em}

\section{Background and Problem Formulation}
\label{sec:problem}

To set our framework in context, we first state traditional complexity measures and convergence rates for single machine algorithms,  and then demonstrate that these must be adapted to more accurately represent the performance of an algorithm in the distributed setting.

When running an iterative optimization algorithm $\mathcal{A}$ on a single machine, its performance is typically measured by the total runtime:
\begin{equation}
\label{eq:runtime:default}
\tag{T-A}
\text{TIME}(\mathcal{A}) = \mathcal{I}_\mathcal{A} (\epsilon) \times \mathcal{T}_\mathcal{A} \, .
\end{equation}
Here, $\mathcal{T}_\mathcal{A}$ stands for the time it takes to perform a single iteration of algorithm $\mathcal{A}$, and $\mathcal{I}_\mathcal{A} (\epsilon)$ is the number of iterations  $\mathcal{A}$ needs to attain an $\epsilon$-accurate objective.\footnote{While for many algorithms the cost of a single iteration will vary throughout the iterative process, this simple model will suffice for our purpose to highlight the key issues associated with extending algorithms to a distributed framework.} 
 
\mcx{On a single machine, most of the state-of-the-art first-order optimization methods can achieve quick convergence in practice in terms of \eqref{eq:runtime:default} by performing a large amount of relatively fast iterations.} In the distributed setting, however, time to communicate between two machines can be several orders of magnitude slower than even a single iteration of such an algorithm. As a result, the overall time needed to perform a single iteration can increase significantly. 

Distributed timing can therefore be better illustrated using the following practical distributed efficiency model (see also \cite{marecek2014distributed}), where
\begin{equation}
\label{eq:runtime:distributed}
\tag{T-B}
\text{TIME}(\mathcal{A}) = \mathcal{I}_\mathcal{A} (\epsilon) \times \left( c + \mathcal{T}_\mathcal{A} \right).
\end{equation}
The extra term $c$ is the time required to perform one round of communication.\footnote{%
For simplicity, we assume here a fixed network architecture, and compare only to classes of algorithms that communicate a single vector in each iteration, rendering $c$ to be a constant, assuming we have a fixed number of machines. Most first-order algorithms would fall into this class.}%
 As a result, an algorithm that performs well in the setting of \eqref{eq:runtime:default} does not necessarily perform well in the distributed setting \eqref{eq:runtime:distributed}, especially when implemented in a straightforward or na\"{i}ve way. 
In particular, if $c \gg \mathcal{T}_\mathcal{A}$, we could intuitively expect less potential for improvement, as most of the time in the method will be spent on communication, not on actual computational effort to solve the problem. In this setting, novel optimization procedures are needed that carefully consider the amount of communication and the distribution of data across multiple machines.

One approach to this challenge is to design novel optimization algorithms from scratch, designed to be efficient in the distributed setting. This approach has one obvious practical drawback: 
There have been numerous highly efficient solvers developed, fine-tuned to  particular problems of interest, as long as the problem fits onto a single machine. These solvers are ideal if run on a single machine, but with the growth of data and necessity of data distribution, they must be re-designed to work in modern data regimes.

Recent work \cite{Jaggi:cocoa,Ma:2015ti,Yang:2013vl,Yang:2013ui,Yu:2012fp,Smith:2015ua} has attempted to address this issue by designing algorithms that reduce the communication bottleneck by allowing infrequent communication, while utilizing already existing algorithms as local sub-procedures. 
\mcx{The presented work here builds on the promising approach of \cite{Yang:2013vl,Jaggi:cocoa} in this direction.
See Section \ref{sec:relatedWork} for a detailed discussion of the related literature.
}%

The core idea is that one can formulate a local subproblem for each individual machine, and run an arbitrary local solver dependent only on local data for a number of iterations --- obtaining a partial local update. After each worker returns its partial update, a global update is formed by their aggregation.  

The big advantage of this is that companies and practitioners do not have to implement new algorithms that would be suitable for the distributed setting. We provide a way for them to utilize their existing algorithms that work on a single machine, and provide a novel communication protocol on top of this. 

In the original work on \cocoa \cite{Jaggi:cocoa}, authors provide convergence analysis only for the case when \mcx{the overall update is formed as an average of the partial updates}, and note that in practice it is possible to improve performance \mcx{by making a longer step in the same direction}. The main contribution of this work is a more general convergence analysis of various settings, which enables us to do better than averaging. In one case, we can even sum the partial updates \mcx{to obtain the overall update}, which yields the best result, both in theory and practice. We will see that this can result in significant performance gains, see also \cite{Ma:2015ti,Smith:2015ua}.

In the analysis, we will allow local solvers of arbitrarily weak accuracy, each working on its own subproblem which is defined in a completely data-local way for each machine. The relative accuracy obtained by each local solver will be denoted by $\Theta\in[0,1]$, where $\Theta=0$ describes an exact solution of the subproblem, and~$\Theta=1$ means that the local subproblem objective has not improved at all, for this call of the local solver. 
This paradigm results in a substantial change in how we analyze efficiency in the distributed setting. The formula practitioners are interested in minimizing thus changes to become:
\begin{equation}
\label{eq:runtime:cocoa}
\tag{T-C}
\text{TIME}(\mathcal{A}, \Theta) = \mathcal{I} (\epsilon, \Theta) \times \left( c + \mathcal{T}_\mathcal{A}(\Theta) \right).
\end{equation}
Here, the function $\mathcal{T}_\mathcal{A}(\Theta)$ represents the time the local algorithm $\mathcal{A}$ needs to obtain accuracy~$\Theta$ on the local subproblem.  
Note that the number of outer iterations~$\mathcal{I} (\epsilon, \Theta)$ is independent of choice of the inner algorithm $\mathcal{A}$, which will also be reflected by our convergence analysis presented in Section~\ref{sec:result}. Our convergence rates will hold for any local solver $\mathcal{A}$ achieving local quality $\Theta$.
For strongly convex problems, the general form will be 
$ \mathcal{I} (\epsilon, \Theta) = \frac{\mathcal{O}(\log(1/\epsilon))}{1 - \Theta} \, . $
The inverse dependence on $1 - \Theta$ suggests that there is a limit to how much we can gain by solving local subproblems to high accuracy --- $\Theta$ close to~$0$. There will always be order of $\log(1 / \epsilon)$ outer iterations needed. Hence, excessive local accuracy should not be necessary. On the other hand, if $\Theta \rightarrow 1$, meaning that the cost and quality of the local solver diminishes, then the number of rounds $\mathcal{I} (\epsilon, \Theta)$ will explode, which is to be expected.

To illustrate the strength of the paradigm \eqref{eq:runtime:cocoa} compared to \eqref{eq:runtime:distributed}, suppose we run gradient descent as local solver $\mathcal{A}$ for just a single iteration. It turns out that within our framework, choosing this local solver would lead to a method which is equivalent to running naively distributed gradient descent\footnote{Note that this is not obvious at this point. They are identical, subject to choice of local subproblems as specified in Section~\ref{sec:subproblem}.}. Running gradient descent for a single iteration would happen to attain a particular value of $\Theta$. Note that we typically do not set the value $\Theta$ explicitly. It is implicitly chosen by the number of iterations or stopping criterion specified by the user for the local solver. There is no reason to think that the value attained by single iteration of gradient descent would be optimal. For instance, it may be the case that running gradient descent for, say,  $200$ iterations, instead of 
just one, would give substantially better result in practice, due to better communication efficiency. These kinds of considerations are discussed at length in Section~\ref{sec:experiments}.

In general, one would intuitively expect that the optimal choice would be to have~$\Theta$ such that $\mathcal{T}_\mathcal{A}(\Theta) = \mathcal{O}(1) \times c$. In practice, however, the best strategy for any given local solver is to try several values for the number of local iterations to estimate the optimal choice. We will further discuss the importance of $\Theta$, both theoretically and empirically, later in Sections~\ref{sec:result} and \ref{sec:experiments}.
\vspace{-1mm}

\subsection{Problem Formulation}
\label{sec:problemformulation}
Let the training data $\{\xv_i \in \R^d, y_i \in \R \}_{i=1}^n$ be the set of input-output pairs, where $y_i$ can be real valued, but also from a discrete set in the case of classification problems.
We will assume without loss of generality  that $\forall i: \|\xv_i\|\leq 1$. Many common tasks in machine learning and signal processing can be cast as the following optimization problem:
\begin{equation}
\label{eq:primal}
\min_{\wv\in \R^d} \left\{ P(\wv) \eqdef \tfrac1n \textstyle{\sum}_{i=1}^n \ell_i( \xv_i^T \wv) + \lambda g(\wv) \right\},
\end{equation}
where $\ell_i$ is some convex loss function and $g$ is a regularizer. Note that $y_i$ is typically hidden in the formulation of functions $\ell_i$. Table \ref{tbl:differentLossFunctions} lists several common loss functions together with their convex conjugates $\ell^*_i$ \cite{ShalevShwartz:2014dy}.

\begin{table} 
\tbl{Examples of commonly used loss functions.}
{\begin{tabular}{llll}
\toprule
\multicolumn{1}{c}{Loss function} 
  & \multicolumn{1}{c}{ $\ell_i(a)$}  & \multicolumn{1}{c}{$\ell^*_i(b)$}
  & \multicolumn{1}{c}{Property  of $\ell$}
   \\ \colrule
Quadratic loss  & $\frac12 (a-y_i)^2$ &  $\frac12 b^2 + y_ib $
& Smooth 
 \\
Hinge loss & $\max\{0, y_i-a\}$ &  $y_i b, \quad b\in[-1,0] $ %
 & Continuous \\
Squared hinge loss  & $(\max\{0, y_i-a\})^2$ & $\frac{b^2}{4}, \quad b\in[-\infty,0] $ %
 & Smooth \\
Logistic loss  & $\log(1+\exp{(-y_ia)})$ & $-\frac{b}{y_i}\log \left(-\frac{b}{y_i}\right) + \left( 1 +\frac{b}{y_i} \right) \log \left( 1 + \frac{b}{y_i} \right) $ 
& Smooth \\
\botrule 
\end{tabular}
}
\label{tbl:differentLossFunctions}
\end{table}

The dual optimization problem for formulation \eqref{eq:primal} -- as a special case of Fenchel duality -- can be written as follows \cite{yuan2012recent,ShalevShwartz:2014dy}:
\begin{equation}
\label{eq:dual}
\max_{\alphav \in \R^n}
 \left\{
 D(\alphav )\eqdef  
 \tfrac1n \left(\textstyle{\sum}_{i=1}^n -\ell_i^*(- \alpha_i)\right)
 - \lambda g^* \left( \tfrac1{\lambda n} X \alphav \right) \right\},
\end{equation}
where $X = [\xv_1, \xv_2, \dots, \xv_n] \in \R^{d \times n}$, and $\ell_i^*$ and $g^*$ are the convex conjugate functions of $\ell_i$ and $g$, respectively. The convex (Fenchel) conjugate of a function $\phi : \R^k \rightarrow \R$ is defined as the function $\phi^* : \R^k \rightarrow \R$, with $\phi^*(\uv) := \sup_{\sv\in\R^k} \{ \sv^T \uv - \phi(\sv) \} $.

For simplicity throughout the paper, let us denote 
\begin{equation}
\label{eq:fRdefinition}
f(\alphav) \eqdef  \lambda g^*\left( \tfrac{1}{\lambda n} X \alphav \right) 
\qquad \text{and} \qquad R(\alphav) \eqdef \tfrac1n \textstyle{\sum}_{i=1}^n \ell_i^*(- \alpha_i)\, ,
\end{equation}
such that
$
D(\alphav) \overset{\eqref{eq:dual}+\eqref{eq:fRdefinition}}{=} - f(\alphav) - R(\alphav).
$

It is well known \cite{QUARTZ,takavc2013mini,ShalevShwartz:2014dy,Dunner:2016vga} that the first-order optimality conditions give rise to \mcx{a natural mapping that relates pairs of primal and dual variables. The mapping employs the linear map given by the data $X$, and maps any dual variable $\alphav \in\R^n$ to a primal candidate vector $\wv \in\R^d$ as follows:}
$$ \wv(\alphav) :=  \nabla g^*( \vv(\alphav) ) 
= \nabla g^*\left( \tfrac1{\lambda n} X \alphav \right),$$
where we denote
$
\vv(\alphav) := \tfrac1{\lambda n} X \alphav \, .
$

For this mapping, under the assumptions that we make in Section~\ref{sec:techassump} below, it holds that if $\alphav^\star$ is an optimal solution of \eqref{eq:dual}, then $\wv(\alphav^\star)$ is an optimal solution of~\eqref{eq:primal}. In particular, {\em strong duality} holds between the primal and dual problems. If we define the duality gap function as
\begin{align}
\label{eq:gap}
\gap(\alphav)
 := \bP(\wv(\alphav))-\bD(\alphav),
\end{align}
then $\gap(\alphav^\star)=0$, 
which ensures that by solving the dual problem \eqref{eq:dual} we also solve the original primal problem of interest~\eqref{eq:primal}. As we will later see, there are many benefits to leveraging this primal-dual relationship, including the ability to use the duality gap as a certificate of solution quality, and, in the distributed setting, as a tool through which we can effectively distribute computation.

\vspace{-1em}
\paragraph*{\it Notation.}

We assume that to solve problem \eqref{eq:dual} we have a network of $K$ machines at our disposal.
The data $\{\xv_i,y_i\}_{i=1}^n$ is residing on the $K$ machines in a distributed way, with every machine only holding a part of the whole dataset.
In the same way we split the dual variables $\alpha_i$, with each corresponding to an individual data point~$\xv_i$. The given data distribution is described using a partition $\mathcal{P}_1, \dots, \mathcal{P}_K$ that corresponds to the indices of data and dual variables residing on machine~$k$. Formally, $\mathcal{P}_k \subseteq \{1, 2, \dots, n\}$ for each $k$,  $\mathcal{P}_k \cap \mathcal{P}_l = \emptyset$ whenever $k \neq l$, and $\bigcup_{k = 1}^K \mathcal{P}_k = \{1, 2, \dots, n\}$.

In order to efficiently use this structure in the text, we introduce the following notation. For any $\hv \in \R^n$, let $\hk$ be the vector in $\R^n$ defined in such a way that
$ (\hk)_i = h_i$ if $i \in \mathcal{P}_k$ and $0$ otherwise.
Note that, in particular, 
$
\hv = \textstyle{\sum}_{k = 1}^ K   \hk.
$
Analogously, we write~$\Xk$ for the matrix consisting only of the columns $i \in \mathcal{P}_k$, padded with zeros in all other columns.

\subsection{Technical Assumptions}\label{sec:techassump}
Here we first state the properties and assumptions used throughout the paper.
We assume that for all $i\in \{1,\dots,n\}$, the function 
$\ell_i$ is convex, i.e.,
$\forall \lambda \in [0, 1]$ and $\forall x,y \in \R$ we have 
$\ell_i(\lambda x + (1 - \lambda) y) \leq \lambda \ell_i(x) + (1 - \lambda) \ell_i(y) \, .$
 
We also assume the function  
$g$ to be $1$-strongly convex, i.e., for all $\wv, \uv \in \R^d$ it holds that
$
g(\wv +  \uv) \geq g(\wv) + \< \nabla g(\wv), \uv > + \tfrac{1}{2} \| \uv \|^2,
$
where $\nabla g(\wv)$ is any subgradient\footnote{A subgradient of a convex function $\phi$ in a point $\xv' \in \R^d$ is defined as any $\xi \in \R^d$ satisfying for all $\xv \in \R^d$, $\phi(\xv) \geq \phi(\xv') + \< \xi, \xv - \xv' >$.} of the function $g$.
Here, $\| \cdot \|$ denotes the standard Euclidean norm.

Note that we use subgradients in the definition of strong convexity. This is due to the fact that while we will need the function $g$ to be strongly convex in our analysis, we do not require smoothness. An example used in practice is $g(\wv) = \| \wv \|^2 + \lambda' \| \wv \|_1$ for some $\lambda' \in \R$. Also note that in the problem formulation \eqref{eq:primal} we have a regularization parameter~$\lambda$, which controls the strong convexity parameter of the entire second term. Hence, fixing the strong convexity parameter of $g$ to $1$ is not restrictive in this regard. For instance, this setting has been used previously in~\cite{ShalevShwartz:2014dy,QUARTZ,csiba2015stochastic}.

The following assumptions state properties of the functions $\ell_i$, which we use only in certain results in the paper. We always explicitly state when we require each assumption.

\begin{assumption}[$(1/\gamma)$-Smoothness]
\label{ass:1gammasmooth}
Functions $\ell_i: \R \rightarrow \R$ are $1/\gamma$-smooth, if $\forall i \in \{ 1, \dots, n \}$ and $\forall x, h \in \R$ it holds that
\begin{equation}
\label{def:Lsmoothness}
\ell_i(x + h) \leq \ell_i(x) + h \nabla \ell_i(x) + \tfrac{1}{2 \gamma} h^2,
\end{equation}
where $\nabla \ell_i(x)$ denotes the gradient of the function $\ell_i$.
\end{assumption}
\begin{assumption}[$L$-Lipschitz Continuity]
\label{ass:LLip}
Functions $\ell_i: \R \rightarrow \R$ are $L$-Lipschitz continuous, if $\forall i \in \{ 1, \dots, n \}$ and $\forall x, h \in \R$ it holds that
\begin{equation}
\label{def:LLip}
|\ell_i(x + h) - \ell_i(x)| \leq L |h|.
\end{equation}
\end{assumption}

\begin{remark}
As a consequence of having $1/\gamma$-smoothness of $\ell_i$ and $1$-strong convexity of $g$, we have that the functions $\ell_i^*(\cdot)$ are $\gamma$-strongly convex and $g^*(\cdot)$ is $1$-smooth~\cite{rockafellar2015convex}. These are the properties we will ultimately use as we will be solving the dual problem~\eqref{eq:dual}. Note that $1$-smoothness of $g^* : \R^d \rightarrow \R$ means that for all~$\xv, \hv \in \R^d$,
\begin{equation}
\label{def:Lsmoothness:gstar}
g^*(\xv + \hv) \leq g^*(\xv) + \< \nabla g^*(\xv), \hv > + \tfrac{1}{2} \| \hv \|^2.
\end{equation}
\end{remark}
Following lemma, a consequence of 1-smoothness of $g^*$ and the definition of $f$,  will be crucial in deriving a meaningful local subproblem for the proposed distributed framework.
\begin{lemma}
\label{lem:quartz}
Let $f$ be defined in \eqref{eq:fRdefinition}. Then for all $\alphav, \hv \in \R^n$ we have
\begin{equation}
\label{eq:quartz}
f(\alphav + \hv) \leq f(\alphav) + \< \nabla f(\alphav), \hv> + \tfrac{1}{2 \lambda n^2} \hv^T X^T X \hv.
\end{equation}
\end{lemma}

\begin{remark} Note that although the above inequality appears as a consequence of the problem structure~\eqref{eq:dual} and of strong convexity of $g$, there are other ways to satisfy it. Hence, our dual analysis holds for all optimization problems of the form $\max_{\alphav} D(\alphav)$, where $D(\alphav) = -f(\alphav) - R(\alphav)$, and where $f$ satisfies inequality \eqref{eq:quartz}. However, for the  duality gap  analysis we naturally do require \eqref{eq:quartz} that the dual problem arises from the primal problem with $g$ being strongly convex.
\end{remark}

\vspace{-1em}
\section{The Framework}
\vspace{-1em}
\label{sec:algorithm}

In this section we start by giving a general view of the proposed framework, explaining the most important concepts needed to make the framework efficient. In Section~\ref{sec:subproblem} we discuss the formulation of the local subproblems, and in Section~\ref{sec:implementationnotes} specific details and best practices for implementation.

The data distribution plays a crucial role in Algorithm~\ref{alg:cocoa}, where in each outer iteration indexed by $t$, machine $k$ runs an arbitrary local solver on a problem described only by the data that particular machine owns and other fixed constants or linear functions. 

The crucial property is that the optimization algorithm on machine $k$ changes only coordinates of the dual optimization variable $\alphav^t$ corresponding to the partition~$\mathcal{P}_k$ to obtain an approximate solution to the local subproblem. We will formally specify this in Assumption~\ref{ass:localImprovement}. After each such step, updates from all machines are aggregated to form a new iterate $\alphav^{t+1}$. The aggregation parameter $\aggpar$ will typically be between $\aggpar = 1/K$, corresponding to averaging, and $\aggpar = 1$, to adding. 

\begin{algorithm} 
\caption{Improved \textsc{CoCoA}\texttt{+} Framework} %
\label{alg:cocoa}
\begin{algorithmic}[1]
\State {\bf Input:} starting point $\alphav^0 \in \R^{n}$, aggregation parameter $\aggpar \in (0,1]$, data partition $\{\mathcal{P}_k\}_{k=1}^K$
\For {$t = 0, 1, 2, \dots $}
  \For {$k \in \{1,2,\dots,K\}$ {\bf in parallel over machines}}
     \State Let $\hk^t$ be an approximate solution of the local problem \eqref{eq:subproblem:sigma1}, i.e.\vspace{-2mm}
     \[
     \max_{ \hk \in \R^n } \Gk(\hk; \alphav^t)
     \vspace{-2mm}
     \]
  \EndFor
  \State Set $\alphav^{t+1} := \alphav^t + \aggpar \sum_{k=1}^K \hk^t$
\EndFor  
\end{algorithmic}
\end{algorithm}

Here we list the core conceptual properties of Algorithm~\ref{alg:cocoa}, which are important qualities that allow it to run efficiently.  
\begin{description}

\item [{Locality.}] The local subproblem $\Gk$ \eqref{eq:subproblem:sigma1} is defined purely based on the data points residing on machine $k$, as well as a single shared vector in $\R^d$ (representing the state of the $\alphav^t$ variables of the other machines). Each local solver can then run independently and in parallel, i.e., there is no need for communication while solving the local subproblems.
\item [{Local changes.}] The optimization algorithm used to solve the local subproblem $\Gk$ outputs a vector $\hk^t$ with nonzero elements only in coordinates corresponding to variables $\alphak$ stored locally (i.e., $i\in\mathcal{P}_k$). %
\item [{Efficient maintenance.}] Given the description of the local problem $\Gk(\,\cdot\,; \alphav^t)$ at time $t$, the new local problem $\Gk(\,\cdot\,; \alphav^{t+1})$ at time $t+1$ can be formed on each machine, requiring only communication of a single vector in $\R^d$ from each machine $k$ to the master node, and vice versa, back to each machine $k$.
 
\end{description}

Let us now comment on these properties in more detail. Locality is important for making the method versatile, and is the way we escape the restricted setting described by \eqref{eq:runtime:distributed} that allows us much greater flexibility in designing the overall optimization scheme. %
Local changes result from the fact that along with data, we distribute also coordinates of the dual variable $\alphav$ in the same way, and thus only make updates to the coordinates stored locally.
As we will see, efficient maintenance of the subproblems can be obtained. For this, a  communication efficient encoding of the current shared state $\alphav$ is necessary. To this goal, we will in Section~\ref{sec:implementationnotes} show that communication of a single $d$-dimensional vector is enough to formulate the subproblems \eqref{eq:subproblem:sigma1} in each round, by carefully exploiting  their partly separable structure.

Note that Algorithm~\ref{alg:cocoa} is the ``analysis friendly'' formulation of our algorithm framework, and it is not yet fully illustrative for implementation purposes. In Section~\ref{sec:implementationnotes} we will precisely formulate the actual communication scheme, and illustrate how the above properties can be achieved.

Before that, we will formulate the precise subproblem $\Gk$ in the following section.

\subsection{The Local Subproblems}
\label{sec:subproblem}

We can define a data-local subproblem of the original dual optimization problem~\eqref{eq:dual}, which can be solved on machine $k$ and only requires accessing data which is already available locally, i.e., datapoints with $i\in\mathcal{P}_k$. More formally, each machine $k$ is assigned the following local subproblem, depending only on the previous shared primal vector~$\wv\in\R^d$, and the change in the local dual variables~$\alpha_i$ with $i\in\mathcal{P}_k$:
\begin{equation} 
\max_{\hk\in\R^{n}} \Gks(  \hk; \alphav).
\end{equation}
We are now ready to define the local objective $\Gks(\,\cdot\,; \alphav)$ as follows:
\begin{align}
\label{eq:subproblem:sigma1}
\tag{LO}
\Gks(\hk; \alphav) :=  -\tfrac{1}K f(\alphav) - \< \nabla f(\alphav), \hk > - \tfrac{\lambda \sigma'}{2}  \left\| \tfrac{1}{\lambda n}\Xk \hk \right\|^2 %
 - R_k\!\left( \alphak + \hk \right),
\end{align}
where $ R_k(\alphak) \eqdef  \frac1n \sum_{i \in \mathcal{P}_k} \ell_i^*(-\alphav_i) $.
The role of the parameter $\sigma'\geq 1$ is to measure the ``difficulty'' of the data partition, in a sense which we will discuss in detail in Section \ref{sec:choiceofsigma} below.

The interpretation of the above defined subproblems is that they will form a quadratic approximation of the smooth part of the true objective $D$, which becomes separable over the machines. The approximation keeps the non-smooth $R$ part intact.
The variable $\hk$ expresses the update proposed by machine $k$. In this spirit, note also that the approximation coincides with $D$ at the reference point $\alphav$, i.e. $\sum_{k=1}^K \Gks( {\bf 0}; \alphav) = D(\alphav)$.
We will discuss the interpretation and properties of these subproblems in more detail below in Section \ref{sec:choiceofsigma}.

\vspace{-1em}

\subsection{Practical Communication Efficient Implementation}
\label{sec:implementationnotes}

We will now discuss how Algorithm~\ref{alg:cocoa} can efficiently be implemented in a distributed environment. Most importantly, it remains to clarify how the ``local'' subproblems can actually be formulated and solved by using only local information from the corresponding machine, and to make precise what information needs to be communicated in each round.

Recall that the local subproblem objective $\Gks(\,\cdot\,; \alphav)$ was defined in \eqref{eq:subproblem:sigma1}.
We will now equivalently rewrite this optimization problem, to clarify how it is expressed only using \emph{local} information. To do so, we use our simplifying notation~$\vv = \vv(\alphav) := \tfrac1{\lambda n} X \alphav$ for given~$\alphav$. As we see in the reformulation, it is precisely this vector $\vv\in\R^d$ which contains all the necessary shared information between the machines. Given the vector~$\vv$, the subproblem \eqref{eq:subproblem:sigma1} writes equivalently as
\begin{align}
\label{eq:subproblemPr}
\tag{LO'}
\Gks(\hk; \vv, \alphak) &:=  %
 -\tfrac{\lambda}K g^*(\vv) 
 - \langle \tfrac{1}{n} \Xk^T \nabla g^*(\vv), \hk \rangle - \tfrac{\lambda}{2} \sigma' \left\| \tfrac{1}{\lambda n}\Xk \hk \right\|^2 %
 \\&\qquad  - R_k\!\left( \alphak + \hk \right). \notag
\end{align} 
Here for the reformulation of the gradient term, we have simply used the chain rule on the objective $f$ (recall the definition $f(\alphav) \eqdef  \lambda g^*( \vv )$), giving
$$
\vsub{\nabla f(\alphav)}{k} = \tfrac{1}{n} \Xk^T \nabla g^* ( \vv ) .
$$

\paragraph*{\it Practical Distributed Framework.}
In summary, we have seen that each machine can formulate the local subproblem given purely local information (the local data $\Xk$ as well as the local dual variables $\alphak$). No information about the other machines variables $\alphav$ or their data is necessary.

The only requirement for the method to work is that between the rounds, the changes in the $\alphak$ variables on each machine and the resulting global change in~$\vv$ are kept consistent, in the sense that $\vv^t = \vv(\alphav^t) := \tfrac1{\lambda n} X \alphav^t$ must always hold. Note that for the evaluation of $\nabla g^*(\vv)$, the vector $\vv$ is all that is needed. In practice, $g$ as well as its conjugate $g^*$ are simple vector valued regularization functions, the most prominent example being $g(\vv) = g^*(\vv) = \frac12 \| \vv \|^2$.

In the following more detailed formulation of the \cocoap framework shown in Algorithm~\ref{alg:cocoaPractical} (equivalent reformulation of Algorithm~\ref{alg:cocoa}), the crucial communication pattern of the framework finally becomes more clear: Per round, \emph{only a single vector} (the update on $\vv\in\R^d$) needs to be sent over the communication network. 
The reduce-all operation in line 10 means that each machine sends their vector $\Delta \vv_k^t\in\R^d$ to the network, which performs the addition operation of the $K$ vectors to the old~$\vv^t$. The resulting vector $\vv^{t+1}$ is then communicated back to all machines, so that all  have the same copy of $\vv^{t+1}$ before the beginning of the next round.

The framework as shown below in Algorithm~\ref{alg:cocoaPractical} clearly maintains the consistency of $\alphav^t$ and $\vv^t = \vv^t(\alphav^t)$ after each round, no matter which local solver is used to approximately solve \eqref{eq:subproblemPr}. A diagram illustrating the communication and computation involved in the first two full iterations of Algorithm~\ref{alg:cocoaPractical} is given in Figure~\ref{fig:diagram}.

\begin{algorithm} 
\caption{Improved \textsc{CoCoA}\texttt{+} Framework, Practical Implementation}
\label{alg:cocoaPractical}
\begin{algorithmic}[1]
\State {\bf Input:} starting point $\alphav^0 \in \R^n$, aggregation parameter $\aggpar \in (0,1]$, data partition $\{\mathcal{P}_k\}_{k=1}^K$
\State $\vv^0 := \frac1{\lambda n} X \alphav^0  \in \R^d$
\For {$t = 0, 1, 2, \dots $}
  \For {$k \in \{1,2,\dots,K\}$ {\bf in parallel over machines}}
	\State Precompute $ \Xk^T \nabla g^*(\vv^t)$
     \State Let $\hk^t$ be an approximate solution of the local problem \eqref{eq:subproblemPr}, i.e.\vspace{-2mm}
\[
\max_{ \hk \in \R^n } \Gks(\hk; \vv^t, \alphak^t)  %
\vspace{-6mm}
\]  
	\Comment{{\it computation}}
	\State Update local variables $\alphak^{t+1} := \alphak^t + \aggpar \hk^t$
  	
	\State Let $\Delta \vv_k^t := \frac{1}{\lambda n} \Xk \hk^t$
  \EndFor %
  \State \textbf{reduce all} to compute
	$
	\vv^{t+1} := \vv^t + \aggpar \sum_{k=1}^K \Delta \vv_k^t
	$
 \Comment{{\it communication}}
\EndFor  
\end{algorithmic}
\end{algorithm}

\begin{figure}[t]
\centering
\includegraphics[width=\linewidth,trim={10 60 20 100}]{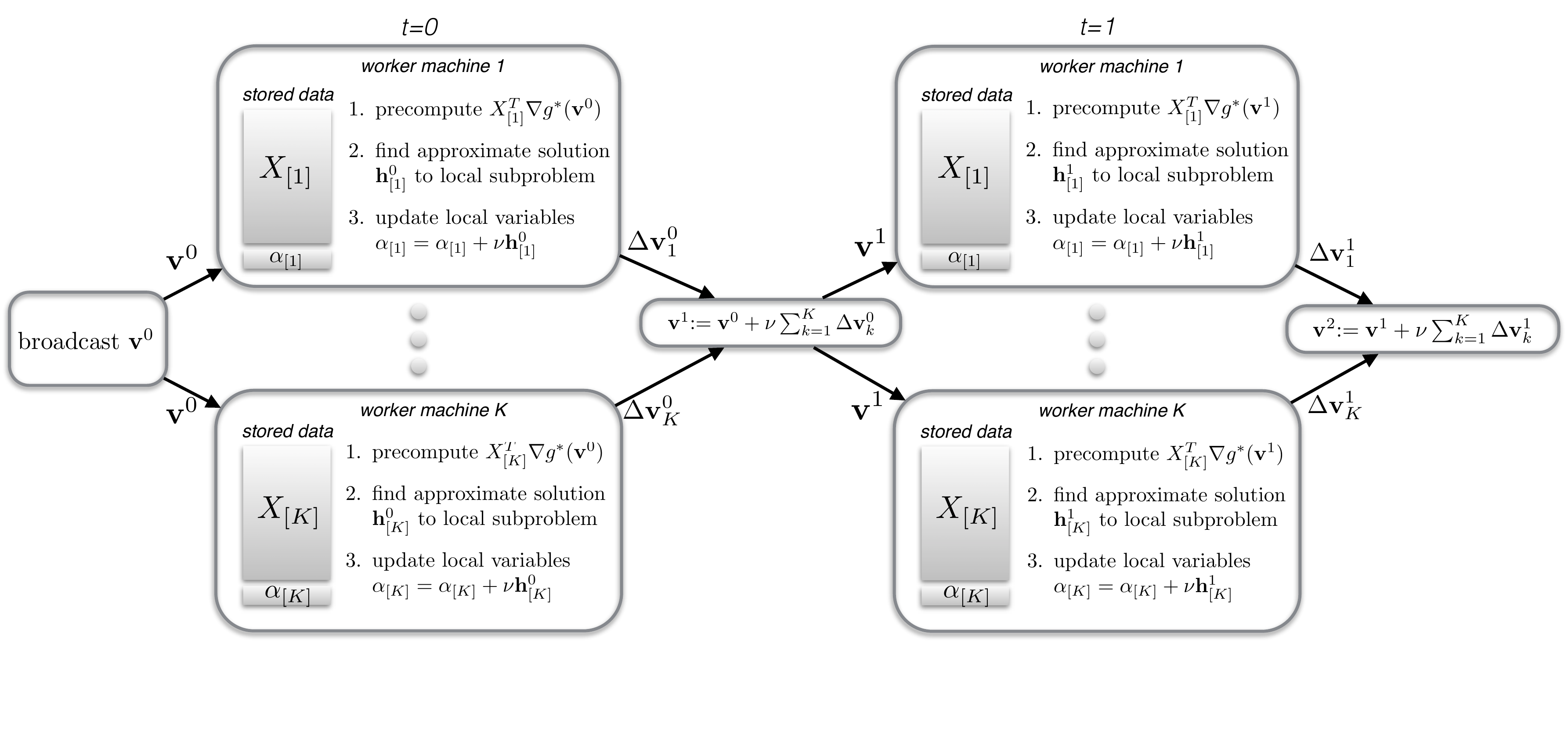}
\caption{The first two iterations of the improved framework (practical implementation).}
\label{fig:diagram}
\end{figure}

\vspace{-1em} 
\subsection{Compatibility of the Subproblems for Aggregating Updates}
\label{sec:choiceofsigma}
In this subsection, we shed more light on the local subproblems on each machine, as defined in \eqref{eq:subproblem:sigma1} above, and their interpretation.
More formally, we will show how the aggregation parameters $\nu$ (controlling the level of adding versus averaging the resulting updates from each machine) and $\sigma'$ (the subproblem parameter) interplay together, to in each round achieve a valid approximation to the global objective function~$D$.

The role of the subproblem parameter~$\sigma'$ is to measure the difficulty of the given data partition. For the convergence results discussed below to hold, $\sigma'$ must be chosen not smaller than
\begin{equation}
\label{eq:sigmaPrimeSafeDefinition}
\sigma'
\geq
\sigma'_{min}
 \eqdef
 \nu \cdot
 \max_{\hv\in \R^n}
 \big\{
 \hv^T X^T X \hv \ \big|\ \hv^T \nBG \hv \le 1\big\} \, . \vspace{-1mm}
\end{equation}

Here, $\nBG$ is the block diagonal submatrix of the data covariance matrix $X^T X$, corresponding to the partition $\{\mathcal{P}_k\}_{k=1}^K$, i.e.,
\begin{equation}
\label{eq:nBGDefinition}
\nBG_{ij} \eqdef
\begin{cases}
	\xv_i^T \xv_j = (X^TX)_{ij}, & \mbox{if}\ \exists k \ \mbox{such that} \ i,j \in \mathcal{P}_k, \\ 
	0,& \mbox{otherwise}. 
\end{cases}
\end{equation}

In this notation, it is easy to see that the crucial quantity defining $\sigma'_{min}$ above is written as $\hv^T \nBG \hv = \textstyle\sum_{k=1}^K \|X_{[k]} \hk\|^2$.

The following lemma shows that if the aggregation and subproblem parameters
$\nu$ and~$\sigma'$ satisfy \eqref{eq:sigmaPrimeSafeDefinition}, then the sum of the subproblems $\sum_k \Gks$ will closely approximate the global objective function $D$. More precisely, this sum is a block-separable lower bound on $D$.

\begin{lemma}
\label{lem:step:sigma1}
\label{lowerboundOnD}
\label{lem:RelationOfDTOSubproblems}
Let $\sigma' \geq 1$ and
$\nu \in [0, 1]$ %
satisfy \eqref{eq:sigmaPrimeSafeDefinition} (that is $\sigma' \geq \sigma'_{min}$).
Then
$\forall \alphav, \hv  \in \R^n$, it holds that
\begin{equation}
\label{eq:asfdjalkfjlsaflasdfa}
D\big(\alphav+\nu \textstyle{\sum}_{k=1}^K \hk\big)
\geq (1-\nu) D(\alphav) +
\nu  \textstyle{\sum}_{k=1}^K \Gks(\hk; \alphav),
\end{equation}
\end{lemma}

The following lemma gives a simple choice for the subproblem parameter $\sigma'$, which is trivial to calculate for all values of the aggregation parameter $\aggpar\in\R$, and safe in the sense of the desired condition \eqref{eq:sigmaPrimeSafeDefinition} above. 
Later we will show experimentally (Section \ref{sec:experiments}) that the choice of this safe upper bound for $\sigma'$ only has a minimal effect on the overall performance of the algorithm. 

\begin{lemma}\label{lem:sigmaPrimeNotBad}
For any %
aggregation parameter $\aggpar\in[0,1]$, the choice of the subproblem parameter $\sigma' := \aggpar K$ is valid for \eqref{eq:sigmaPrimeSafeDefinition}, i.e.,
$
\aggpar K
\geq
\sigma'_{min}. 
$
\end{lemma}

\section{Main Results}
\label{sec:result}

In this section we state the main theoretical results of this paper. Before doing so, we elaborate on one of the most important aspects of the algorithmic framework: the \emph{quality of approximate local solutions}. 

\subsection{Quality of Local Solutions}

The notion of approximation quality provided by the local solvers is measured according to the following:
\begin{assumption}[Quality of local solution]
\label{ass:localImprovement}
Let $\Theta \in [0, 1)$ and $\alphav \in \R^n$ be fixed, and let $\hk^\star$ be the optimal solution of a local subproblem $\Gk(\,\cdot\,; \alphav)$.
We assume the local optimization procedure run on every node $k \in [K]$, in each iteration $t$ produces a (possibly random) output $\hk$ satisfying
\begin{equation}
\label{eq:localQualityOfImprovement}
\Exp \left[ \Gk(\hk^\star; \alphav) - \Gk(\hk; \alphav) \right] \leq \Theta \left[ \Gk(\hk^\star; \alphav) - \Gk({\bf 0}; \alphav) \right].
\end{equation}
\end{assumption}

The assumption specifies the (relative) accuracy $\Theta$ obtained on solving the local subproblem $\Gk$. Considering the two extreme examples, setting $\Theta = 0$ would require to find the exact maximum, while $\Theta = 1$ states that no improvement was achieved at all by the local solver. Intuitively, we would prefer $\Theta$ to be small, but spending many computational resources to drive $\Theta$ to $0$ can be excessive in practice, since~$\Gk$ is actually not the problem we are interested in solving \eqref{eq:dual}, but is the problem to be solved per communication round.
The best choice in practice will therefore be to choose $\Theta$ such that the local solver runs for a time comparable to the time it takes for a single communication round. This freedom of choice of $\Theta \in [0,1]$ is a crucial property of our proposed framework, allowing it to adapt to the full range of communication speeds on real world systems, ranging from supercomputers on one extreme to very slow communication rounds like MapReduce systems on the other extreme.

In Section \ref{sec:experiments} we study impact of different values of this parameter to the overall performance on solving \eqref{eq:dual}.

\subsection{Complexity Bounds}

Now we are ready to state the main results.
Theorem \ref{thm:mainResult} covers the case
when $\forall i$ the loss function 
$\ell_i$ is $1/\gamma$ smooth
and Theorem \ref{thm:mainResult:gcc}
covers the case
when 
$\ell_i$ is   $L$-Lipschitz continuous. For simplicity in the rates, we define the following two quantities:
$$
\forall k: \sigma_k \eqdef
 \max_{\vsubset{\alphav}{k} \in \R^n}
 \tfrac{\|\Xk \vsubset{\alphav}{k}\|^2}{
 \|\vsubset{\alphav}{k}\|^2}
 \qquad\mbox{and}\qquad
 \sigma \eqdef \textstyle{\sum} _{k=1}^K 
\sigma_k  |\mathcal{P}_k|.
$$

\begin{theorem}[Smooth loss functions]
\label{thm:convergenceSmoothCase}
\label{thm:mainResult}
Assume the loss functions functions 
$\ell_i$ are $(1/\gamma)$-smooth $\forall i\in[n]$.
We define $\sigma_{\max} = 
\max_{k\in[K]} \sigma_k$. Then after $T$ iterations of Algorithm \ref{alg:cocoaPractical}, with  
$$
 T
    \geq 
\tfrac{1}
   {\aggpar
(1-\Theta)}
\tfrac
{\lambda\gamma n+
\sigma_{\max} \sigma'}
{ \lambda\gamma n }
    \log \tfrac1{\epsilon_\bD} , \vspace{-1mm}
$$
it holds that\vspace{-3mm}
$$\Exp[\bD(\alphav^\star)
  - \bD(\vc{\alphav}{T})]
   \leq \epsilon_\bD.$$
Furthermore, after $T$ iterations with\vspace{-1mm}
\begin{equation}
\label{afdsafdafsafdsafda}
 T 
    \geq 
\tfrac{1}
   {\aggpar
(1-\Theta)}
\tfrac
{\lambda\gamma n+
\sigma_{\max} \sigma'}
{ \lambda\gamma n }
    \log 
\left(
\tfrac{1}
   {\aggpar
(1-\Theta)}
\tfrac
{\lambda\gamma n+
\sigma_{\max} \sigma'}
{ \lambda\gamma n }
    \tfrac1{\epsilon_\gap}
    \right)  ,
\end{equation}
we have the expected duality gap
$$
\Exp[
\bP( \wv(\vc{\alphav}{T})) - \bD(\vc{\alphav}{T})
]\leq \epsilon_\gap.
$$
\end{theorem}

\begin{theorem}[Lipschitz continuous loss functions]
\label{thm:convergenceNonsmooth}
 \label{thm:mainResult:gcc}
Consider Algorithm \ref{alg:cocoaPractical} with Assumption \ref{ass:localImprovement}. 
Let $\ell_i(\cdot)$ be $L$-Lipschitz continuous,
and $\epsilon_\gap$ $>0$ be the desired duality gap (and hence an upper-bound on primal sub-optimality).
Then after $T$ iterations, where
\begin{align}\label{eq:dualityRequirements}
T
&\geq
T_0 + 
\max\left\{\left\lceil \tfrac1{\aggpar (1-\Theta)}\right\rceil,\tfrac
{4L^2  \sigma   \sigma'}
{\lambda n^2 \epsilon_\gap
\aggpar (1-\Theta)}\right\},  
\\
T_0
&\geq t_0+
 \max\left\{0,
\tfrac{2}{ \aggpar (1-\Theta) }
\left(
\tfrac
{8L^2  \sigma   \sigma'}
{\lambda n^2 \epsilon_\gap}
-1
\right)
\right\},\notag
\\
t_0 &\geq 
  \max\left\{0,\Big\lceil \tfrac1{\aggpar (1-\Theta)}
\log(
\tfrac{
 2\lambda n^2 (\bD(\alphav^\star )-\bD(\vc{\alphav}{0}))
  }{4 L^2 \sigma \sigma'}
  )
 \Big\rceil\right\},\notag
\end{align}
we have that the expected duality gap satisfies
\[
\Exp[\bP( \wv(\overline\alphav)) - \bD(\overline \alphav) ] \leq \epsilon_\gap,
\]
at the averaged iterate
\begin{equation}\label{eq:averageOfAlphaDefinition}
\overline \alphav: = \tfrac1{T-T_0}\textstyle{\sum}_{t=T_0+1}^{T-1} \vc{\alphav}{t}. 
\end{equation}

\end{theorem}

The most important observation regarding the above result is that we do not impose any assumption on the choice of the local solver, apart from sufficient decrease condition on the local objective in Assumption~\ref{ass:localImprovement}.

Let us now comment on the leading terms of the complexity results.
The inverse dependence on $1 - \Theta$ suggests that it is worth pushing the rate of local accuracy $\Theta$ down to zero. However, when thinking about overall complexity, we have to bear in mind that achieving high accuracy on the local subproblems might be too expensive. The optimal choice would depend on the time we estimate a round of communication would take. In general, if communication is slow, it would be worth spending more time on solving local subproblems, but not so much if communication is relatively fast. We discussed this tradeoff in Section~\ref{sec:problem}.

We achieve a significant speedup by replacing the slow averaging aggregation (as in~\cite{Jaggi:cocoa}) by more aggressive adding instead, that is $\aggpar = 1$ instead of  $\aggpar = 1/K$.
Note that the safe subproblem parameter for the averaging case ($\aggpar = 1/K$) is $\sigma' := 1$, while for adding ($\aggpar = 1$) it is given by $\sigma' := K$,
both proven in Lemma~\ref{lem:sigmaPrimeNotBad}.
The resulting speedup from more aggressive adding is strongly reflected in the resulting convergence rate as shown above, when plugging in the actual parameter values $\aggpar$ and $\sigma'$ for the two cases, as we will illustrate more clearly in the next subsection.

\subsection{Discussion and Interpretations of Convergence Results}

As the above theorems suggest, it is not possible to meaningfully change the aggregation parameter $\aggpar$ in isolation. It comes naturally coupled with a particular subproblem.

In this section, we explain a simple way to be able to have the aggregation parameter $\aggpar = 1$, that is to aggressively add up the updates from each machine. The motivation for this comes from a common practical setting. When solving the SVM dual (Hinge loss: $\ell_i(a) = \max\{0, y_i-a\}$), the optimization problem comes with ``box constraints'', i.e., for all $i \in \{ 1, \dots, n \}$, we have $\alpha_i \in [0, 1]$ (see Table \ref{tbl:differentLossFunctions}). The particular values of $\alpha_i$ being~$0$ or~$1$ have a particular interpretation in the context of original problem~\eqref{eq:primal}. If we used $\aggpar < 1$, we would never be able reach the upper boundary of any variable $\alpha_i$, when starting the algorithm at $\mathbf{0}$.
This example illustrates some of the downsides of averaging vs adding updates, coming from the fact that the step-size from using averaging (by being $1/K$ times shorter) can result in $1/K$ times slower convergence.

For the case of aggressive adding, the convergence theorems for local objective~\eqref{eq:subproblem:sigma1} results derived in  Theorems \ref{thm:mainResult} and \ref{thm:mainResult:gcc}
are as follows:
\begin{corollary}[Smooth loss functions - adding]
\label{thm:mainResult:adding}
Let the assumptions of Theorem \ref{thm:mainResult} be satisfied.
If we run Algorithm~\ref{alg:cocoa} with $\aggpar = 1, \sigma'=K$
for  
\begin{equation}
\label{asfdafdafa}
 T 
   \overset{\eqref{afdsafdafsafdsafda}}{=} 
\tfrac{1}
   {
1-\Theta}
\tfrac
{\lambda\gamma n+
\sigma_{\max} K}
{ \lambda\gamma n }
    \log 
\left(
\tfrac{1}
   {1-\Theta}
\tfrac
{\lambda\gamma n+
\sigma_{\max} K}
{ \lambda\gamma n }
    \tfrac1{\epsilon_\gap}
    \right)
\end{equation}
iterations, we have 
$\Exp[
\bP( \wv(\vc{\alphav}{T})) - \bD(\vc{\alphav}{T})
]\leq \epsilon_\gap.$
\end{corollary}
On the other hand, if we would just average results (as proposed in \cite{Jaggi:cocoa}), we would obtain following corollary:
\begin{corollary}[Smooth loss functions - averaging] 
Let the assumptions of Theorem~\ref{thm:mainResult} be satisfied.
If we run Algorithm~\ref{alg:cocoa} with $\aggpar = 1/K, \sigma'=1$
for  
 \begin{equation}
 \label{asfdafdafa2}
 T 
    \overset{\eqref{afdsafdafsafdsafda}}{\geq }
\tfrac{1}
   {
1-\Theta}
\tfrac
{K\lambda\gamma n+
\sigma_{\max} K}
{ \lambda\gamma n }
    \log 
\left(
\tfrac{1}
   {1-\Theta}
\tfrac
{K \lambda\gamma n+
\sigma_{\max} K}
{ \lambda\gamma n }
    \tfrac1{\epsilon_\gap}
    \right)  
\end{equation} 
 iterations, we have 
$\Exp[
\bP( \wv(\vc{\alphav}{T})) - \bD(\vc{\alphav}{T})
]\leq \epsilon_\gap.$
\end{corollary}
Comparing the leading terms in 
Equations \eqref{asfdafdafa} and \eqref{asfdafdafa2}
we see that 
the leading term for the $\nu=1$ choice is
$\mathcal{O}(\lambda\gamma n+
\sigma_{\max} K)$,
which is always better than for the $\nu=1/K$ case, when the leading term is 
$\mathcal{O}(\lambda\gamma n+
\sigma_{\max} K)$.
This strongly suggests that adding in Framework \ref{alg:cocoaPractical} is preferable, especially when 
$\lambda \gamma n \gg \sigma_{\max} $.
 
An analogously significant improvement by an order of $K$ factor follows for the case of the sub-linear convergence rate for general Lipschitz loss functions, as shown in Theorem~\ref{thm:mainResult:gcc}.

Note that the differences in the convergence rate are bigger for relatively big values of the regularizer $\lambda$. When the regularizer is $\mathcal{O}(1 / n)$, the difference is negligible. This behaviour is also present in practice, as we will point out in Section~\ref{sec:experiments}.

\section{Discussion and Related Work}
\label{sec:relatedWork}

In this section, we review a number of methods designed to solve optimization problems of the form of interest here, which are typically referred to as regularized empirical risk minimization (ERM) in the machine learning literature.
Formally described in Section~\ref{sec:problemformulation}, this problem class \eqref{eq:primal} underlies many prominent methods of supervised machine learning.

\paragraph*{\it Single-Machine Solvers.}
Stochastic Gradient Descent (SGD) is the simplest stochastic method one can use to solve the problem of structure \eqref{eq:primal}, and dates back to the work of Robbins and Monro \cite{Robbins:1951ko}. We refer the reader to \cite{moulines2011non, needell2014stochastic, nemirovski2009robust, bottou2012stochastic} for recent theoretical and practical assessment of SGD. Generally speaking, the method is extremely easy to implement, and converges to modest accuracies very quickly, which is often satisfactory in applications in machine learning. On the other hand, difficulty in choosing hyper-parameters make the method sometimes rather cumbersome, and is impractical if higher solution accuracy is needed.

The current state of the art for empirical loss minimization with strongly convex regularizers is randomized coordinate ascent on the dual objective --- Stochastic Dual Coordinate Ascent (SDCA) \cite{ShalevShwartz:2013wl}. In contrast to primal SGD methods, the SDCA algorithm family is often preferred as it is free of learning-rate parameters, and has faster (geometric) convergence guarantees.  This algorithm and its variants are increasingly used in practice \cite{Wright:2015bn,ShalevShwartz:2014dy}. On the other hand, primal-only methods apply to a larger problem class, not only of form \eqref{eq:primal} that enables formation of dual problem \eqref{eq:dual} as considered here.

Another class of algorithms gaining attention in recent very few years are `variance reduced' modifications of the original SGD algorithm. 
They are applied directly to the primal problem \eqref{eq:primal}, but unlike SGD, have property that variance of estimates of the gradients tend to zero as they approach optimal solution. 
Algorithms such as SAG \cite{schmidt2013minimizing}, SAGA \cite{defazio2014saga} and others \cite{shalev2015sdcaWODual, Defazio:2014vx} come at the cost of extra memory requirements --- they have to store a gradient for each training example. 
This can be addressed efficiently in the case of generalized linear models, but prohibits its use in more complicated models such as in deep learning. 
On the other hand, Stochastic Variance Reduced Gradient (SVRG) and its variants \cite{johnson2013accelerating, konecny2013semi, xiao2014proximal, konecny2015mini, nitanda2014stochastic} are often interpreted as `memory-free' methods with variance reduction. 
However, these methods need to compute the full gradient occasionally to drive the variance reduction, which requires a full pass through the data and is an operation one generally tries to avoid. 
This and several other practical issues have been recently addressed in~\cite{babanezhad2015stop}.
Finally, another class of extensions to SGD are stochastic quasi-Newton methods~\cite{bordes2009sgd, byrd2014stochastic}. Despite their clear potential, a lack of theoretical understanding and complicated implementation issues compared to those above may still limit their adoption in the wider community. A stochastic dual Newton ascent (SDNA) method was proposed and analyzed in~\cite{SDNA}. However, the method needs to modified substantially before it can be implemented in a  distributed environment.

\paragraph*{\it SGD-based Algorithms.}
For the empirical loss minimization problems of interest, stochastic subgradient descent (SGD) based methods are well-established.
Several distributed variants of SGD have been proposed, many of which build on the idea of a parameter server \cite{Niu:2011wo, richtarik2013distributed,  Duchi:2013te}.
Despite their simplicity and accessibility in terms of implementation, the downside of this approach is that the amount of required communication is equal to the amount of data read locally, \gs{since one data point is accessed per machine per round} (e.g., mini-batch SGD with a batch size of 1 per worker). These variants are in practice not competitive with the more communication-efficient methods considered in this work, which allow more local updates per communication round.

\paragraph*{\it One-Shot Communication Schemes.}
At the other extreme, there are distributed methods using only a single round of communication, such as \cite{Zhang:2013wq, Zinkevich:2010tj,Mann:2009tr,McWilliams:2014tl,Heinze:2016tu}.
These methods require additional assumptions on the partitioning of the data, which are usually not satisfied in practice if the data are distributed ``as is'', i.e., if we do not have the opportunity to distribute the data in a specific way beforehand. 
Furthermore, some cannot guarantee convergence rates beyond what could be achieved if we ignored data residing on all but a single computer, as shown in \cite{DANE}.
Additional relevant lower bounds on the minimum number of communication rounds necessary for a given approximation quality are presented in \cite{Balcan:2012tc,Arjevani:2015vka}.

\paragraph*{\it Mini-Batch Methods.} Mini-batch methods (which instead of just one data-example use updates from several examples per iteration) are more flexible and lie within these two communication vs. computation extremes. However,
mini-batch versions of both SGD and coordinate descent (CD) \cite{PCDM, richtarik2013distributed,ShalevShwartz:2014dy, marecek2014distributed, Yang:2013vl, Tappenden:2015vh, ALPHA, QUARTZ, csiba2015primal, csiba2016importanceminibatch} suffer from their convergence rate degrading towards the rate of batch gradient descent as the size of the mini-batch is increased. 
This follows because mini-batch updates are made based on the outdated previous parameter vector $\wv$, in contrast to methods that allow immediate local updates like \cocoa.

Another disadvantage of mini-batch methods is that the aggregation parameter is harder to tune, as it can lie anywhere in the order of mini-batch size. The optimal choice is often either unknown, or difficult to compute.
In the \cocoa setting, the parameter lies in the typically much smaller range given by $K$. In this work the aggregation parameter is further simplified and can be simply set to $1$, i.e., adding updates, which is achieved by formulating a more conservative local problem as described in Section~\ref{sec:subproblem}.

\paragraph*{\it Distributed Batch Solvers.}
With traditional batch gradient solvers not being competitive for the problem class  \eqref{eq:primal}, improved batch methods have also received much research attention recently, in the single machine case as well as in the distributed setting.
In distributed environments, often used methods are the alternating direction method of multipliers (ADMM) \cite{Boyd:2010bw}  
as well as quasi-Newton methods such as L-BFGS, which can be attractive because of their relatively low communication requirements. Namely, communication is in the order of a constant number of vectors (the batch gradient information) per full pass through the data.

ADMM also comes with an additional penalty parameter balancing between the equality constraint on the primal variable vector $\wv$ and the original optimization objective~\cite{Boyd:2010bw}, which is typically hard to tune in many applications. Nevertheless, the method has been used for distributed SVM training in, e.g., \cite{Forero:2010vv}. The known convergence rates for ADMM are weaker than the more problem-tailored methods mentioned we study here, and the choice of the penalty parameter is often unclear in practice. 

Standard ADMM and quasi-Newton methods  do not allow a gradual trade-off between communication and computation available here. An exception is the approach of Zhang, Lee and Shin \cite{Zhang:2012usa}, which is similar to our approach in spirit, albeit based on ADMM, in that they allow for the subproblems to be solved inexactly. However, this work focuses on L2-regularized problems and a few selected loss functions, and offers no complexity results.

Interestingly, our proposed \cocoap framework here -- despite clearly aimed at  cheap stochastic local solvers -- does have similarities to block-wise variants of batch proximal methods as well, as explained as follows:

The purpose of our subproblems as defined in \eqref{eq:subproblem:sigma1} is to form a data-dependent block-separable quadratic approximation to the smooth part of the original (dual) objective~\eqref{eq:dual}, while leaving the non-smooth part $R$ intact (recall that $R(\alphav)$ was defined to collect the~$\ell^*_i$ functions, and is separable over the coordinate blocks).
Now if hypothetically each of our regularized quadratic subproblems \eqref{eq:subproblem:sigma1} were to be minimized exactly, the resulting steps could be interpreted as block-wise proximal Newton-type steps on each coordinate block~$k$ of the dual \eqref{eq:dual}, where the Newton-subproblem is modified to also contain the proximal part $R$. 
This connection only holds for the special case of adding ($\aggpar=1$), and would correspond to a carefully adapted step-size in the block-wise Newton case.

One of the main crucial differences of our proposed \cocoap framework compared to all known batch proximal methods (no matter if block-wise or not) is that the latter do require high accuracy subproblem solutions, and do not allow arbitrary solvers of weak accuracy~$\Theta$ such as we do here, see also the next paragraph.
Distributed Newton methods have been analyzed theoretically only when the subproblems are solved to high precision, see e.g. \cite{DANE}. 
This makes the local solvers very expensive and the convergence rates less general than in our framework (which allows weak local solvers). 
Furthermore, the analysis of \cite{DANE} requires additional strong assumptions on the data partitioning, such that the local Hessian approximations are consistent between the machines.

\paragraph*{\it Distributed Methods Allowing Local Optimization.}
Developing distributed optimization methods that allow for arbitrary weak local optimizers requires carefully devising data-local subproblems to be solved after each communication round.

By making use of the primal-dual structure in the line of work of \cite{Yu:2012fp,Pechyony:2011wi,Yang:2013vl,Yang:2013ui,Lee:2015vra}, the \cocoa and \cocoap frameworks proposed here are the first to allow the use of any local solver --- of weak local approximation quality --- in each round.
Furthermore, the approach here also allows more control over the aggregation of updates between machines. 
The practical variant of the DisDCA Algorithm of \cite{Yang:2013vl}, called DisDCA-p, also allows additive updates but is restricted to coordinate decent \mcx{(CD)} being the local solver, and was \mcx{initially} proposed without convergence guarantees. 
\mcx{The work of \cite{Yang:2013ui} has provided the first theoretical convergence analysis for an ideal case, when the distributed data parts are all orthogonal to each other --- an unrealistic setting in practice.}
DisDCA-p can be recovered as a special case of the \cocoap framework when using \mcx{CD} as a local solver, if $|\mathcal{P}_k| = n/K$ and when using \mcx{the conservative bound} $\sigma':=K$, see  %
also \cite{Lee:2015vra,Ma:2015ti}.
The convergence theory presented here therefore also covers that method, \mcx{and extends it to arbitrary local solvers.}

\paragraph*{\it Inexact Block Coordinate Descent.} 
Our framework is related, but not identical, to running an {\em inexact} version of  block coordinate ascent, applied to all block in parallel, and to the dual problem, where the level of inexactness is controlled by the parameter $\Theta$ through the use of a (possibly randomized) iterative ``local'' solver applied to the subproblems (local problems). For previous work on {\em randomized} block coordinate descent we refer to~\cite{ICD}. See also \cite{DQA}.

\section{Numerical Experiments}
\label{sec:experiments}

In this section we explore numerous aspects of our distributed framework and demonstrate  its competitive performance in practice.
Section~\ref{sec:LocalSolverExps} first explores the impact of the local solver on overall performance, by comparing examples of various local solvers that can be used in the framework (the improved \cocoap framework as shown in Algorithms \ref{alg:cocoa} and \ref{alg:cocoaPractical}) as well as testing the effect of approximate solution quality. The results indicate that the choice of local solver can have a significant impact on overall performance. %
In Sections~\ref{sec:adingVsAveraging} and \ref{sec:subproblemParamExps} we further explore framework parameters, looking at the impact of the aggregation parameter $\nu$ and the subproblem parameter $\sigma'$, respectively. Finally, Section~\ref{sec:hugeDatasetExp} demonstrates competitive practical performance of the overall framework on a large 280GB distributed dataset.

We conduct experiments on three datasets of moderate and large size, namely \emph{rcv1test}, \emph{epsilon} and \emph{splice-site.t}\footnote{The datasets are available at \url{http://www.csie.ntu.edu.tw/~cjlin/libsvmtools/datasets/}.}. The details of these datasets are listed in Table~\ref{tab:datasets}.
\begin{table}[h]
\tbl{Datasets used for numerical experiments.}
{
      \begin{tabular}{crrr}
      \toprule
    Dataset & \multicolumn{1}{c}{$n$} &
    \multicolumn{1}{c}{$d$} & 
    \multicolumn{1}{c}{size(GB)} \\
    \colrule 
	rcv1test & 677,399 &
	  47,236 & 1.2 \\
		epsilon & 400,000 &
	  2,000 & 3.1 \\	
	  splice-site.t & 4,627,840 &
	  11,725,480 & 273.4
	\\  \botrule  
      \end{tabular}
}    
\label{tab:datasets}

\end{table}

For solving subproblems, we compare numerous local solver methods, as listed in Table~\ref{tbl:othersolvers}. Also, we apply Euclidean norm as regularizer $g(x) = \|x\|^2$ for all the experiments.  All the algorithms are implemented in C\texttt{++} with MPI, and experiments are run on a cluster of 4 Amazon EC2 m3.xlarge instances. Our open source code is available online at: \url{https://github.com/optml/CoCoA}.
\begin{table}[H]
\tbl{Local solvers used in numerical experiments.}
{
\begin{tabular}{ll}
\toprule
CD  & Coordinate Descent \cite{richtarik}\\

APPROX & Accelerated, Parallel and Proximal Coordinate Descent \cite{APPROX, APPROX-SIREV} \\

GD & Gradient Descent with Backtracking Line Search \cite{NocedalWrightBook}\\

CG & Conjugate Gradient Method \cite{CG}\\

L-BFGS &Quasi-Newton with Limited-Memory BFGS Updating \cite{byrd1995limited}\\

BB &  Barzilai-Borwein Gradient Method   \cite{barzilai1988two}\\

FISTA & Fast Iterative Shrinkage-Thresholding Algorithm \cite{beck2009fast}\\
\botrule
\end{tabular}
}
\label{tbl:othersolvers}
\end{table}

\subsection{Exploration of Local Solvers within the Framework}
\label{sec:LocalSolverExps}
In this section we compare the performance of our framework for various local solvers and various choices of inner iterations performed by a given local solver, resulting in different local accuracy measures $\Theta$. 
For simplicity, we choose the subproblem parameter $\sigma' := \nu K$ (see Lemma \ref{lem:sigmaPrimeNotBad}) as a simple obtainable and theoretically safe value for our framework.

\subsubsection{Comparison of Different Local Solvers}

Here we compare the performance of the various local solvers listed in Table \ref{tbl:othersolvers}. We here show results for quadratic loss function $\ell_i(a) = \frac12 (a-y_i)^2$
with three different values of the regularization parameter, $\lambda$=$10^{-3}$, $10^{-4}$, and $10^{-5}$, for $g(.)$ being the default Euclidean squared norm regularizer.
The dataset is RCV and we ran the framework for a maximum of $T:=100$ communication rounds. We set $\aggpar=1$ (adding) and choose $H$ which gave the best performance in CPU time (see Table \ref{tbl:optH}) for each solver.
From Figure~\ref{fig:dffsolvers}, we find that the \mcx{coordinate descent (CD)} local solver always outperforms the other solvers, even though it may be slower than L-BFGS at the beginning. The reason for this is that \mcx{CD}, as compared to the other methods, does not need to spend time evaluating the full (batch) gradient and function values. Also note that some of the solvers cannot guarantee strict decrease of the duality gap, and sometimes this fluctuation can be very dramatic.

\begin{table}[H]
\tbl{Optimal $H$ for different local solvers for RCV dataset. }
{
\begin{tabular}{c | c c c c c c c}
\toprule
Local Solver &\mcx{CD}  & APPROX & GD& CG& L-BFGS& BB& FISTA\\
$H $ &40,000 & 40,000 & 20&  5 & 10 & 15 &20 \\
\botrule
\end{tabular}
}
\label{tbl:optH}
\end{table}

\begin{figure}[H]  
\centering                  
\includegraphics[scale=.19]{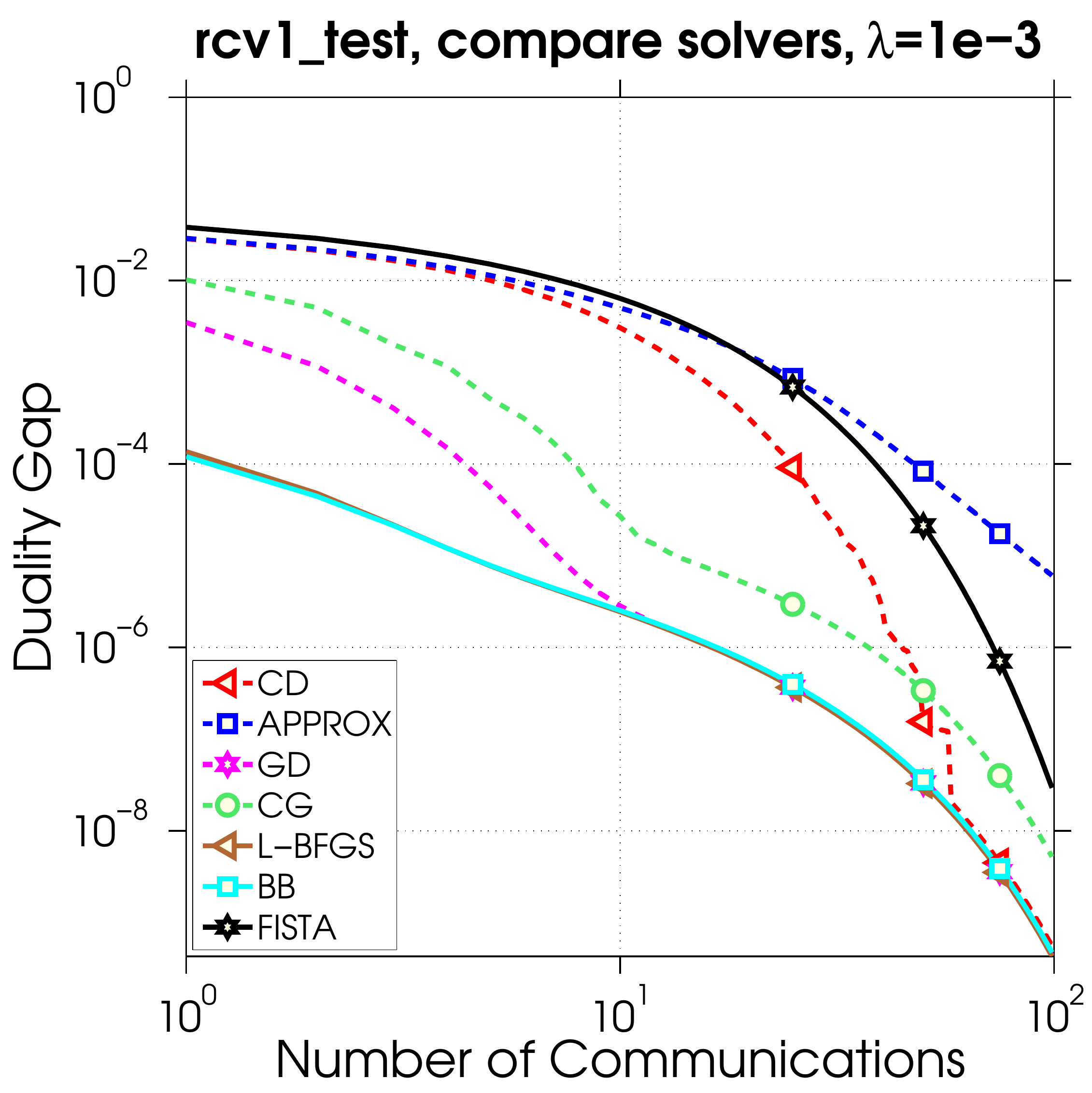}
\includegraphics[scale=.19]{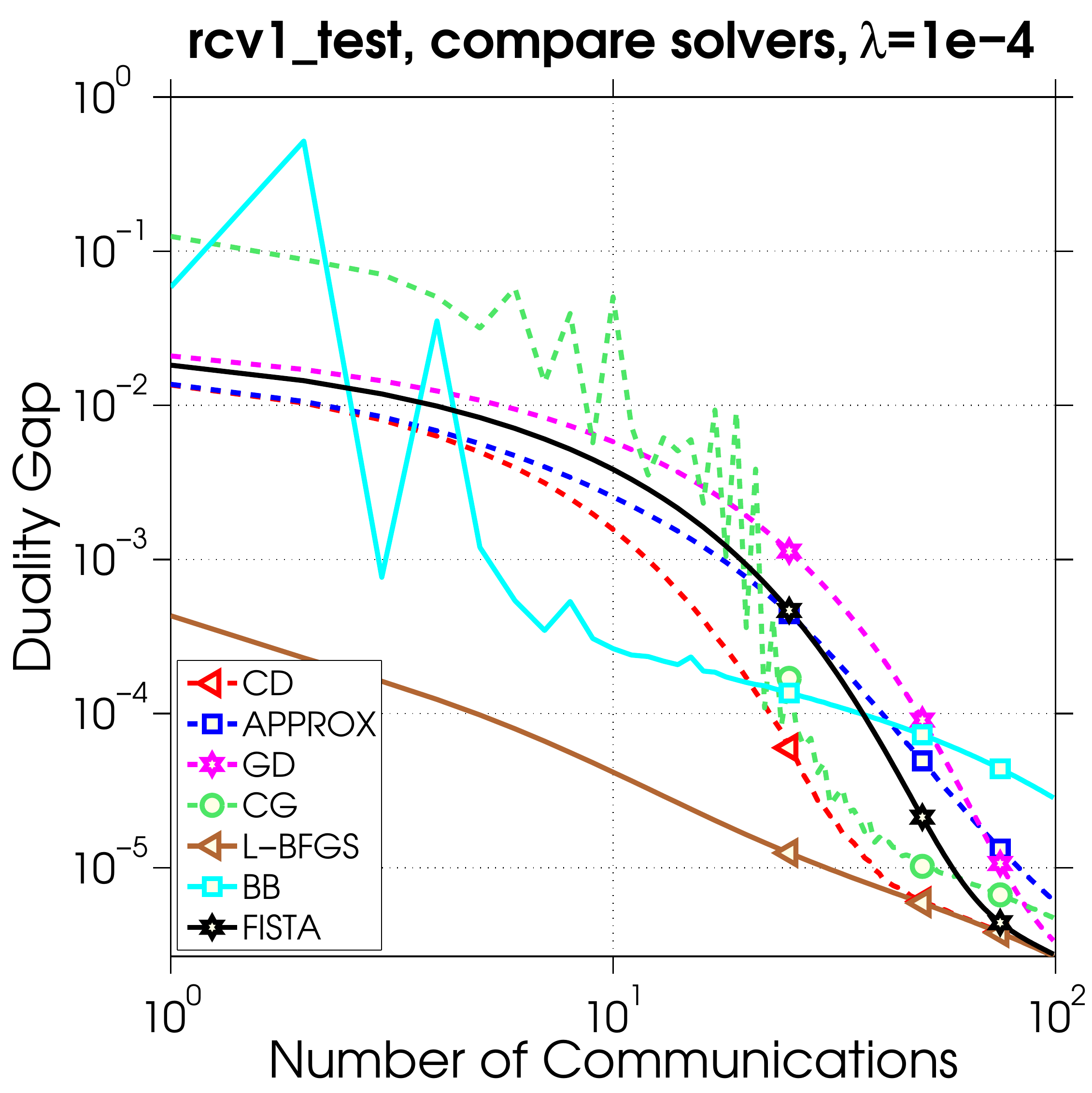}
\includegraphics[scale=.19]{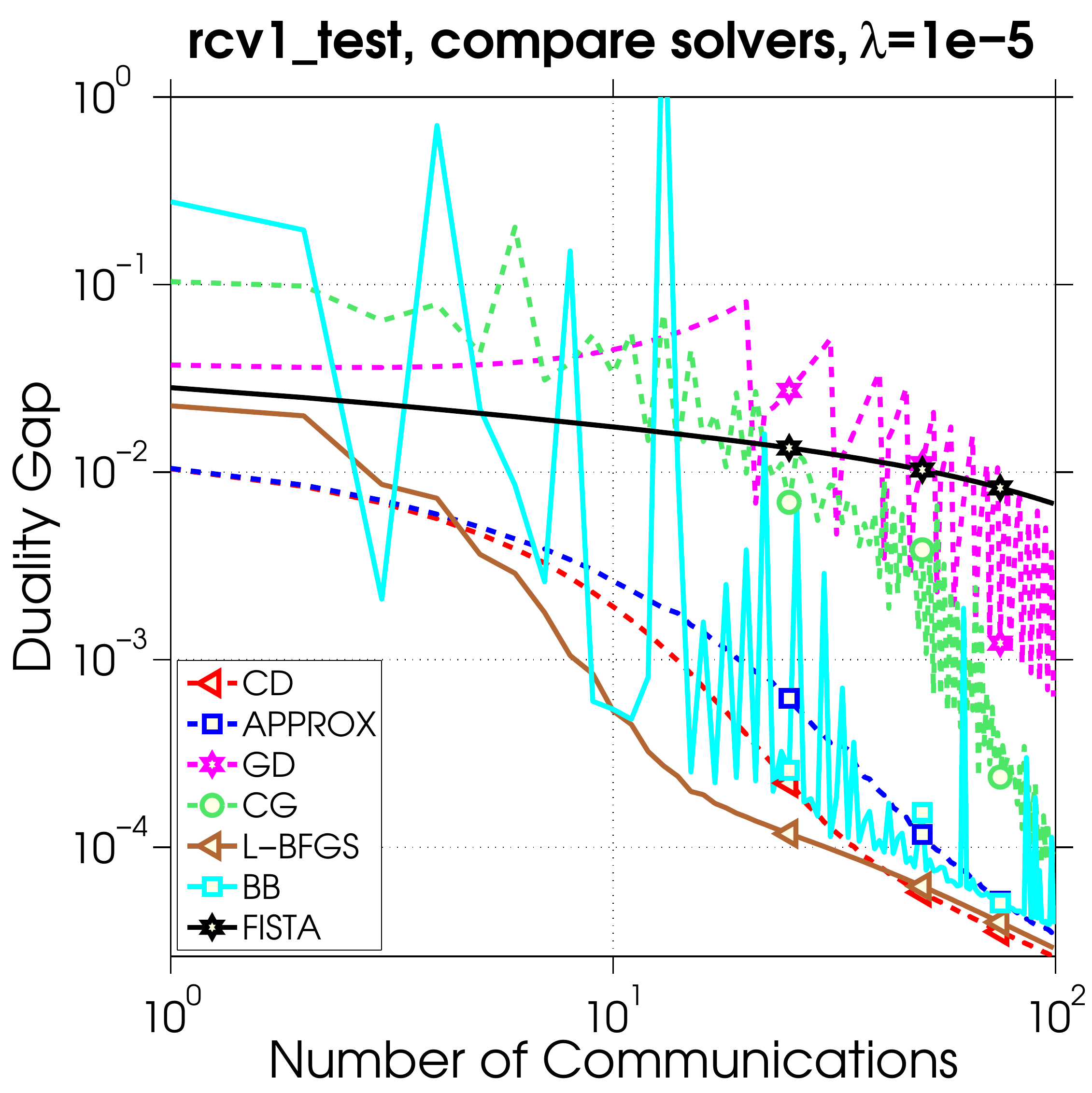}

\includegraphics[scale=.19]{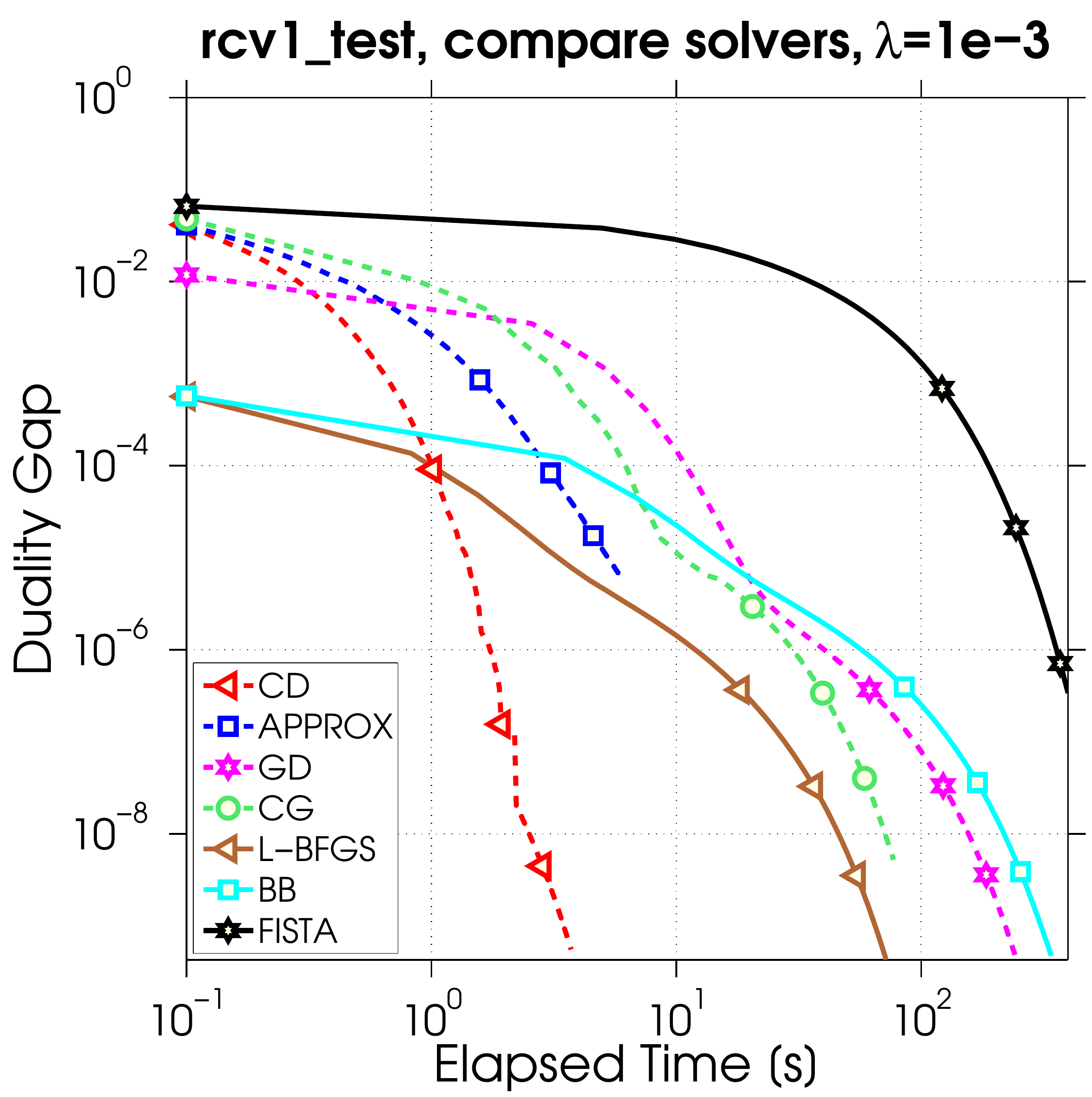}
\includegraphics[scale=.19]{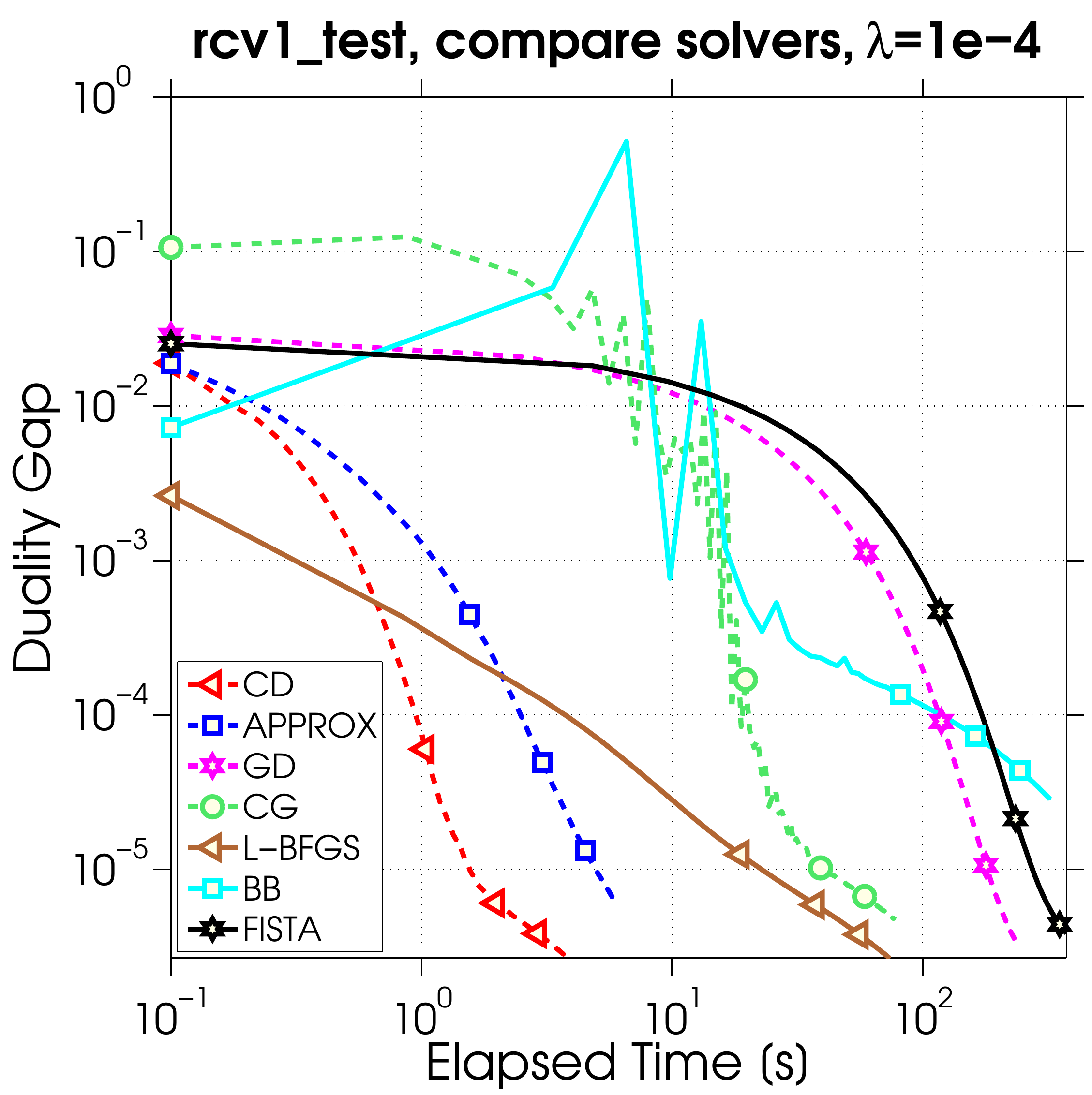}
\includegraphics[scale=.19]{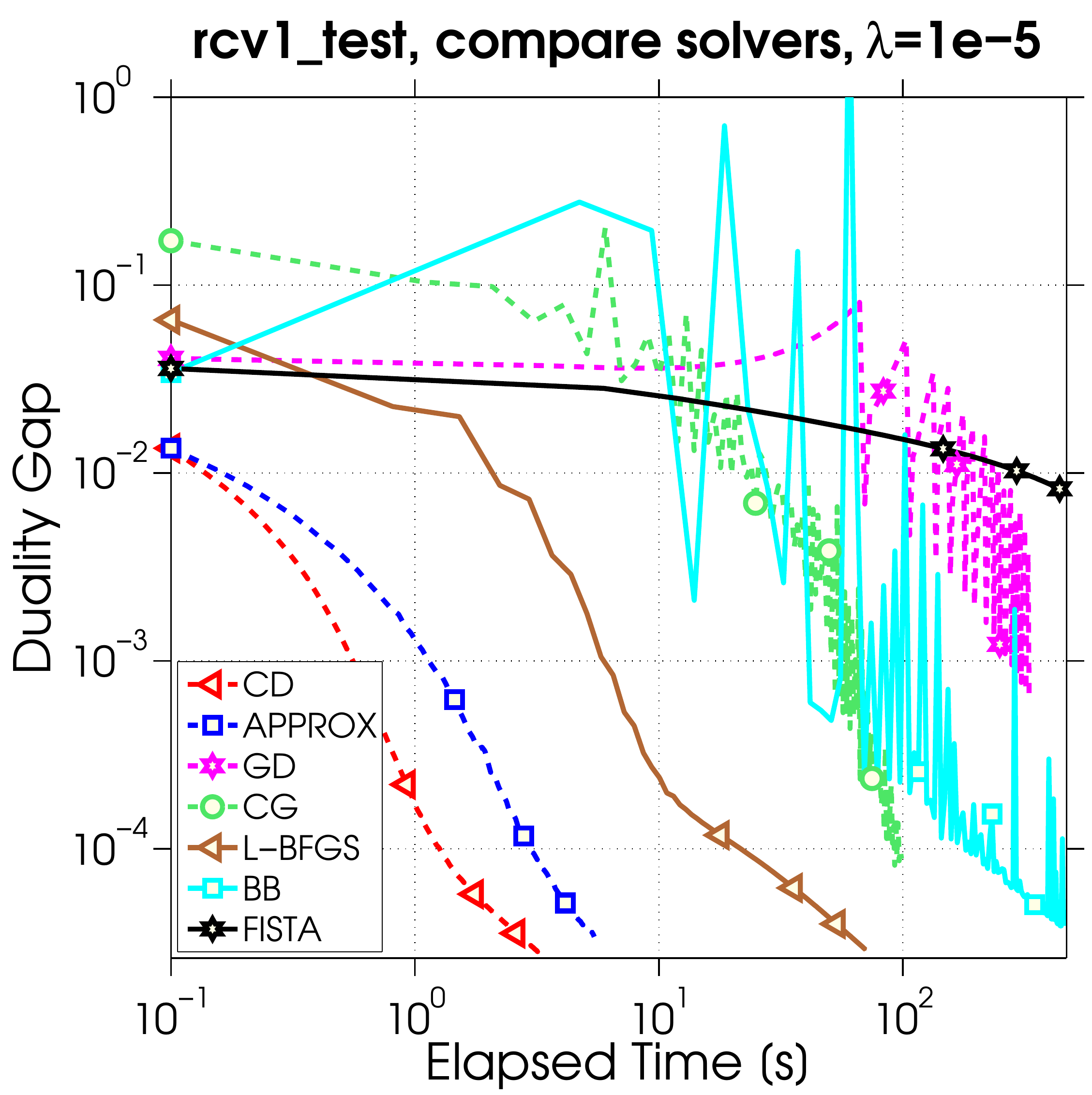}
\caption{Performance of different local solvers.} 
\label{fig:dffsolvers}
\end{figure}
 
\subsubsection{Effect of the Quality of Local Solver Solutions on Overall Performance}

Here we discuss how the quality of subproblem solutions affects the overall performance of Algorithm~\ref{alg:cocoaPractical}. In order to do so, we denote $H$ as the number of iterations the local solver is run for, within each communication round of the framework.
We choose various values for $H$ on two local solvers, \mcx{CD}~\cite{richtarik,ShalevShwartz:2013wl} and L-BFGS~\cite{byrd1995limited}, which gave the best performance in general. For \mcx{CD}, $H$ represents the number of local iterations performed on the subproblem. For L-BFGS, $H$ not only means the number of iterations, but also stands for the size of past information used to approximate the Hessian (i.e., the size of limited memory).  

Looking at Figures~\ref{fig:dffsollbfgs} and \ref{fig:dffsollbfgsxxx}, we see that for both local solver and all values of~$\lambda$, increasing $H$ will lead to less iterations of Algorithm~\ref{alg:cocoaPractical}. Of course, increasing $H$ comes at the cost of the time spent on local solvers increasing. Hence, a larger value of $H$ is not always the optimal choice with respect to total elapsed time. For example, for the rcv\_test dataset, when choosing \mcx{CD} to solve the subproblems, choosing $H$ to be $40,000$ uses less time and provides faster convergence. When using L-BFGS, $H = 10$ seems to be the best choice.

\begin{figure}[H]
\centering
\includegraphics[scale=.19]{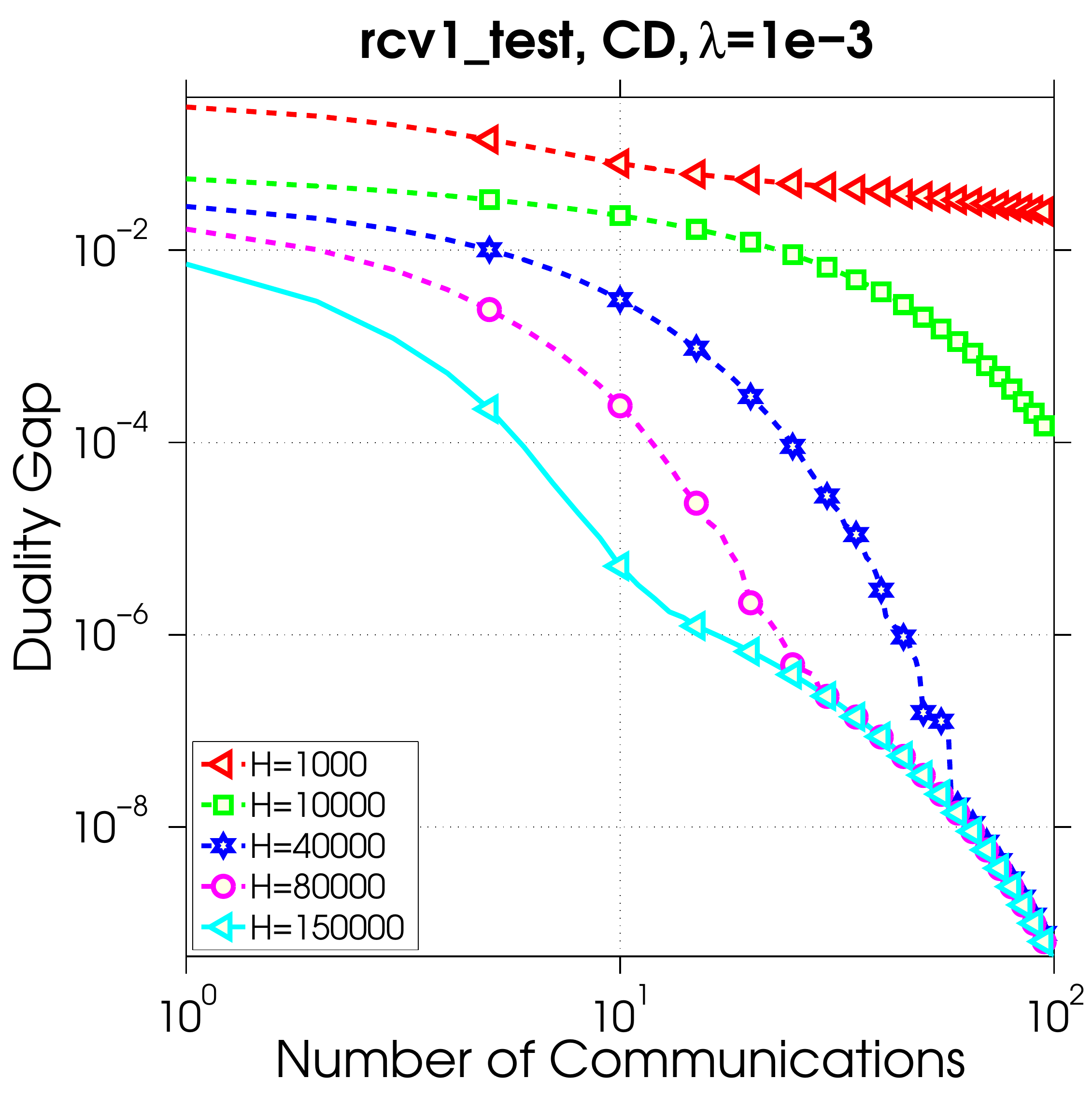}
\includegraphics[scale=.19]{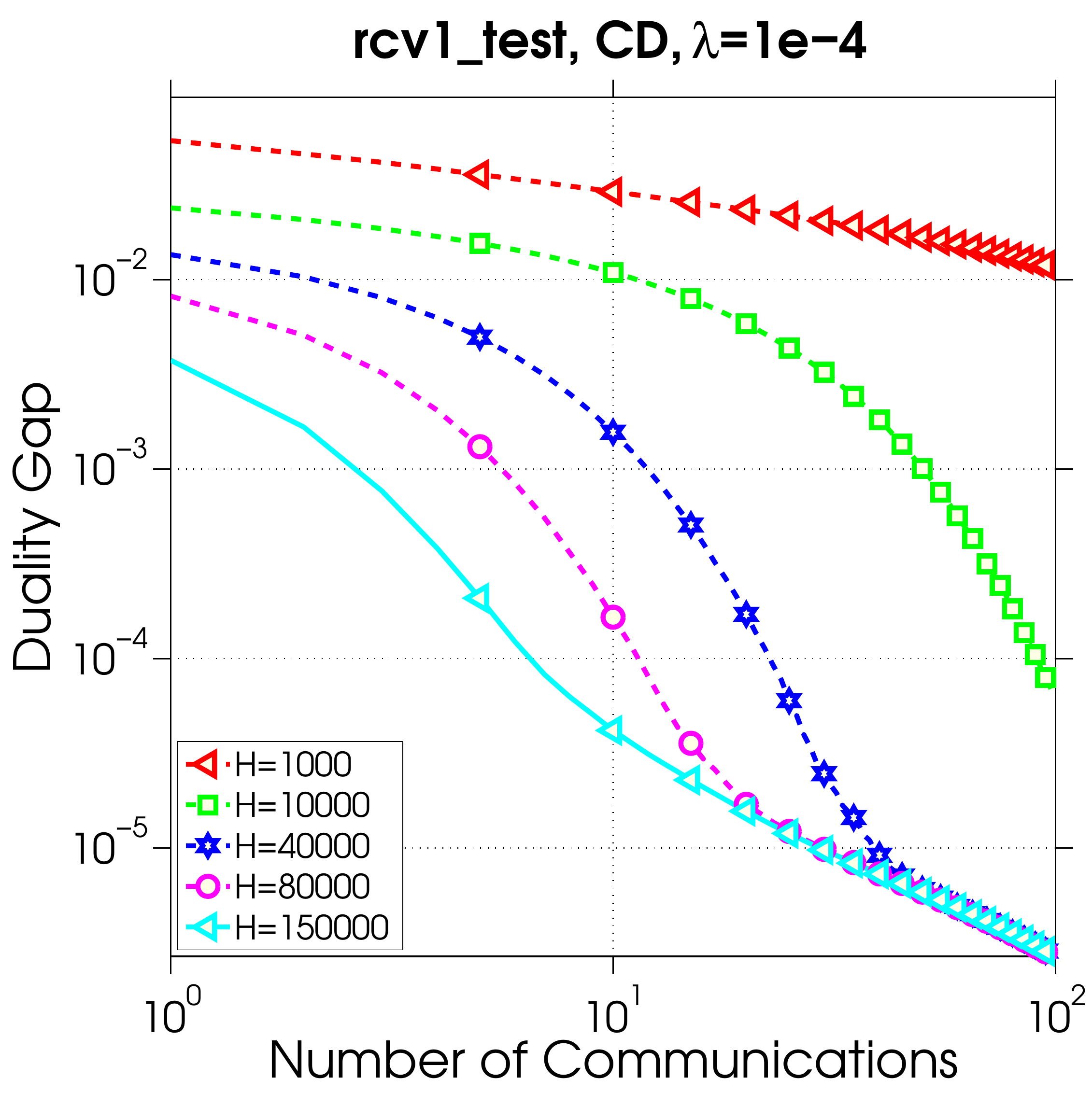}
\includegraphics[scale=.19]{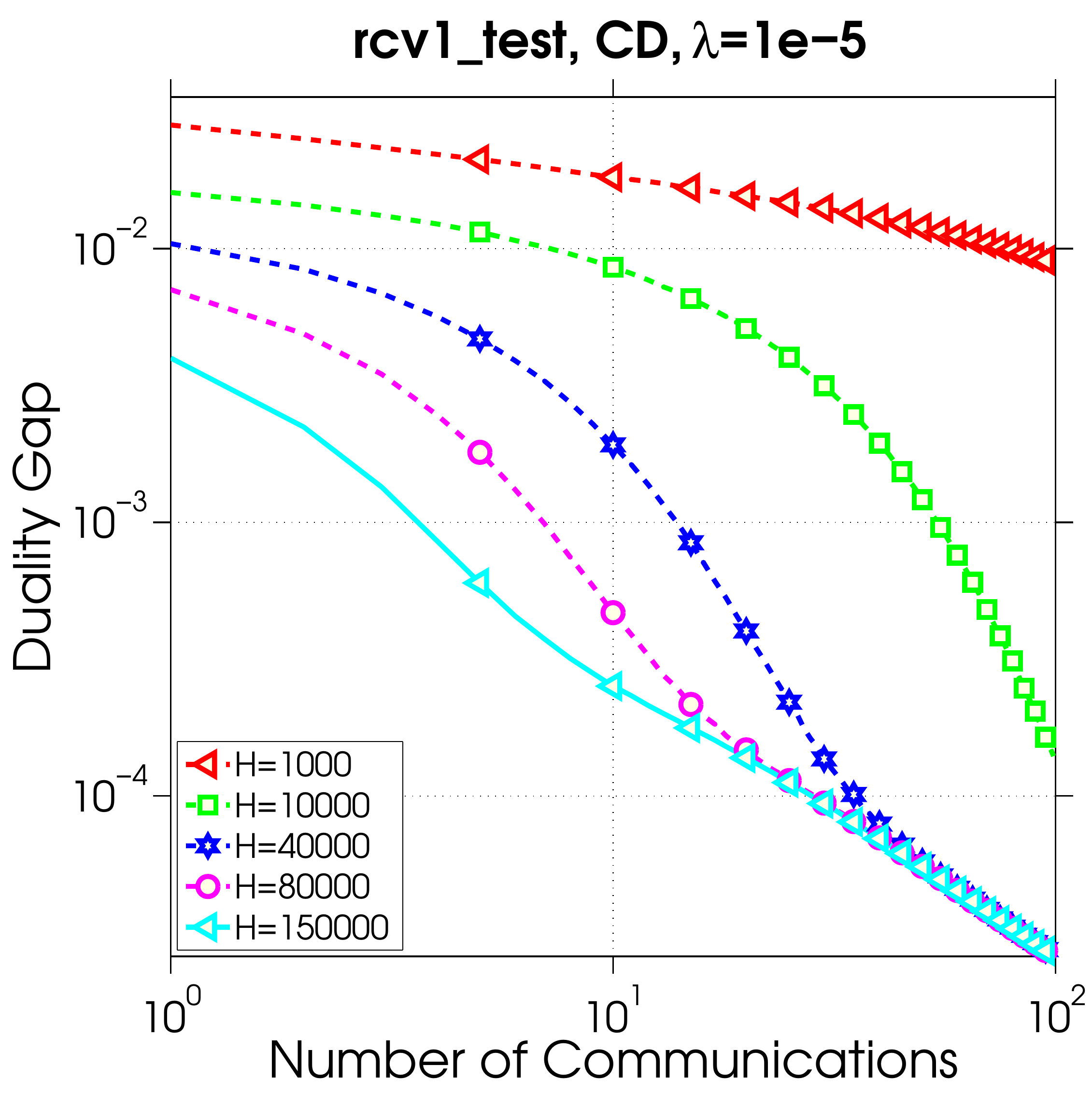}

\includegraphics[scale=.19]{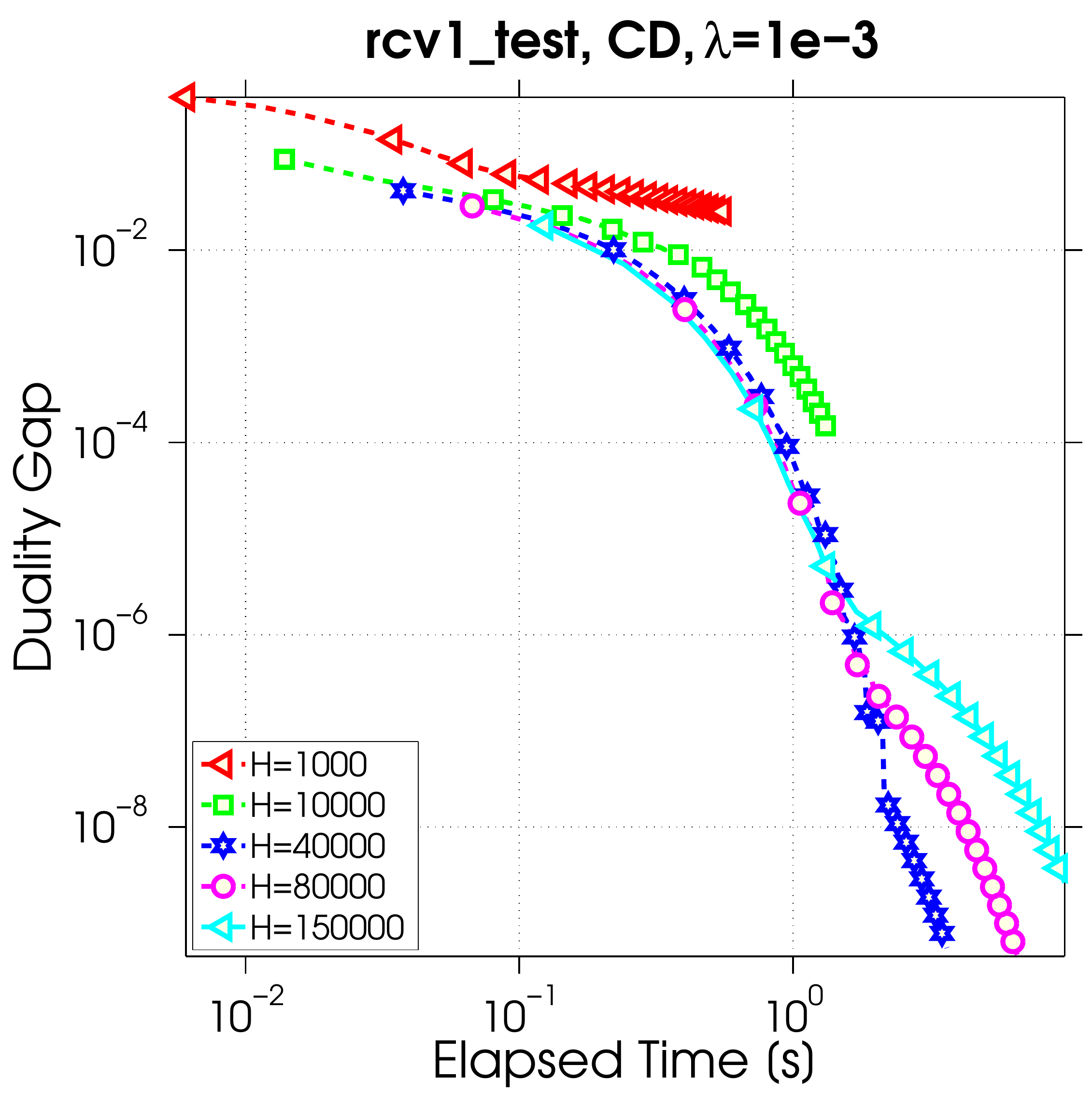}
\includegraphics[scale=.19]{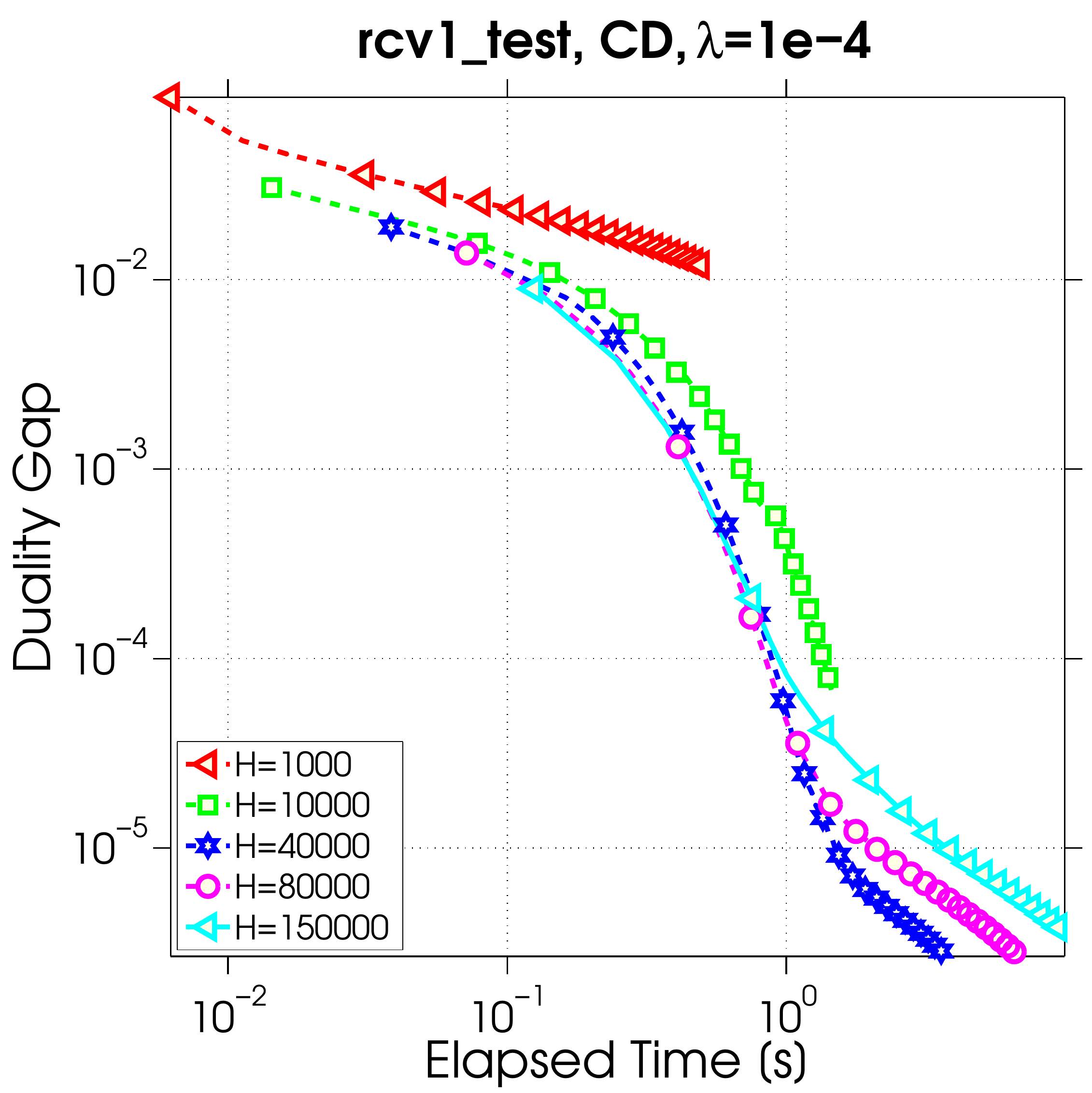}
\includegraphics[scale=.19]{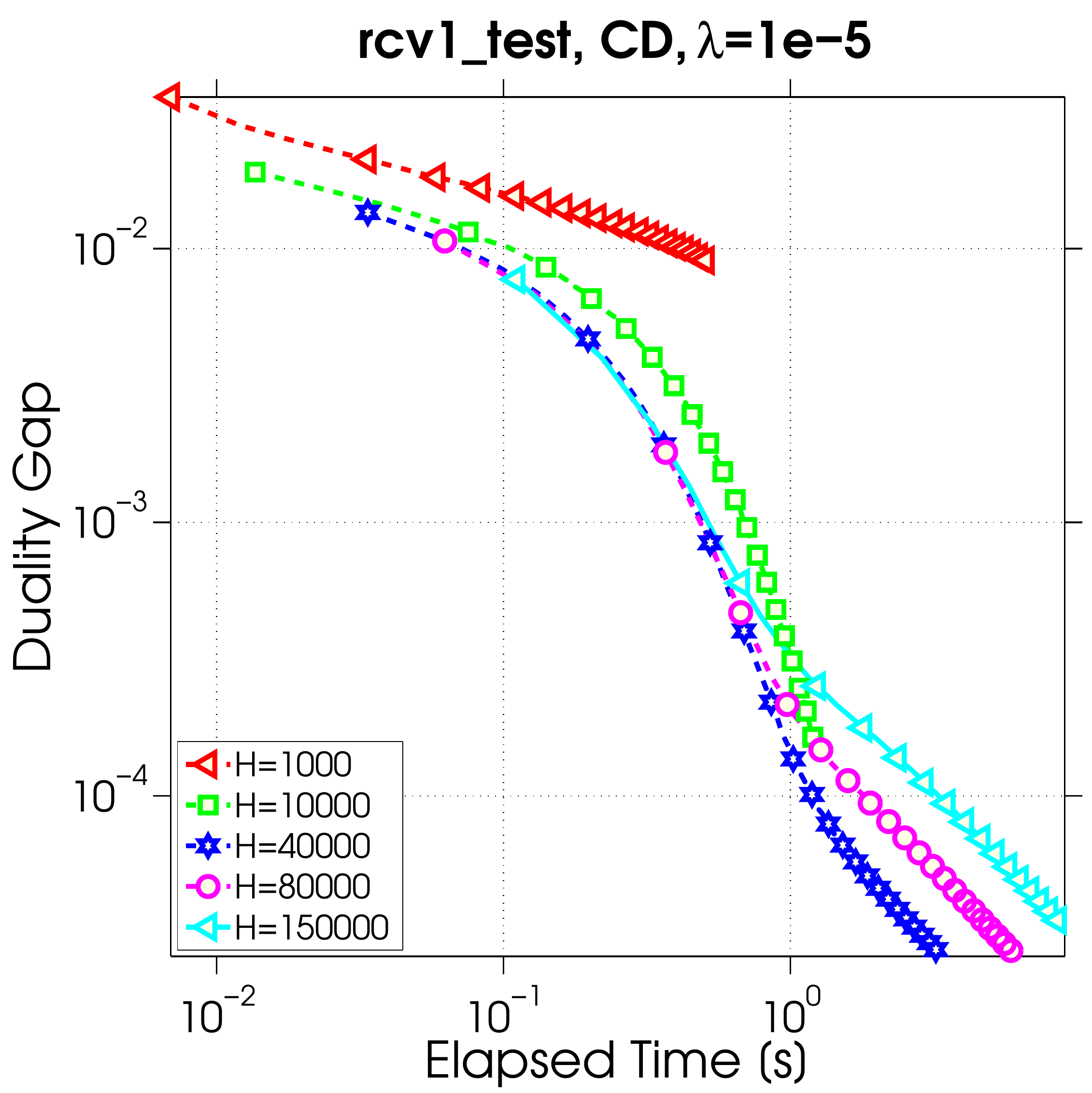}
\caption{Varying the number of iterations of \mcx{CD}  as a local solver.} 
\label{fig:dffsollbfgs}
\end{figure}

\begin{figure}[H]
\centering
\includegraphics[scale=.19]{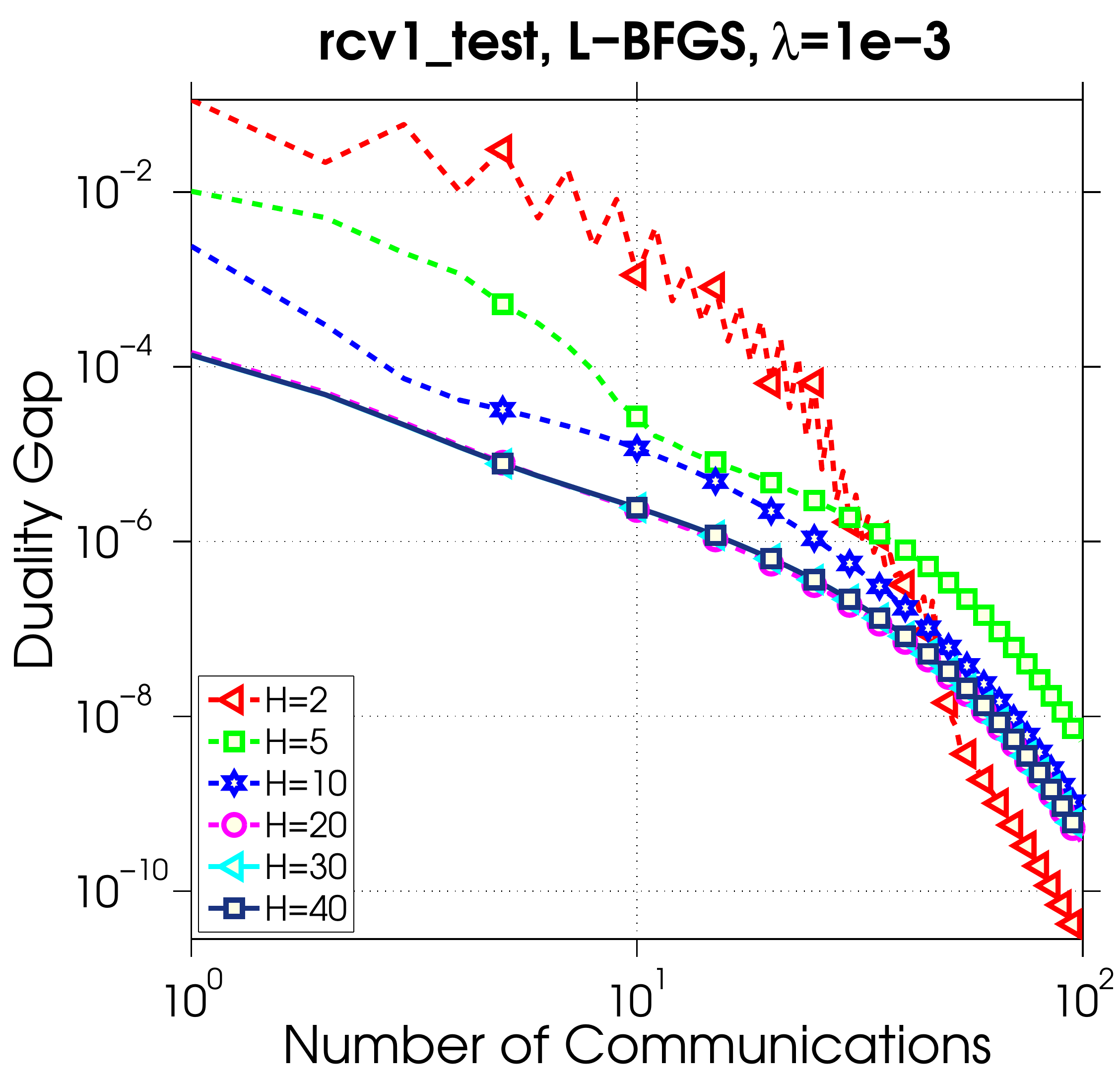}
\includegraphics[scale=.19]{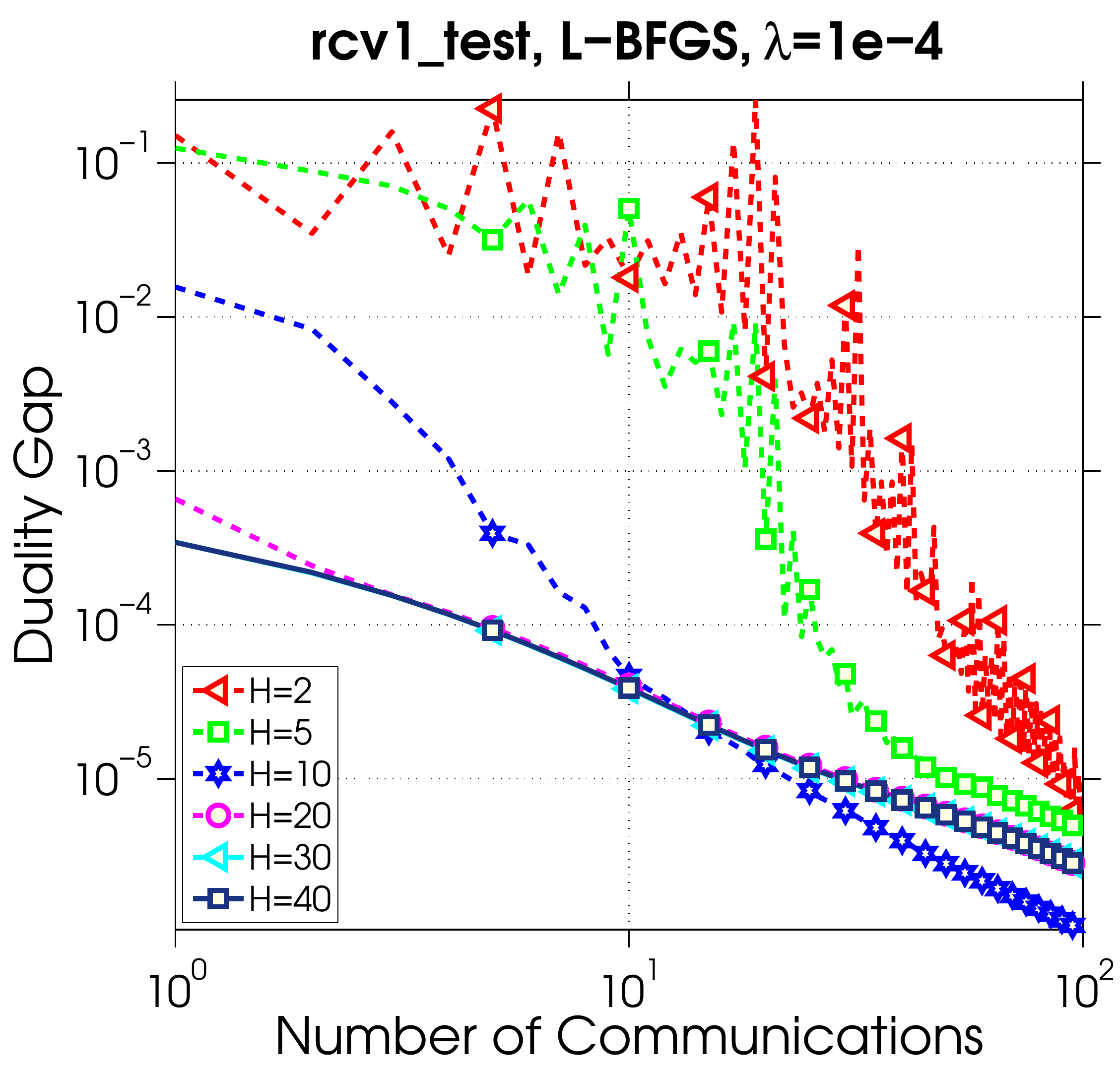}
\includegraphics[scale=.19]{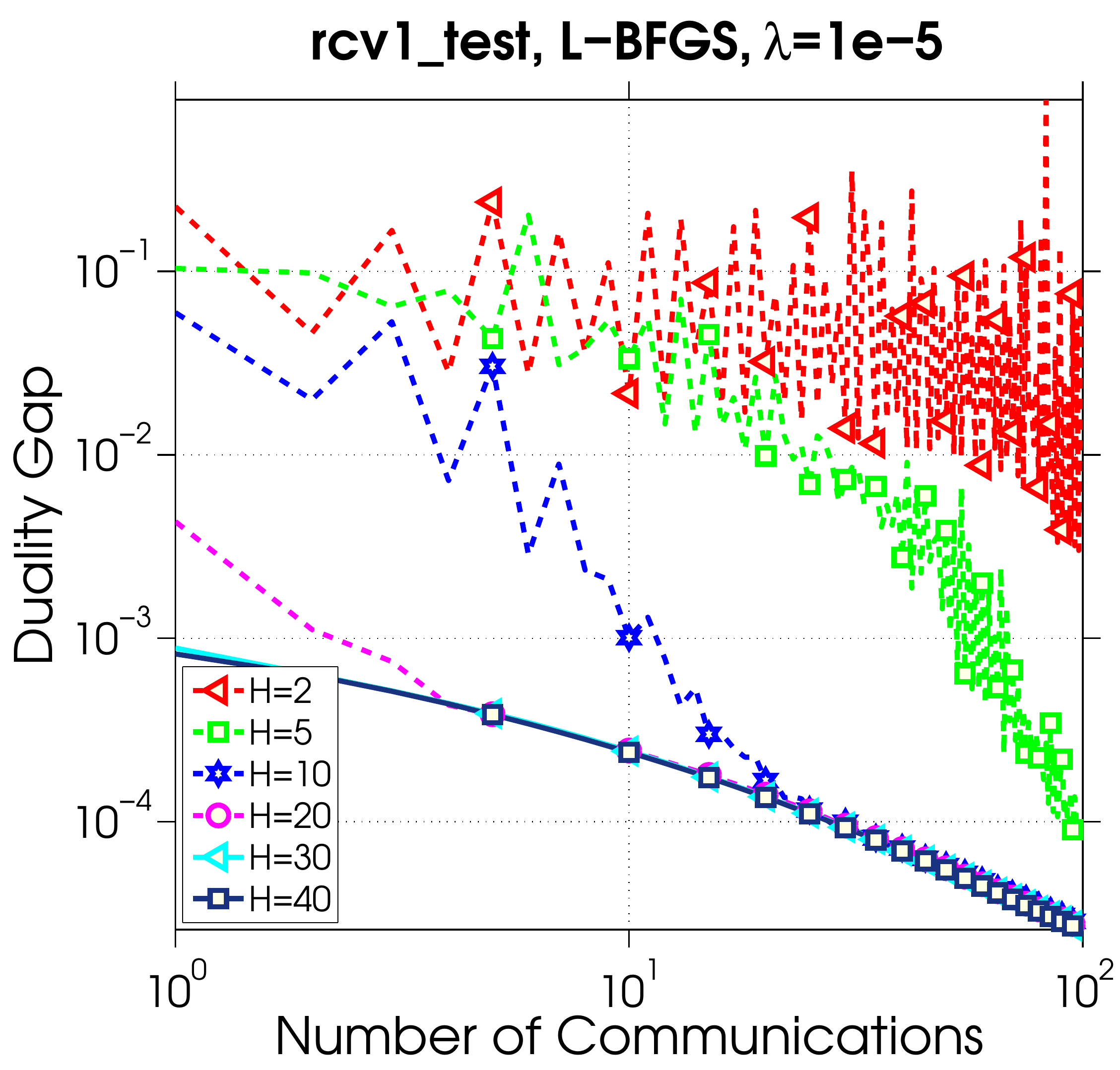}

\includegraphics[scale=.19]{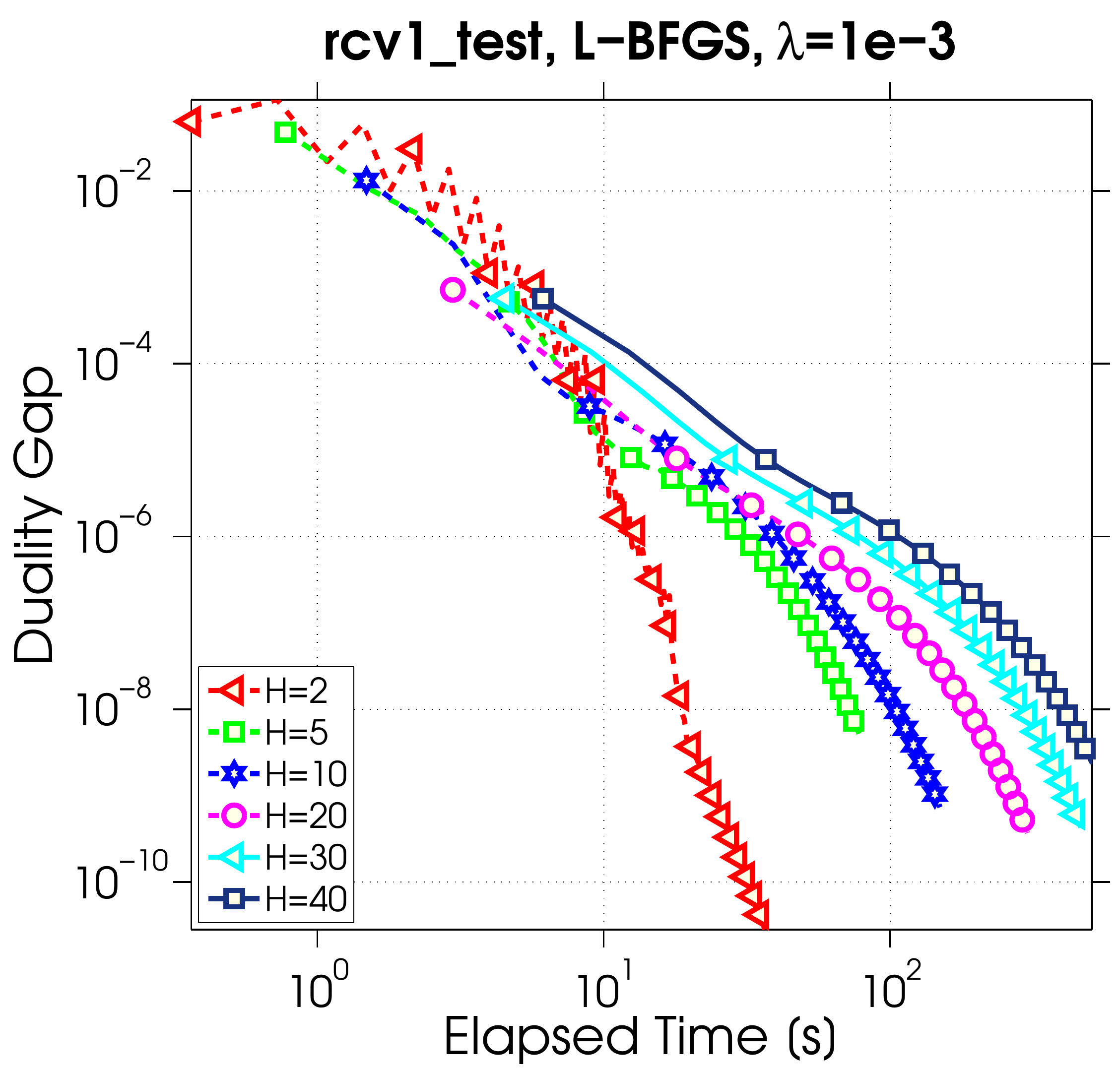}
\includegraphics[scale=.19]{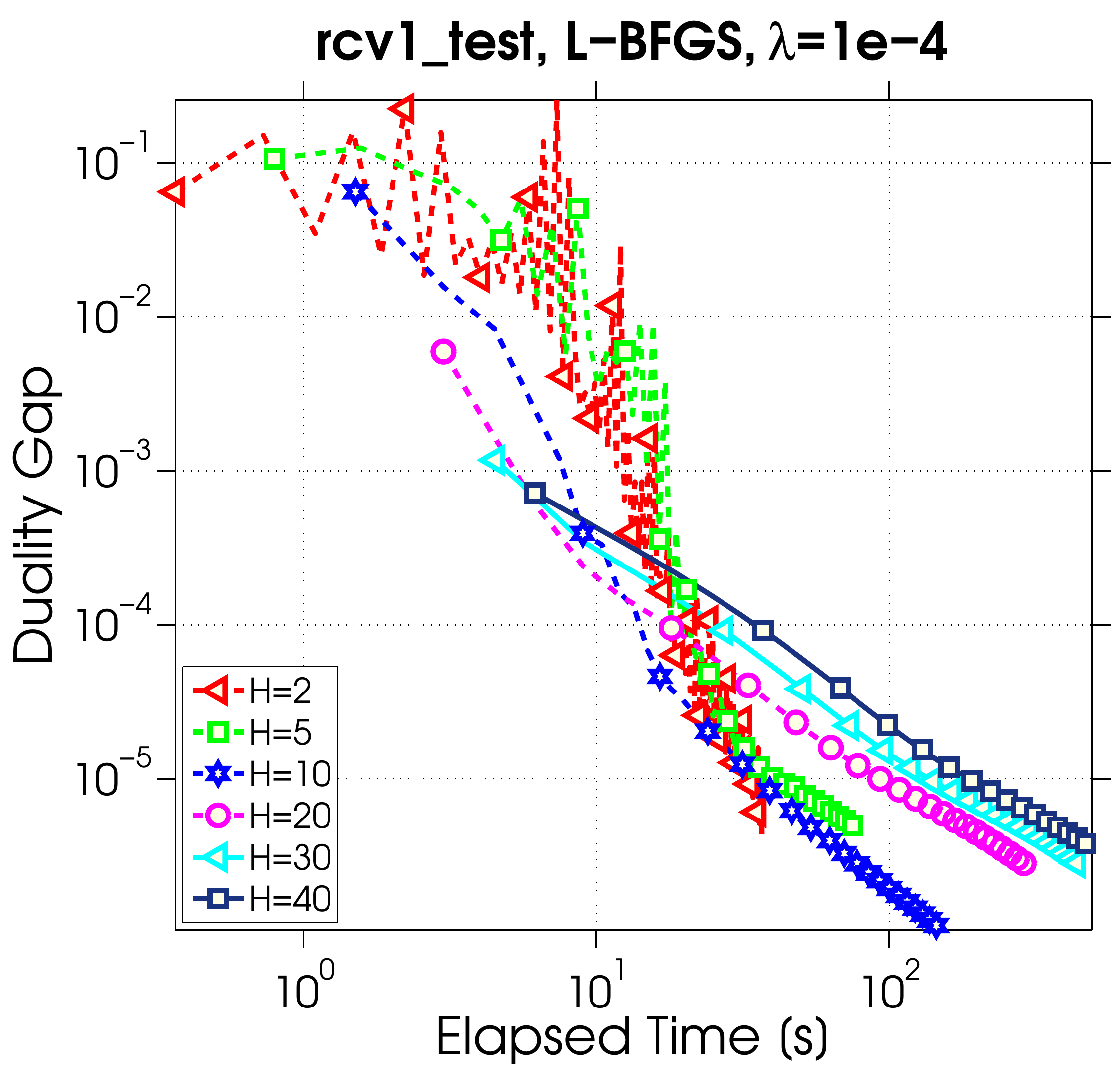}
\includegraphics[scale=.19]{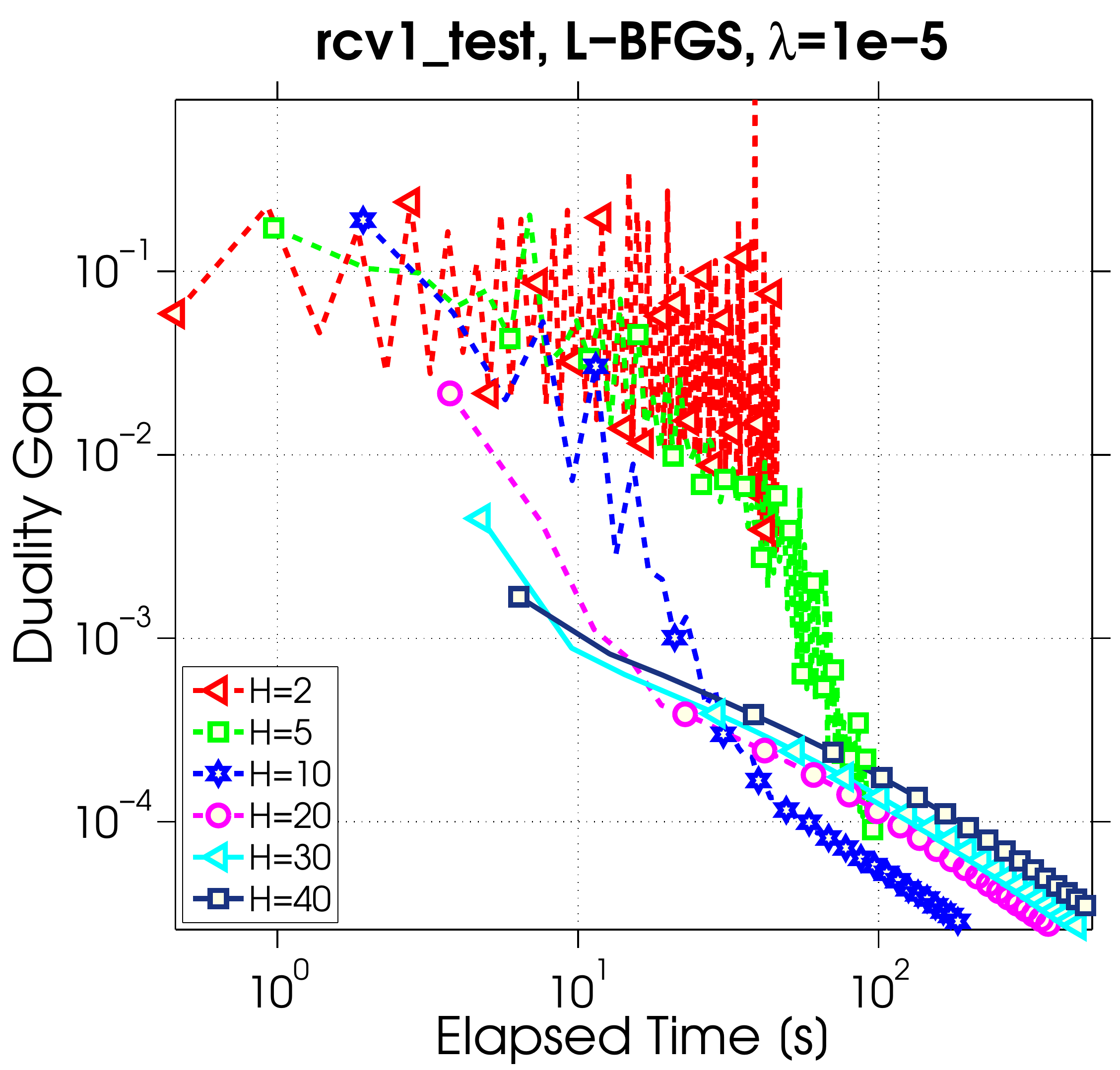}
\caption{Varying the number of iterations of L-BFGS as a local solver.} 
\label{fig:dffsollbfgsxxx}
\end{figure}

\subsection{Averaging vs.\ Adding the Local Updates}  
\label{sec:adingVsAveraging}
In this section, we compare the performance of our algorithm using two different schemes for aggregating partial updates: adding vs. averaging. This corresponds to comparing two extremes for the parameter $\nu$, either $\nu:=\frac1K$ (averaging partial solutions) or $\nu:=1$ (adding partial solutions). As discussed in Section~ \ref{sec:result}, adding the local updates ($\aggpar=1$) will lead to less iterations than taking averaging, due to choosing different $\sigma'$ in the subproblems. We verify this experimentally by considering several of the local solvers listed in Table~ \ref{tbl:othersolvers}.

We show results for RCV dataset, and we apply quadratic loss function with three different choices for the regularization parameter, $\lambda$=$1e-03$, $1e-04$, and $1e-05$. The experiments in Figures~\ref{fig:soler2}--\ref{fig:soler7} indicate that the ``adding strategy'' will  always lead to faster convergence than averaging, even though the difference is minimal when we apply a large number of iterations in the local solver. All the blue solid plots (adding) outperform the red dashed plots (averaging), which indicates the advantage of choosing $\aggpar= 1$. Another note here is that for smaller $\lambda$, we will have to spend more iterations to get the same accuracy (because the original objective function \eqref{eq:primal} is less strongly convex).

\begin{figure}[H]
\centering
\includegraphics[scale=.19]{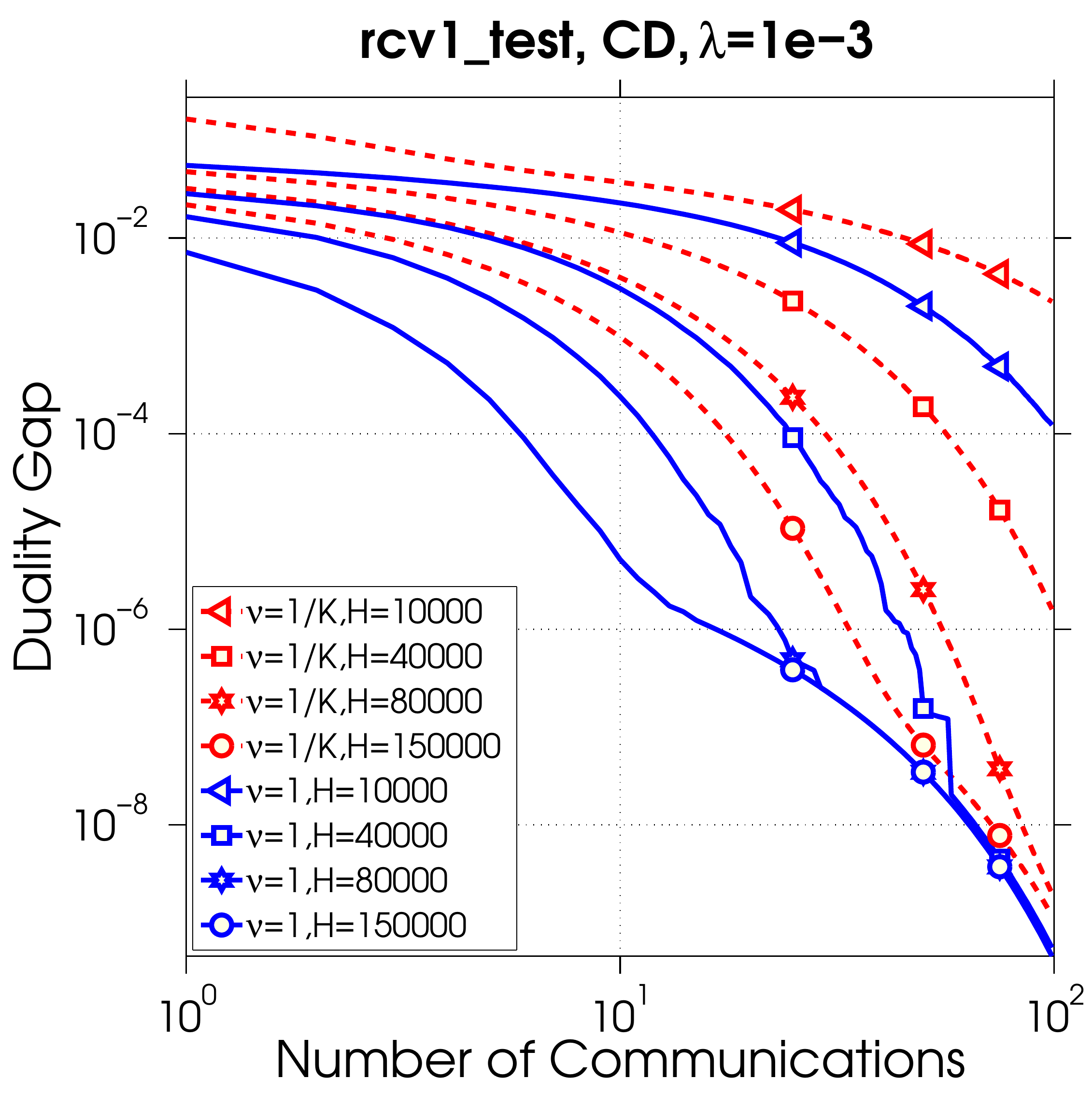}
\includegraphics[scale=.19]{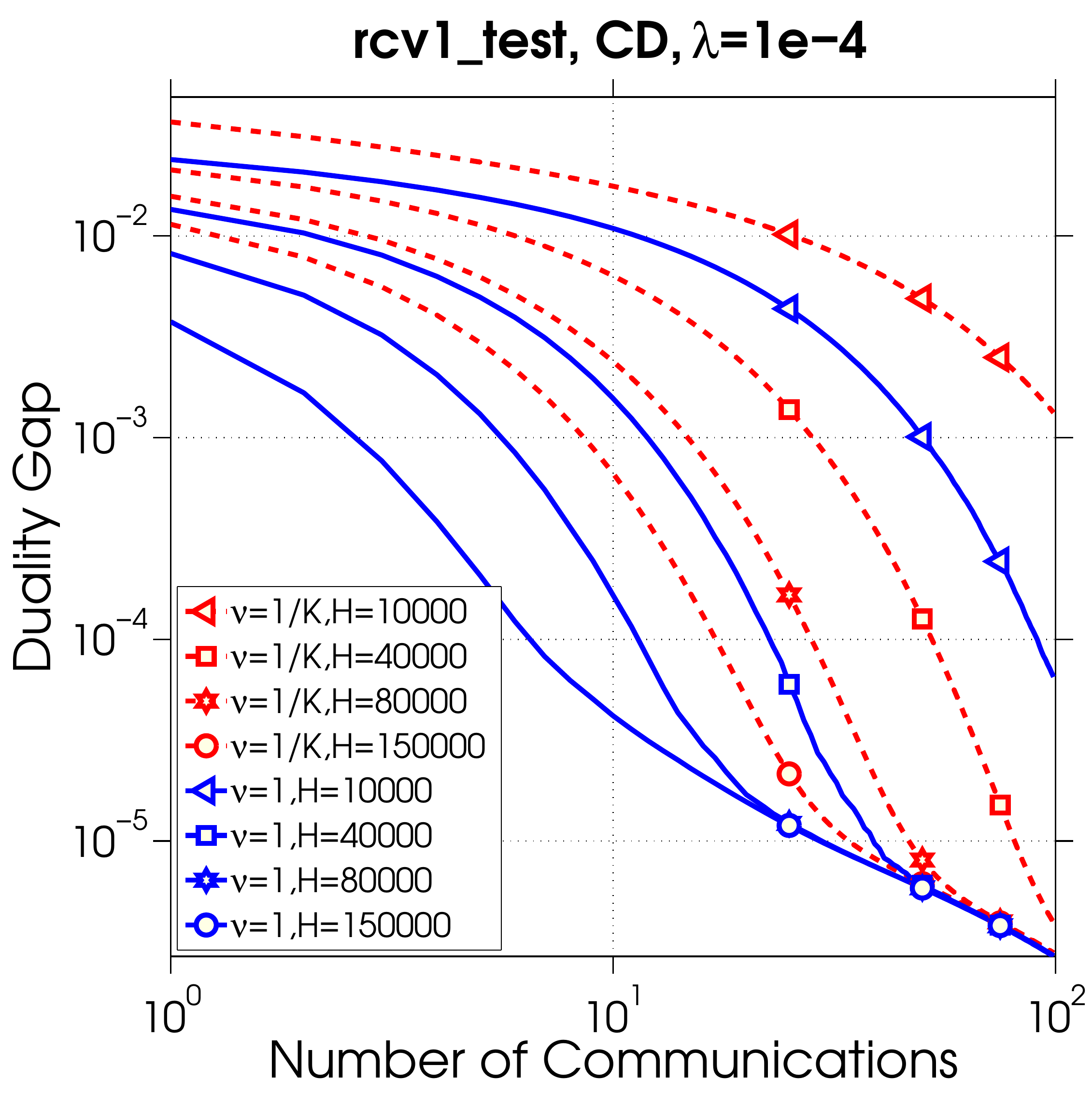}
\includegraphics[scale=.19]{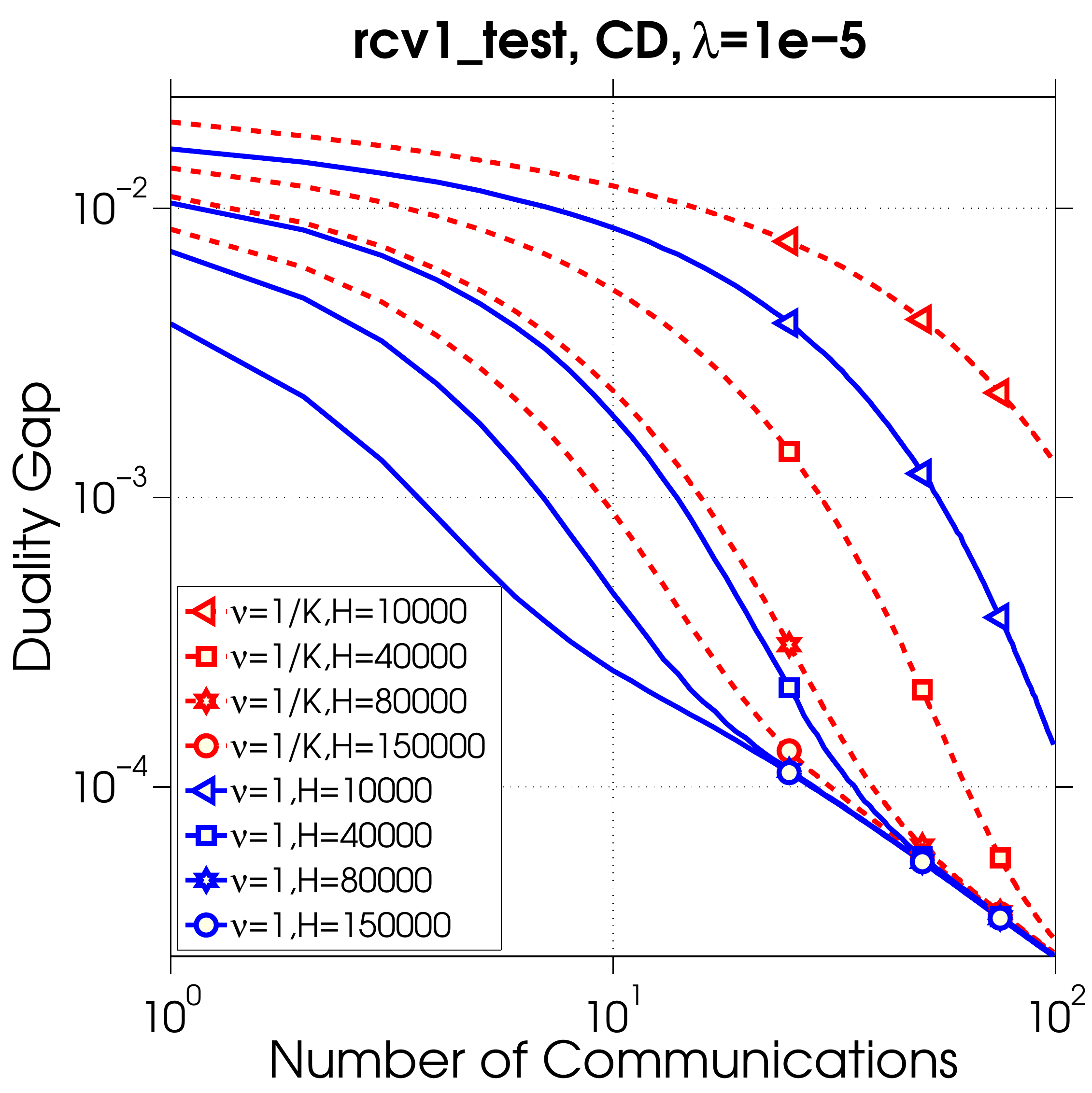}

\includegraphics[scale=.19]{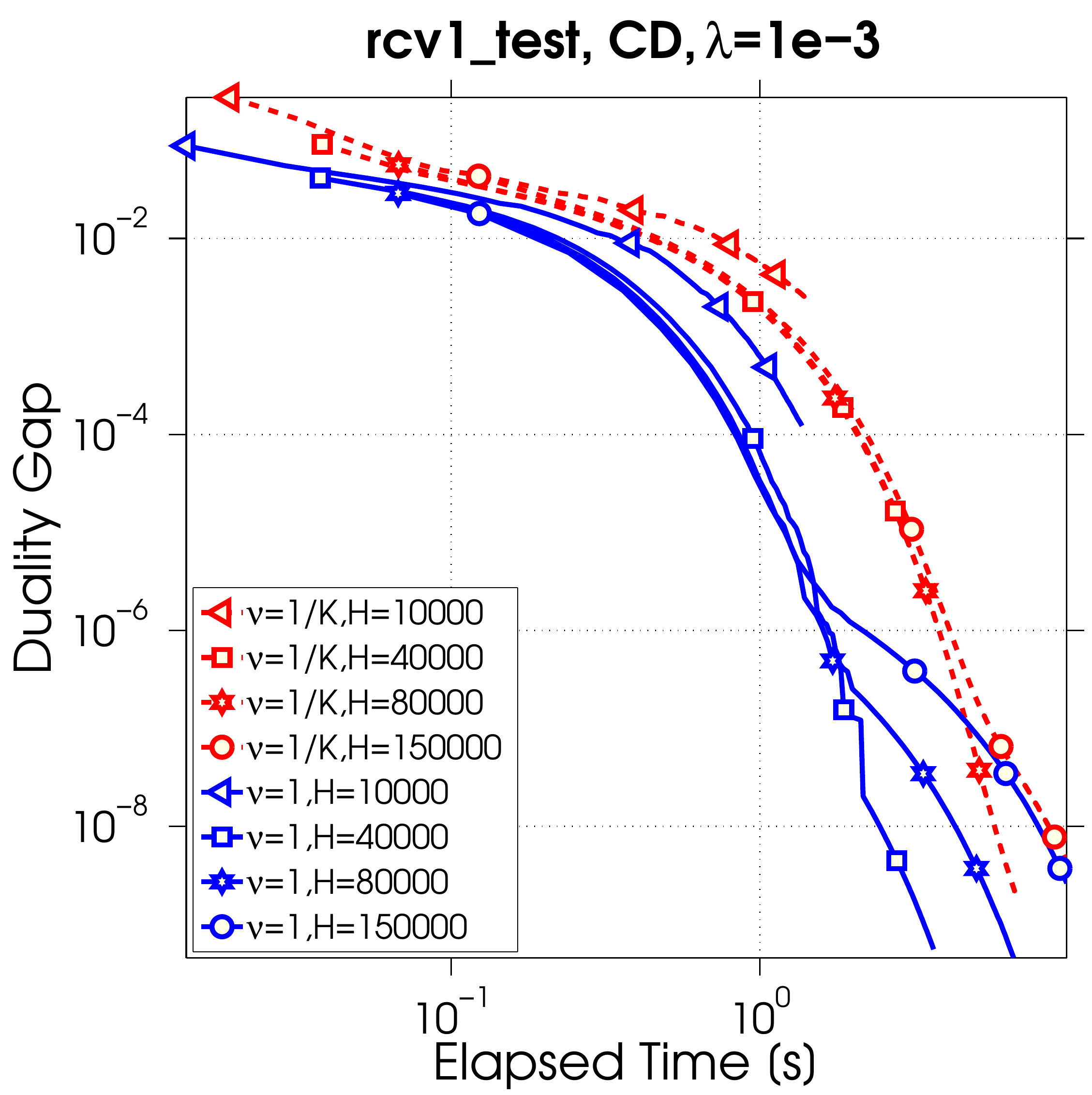}
\includegraphics[scale=.19]{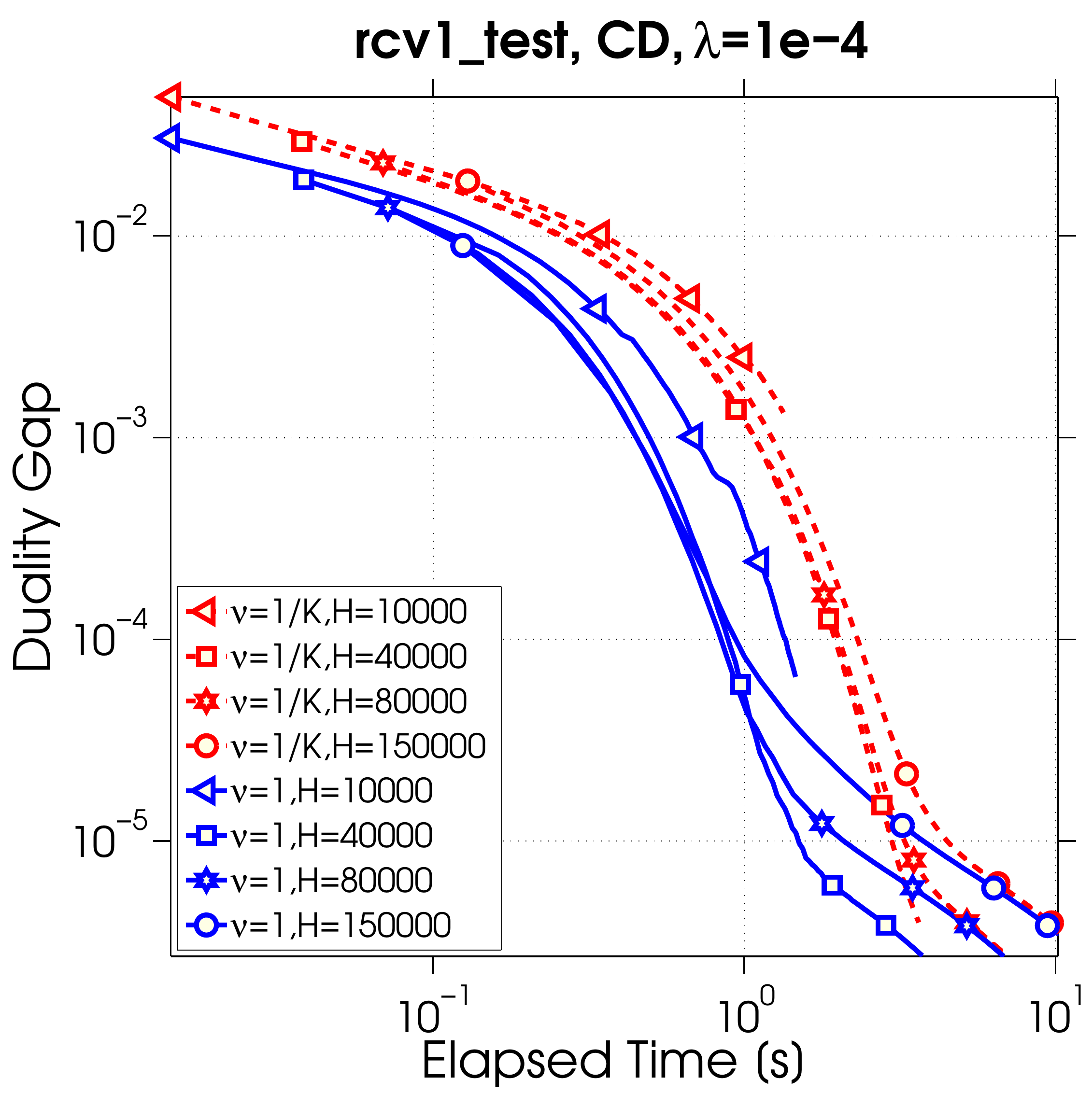}
\includegraphics[scale=.19]{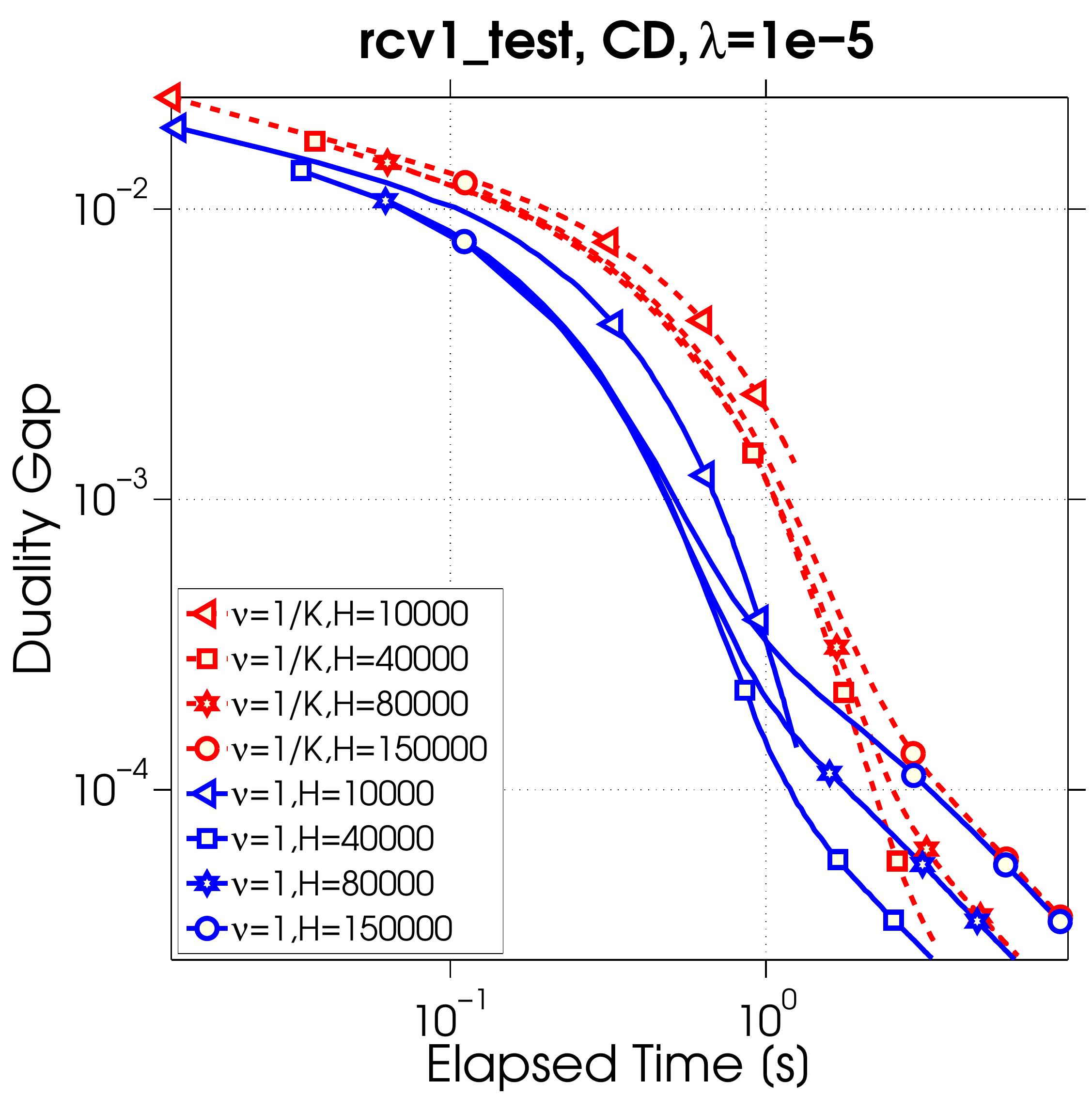}
\caption{Adding (blue solid line) vs Averaging (red dashed line) for \mcx{CD} as the local solver. } 
\label{fig:solver1}
\end{figure}

\begin{figure}[H]
\centering
\includegraphics[scale=.19]{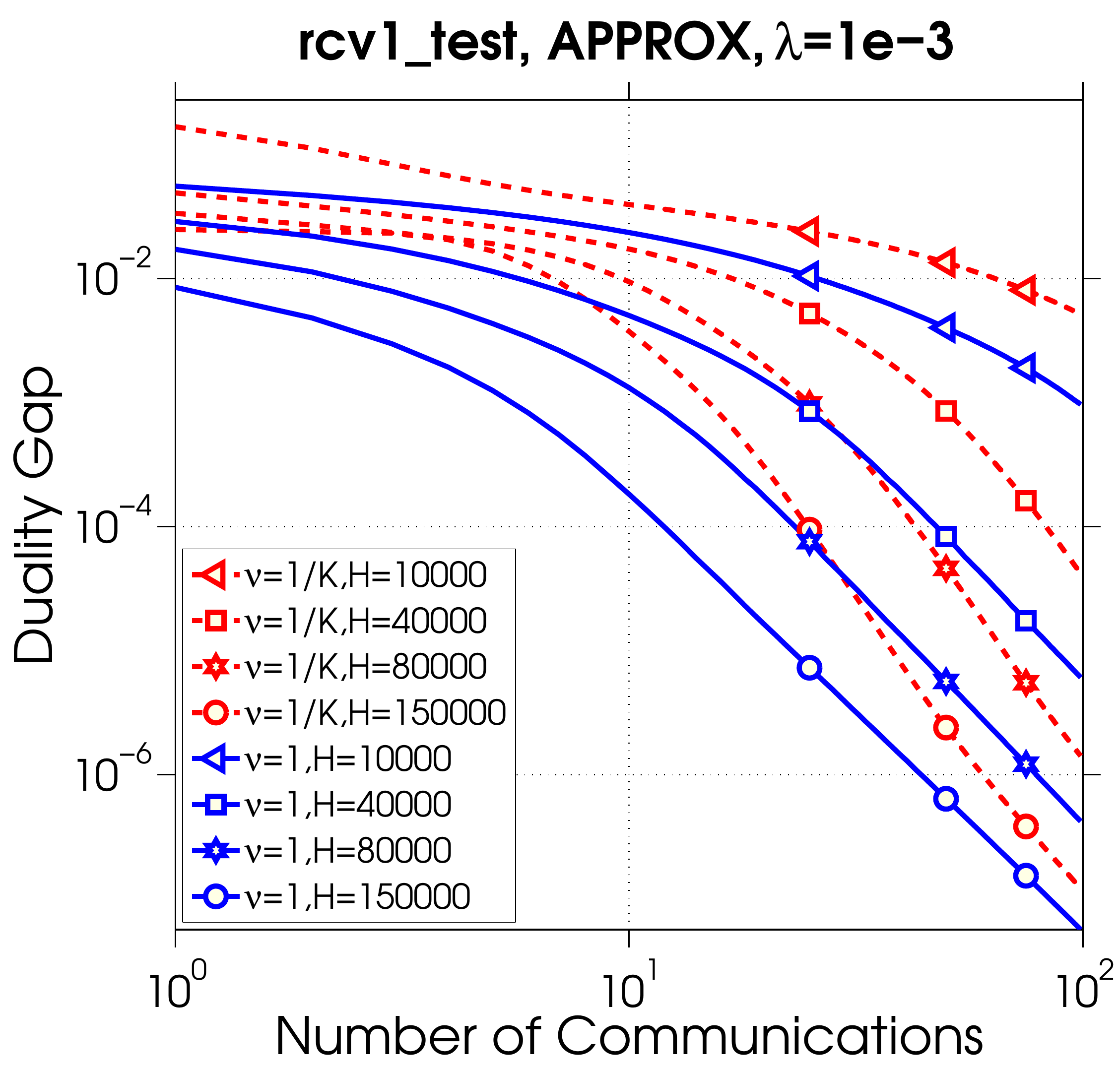}
\includegraphics[scale=.19]{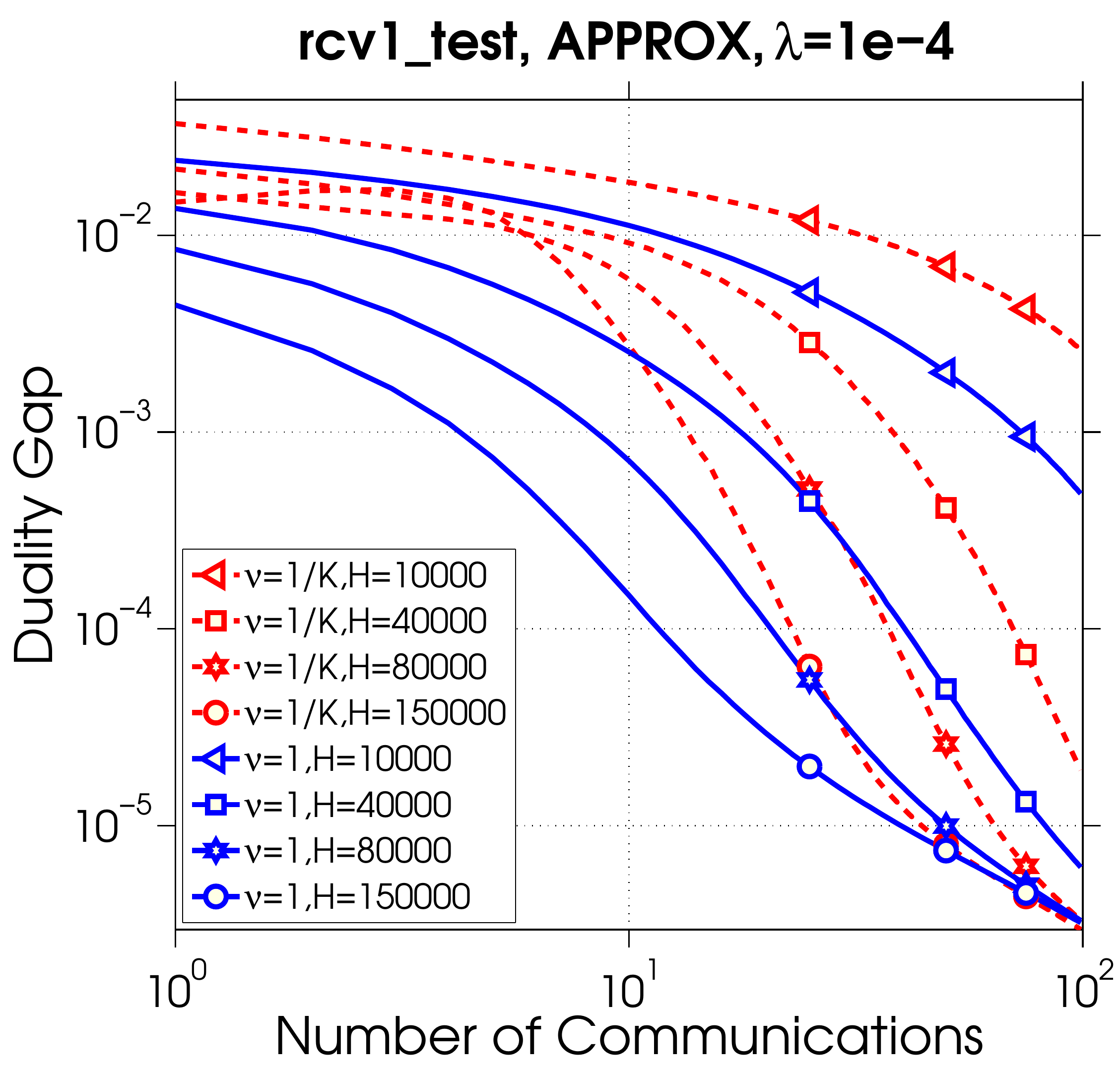}
\includegraphics[scale=.19]{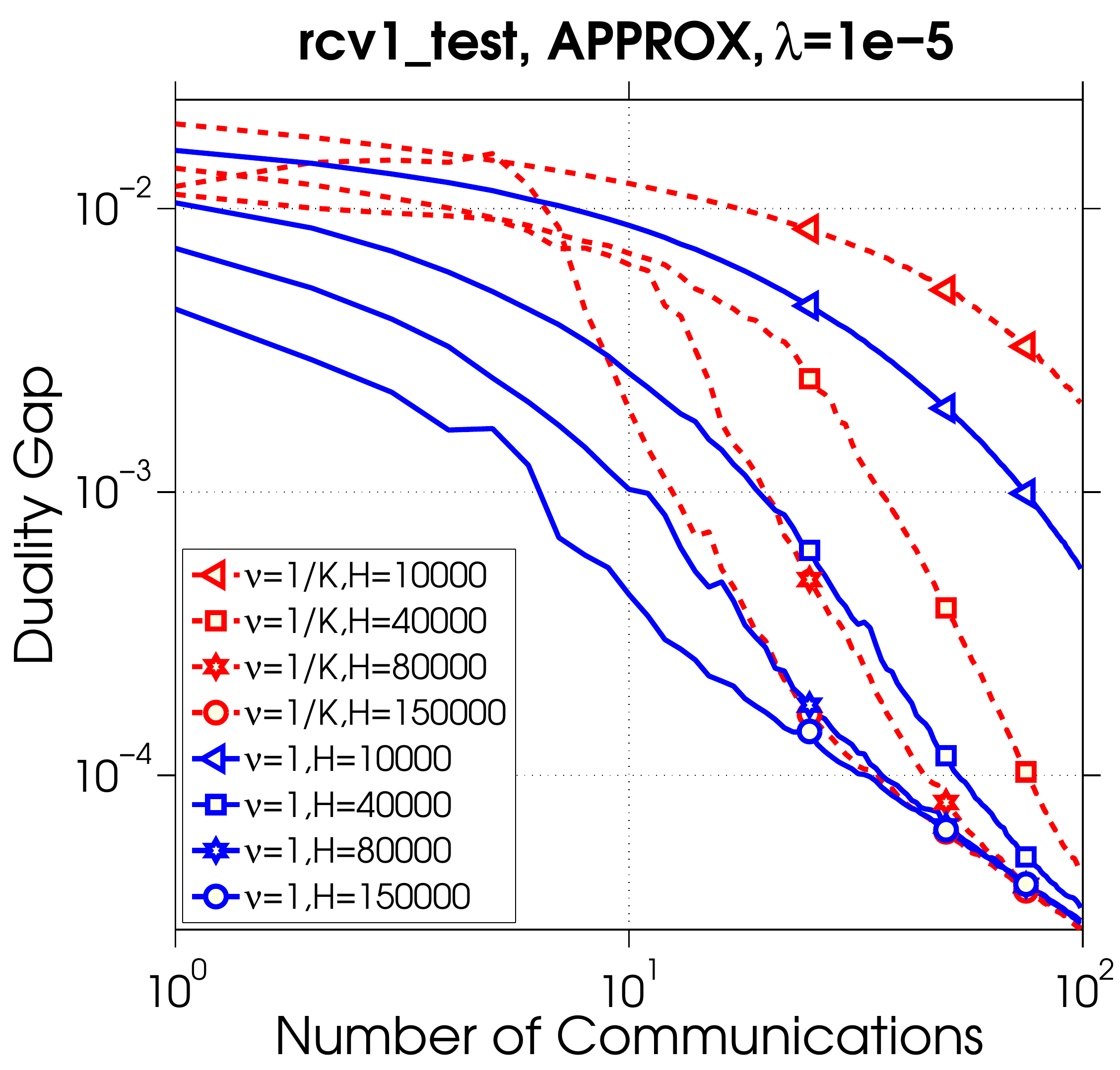}

\includegraphics[scale=.19]{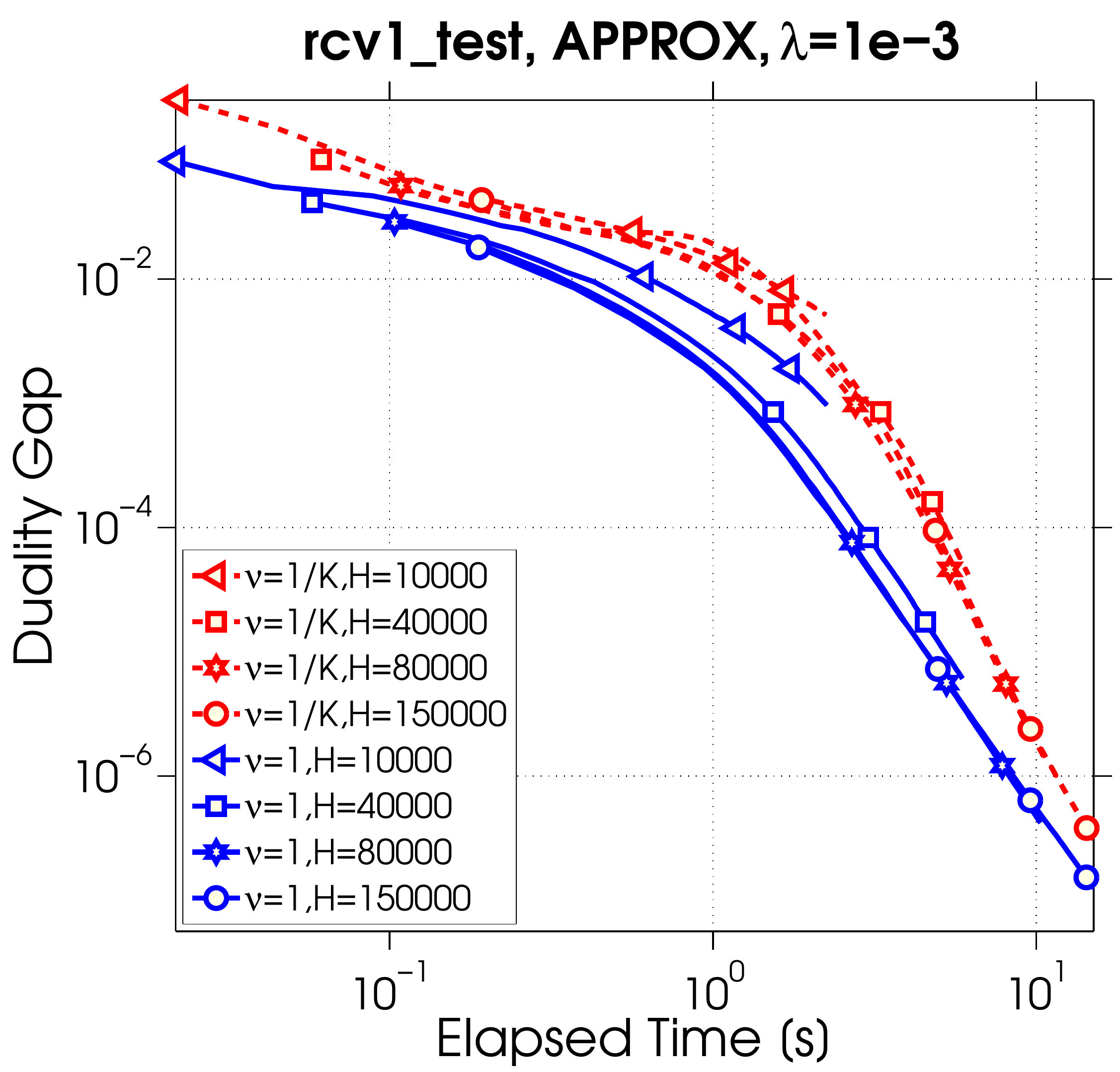}
\includegraphics[scale=.19]{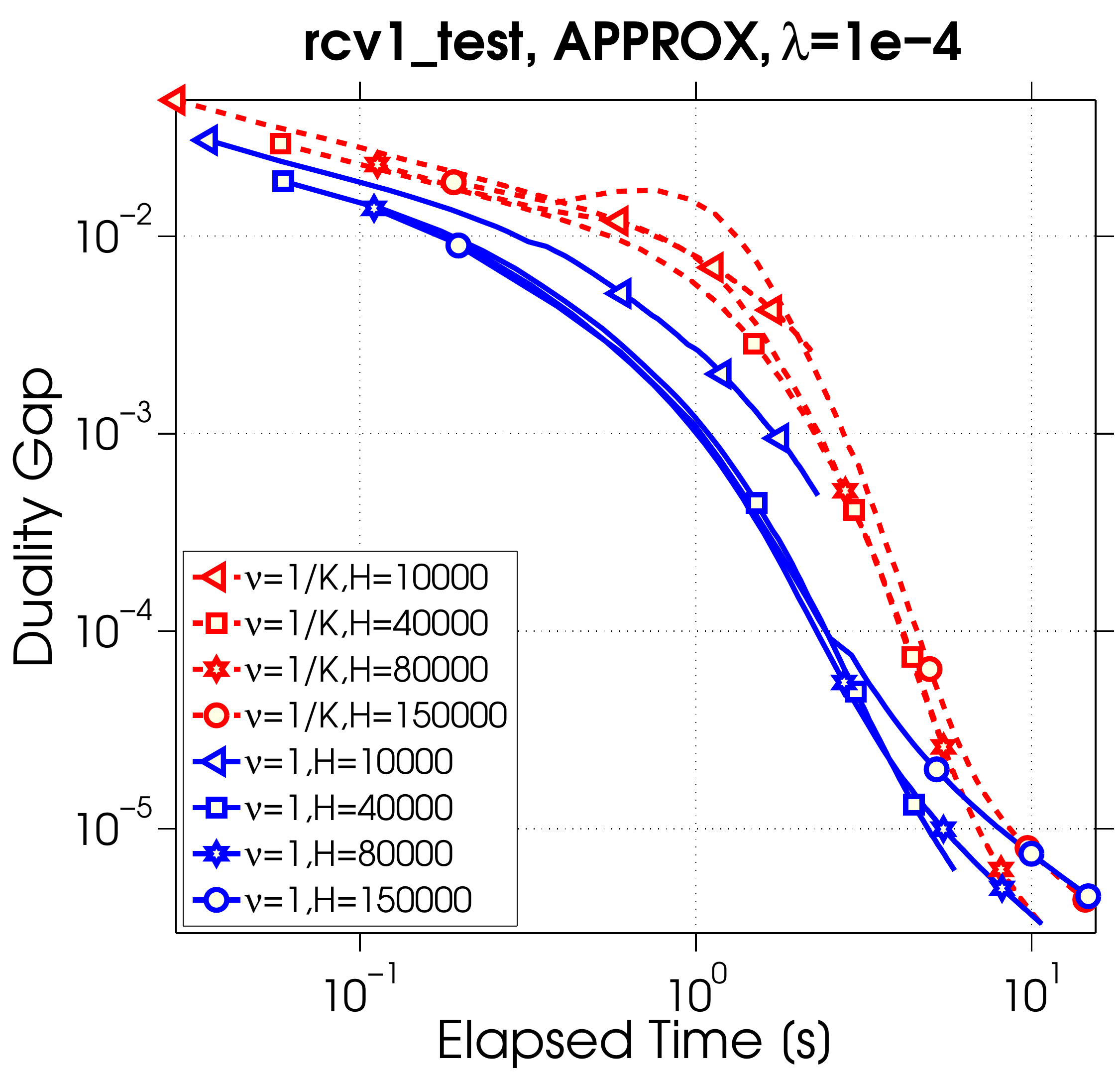}
\includegraphics[scale=.19]{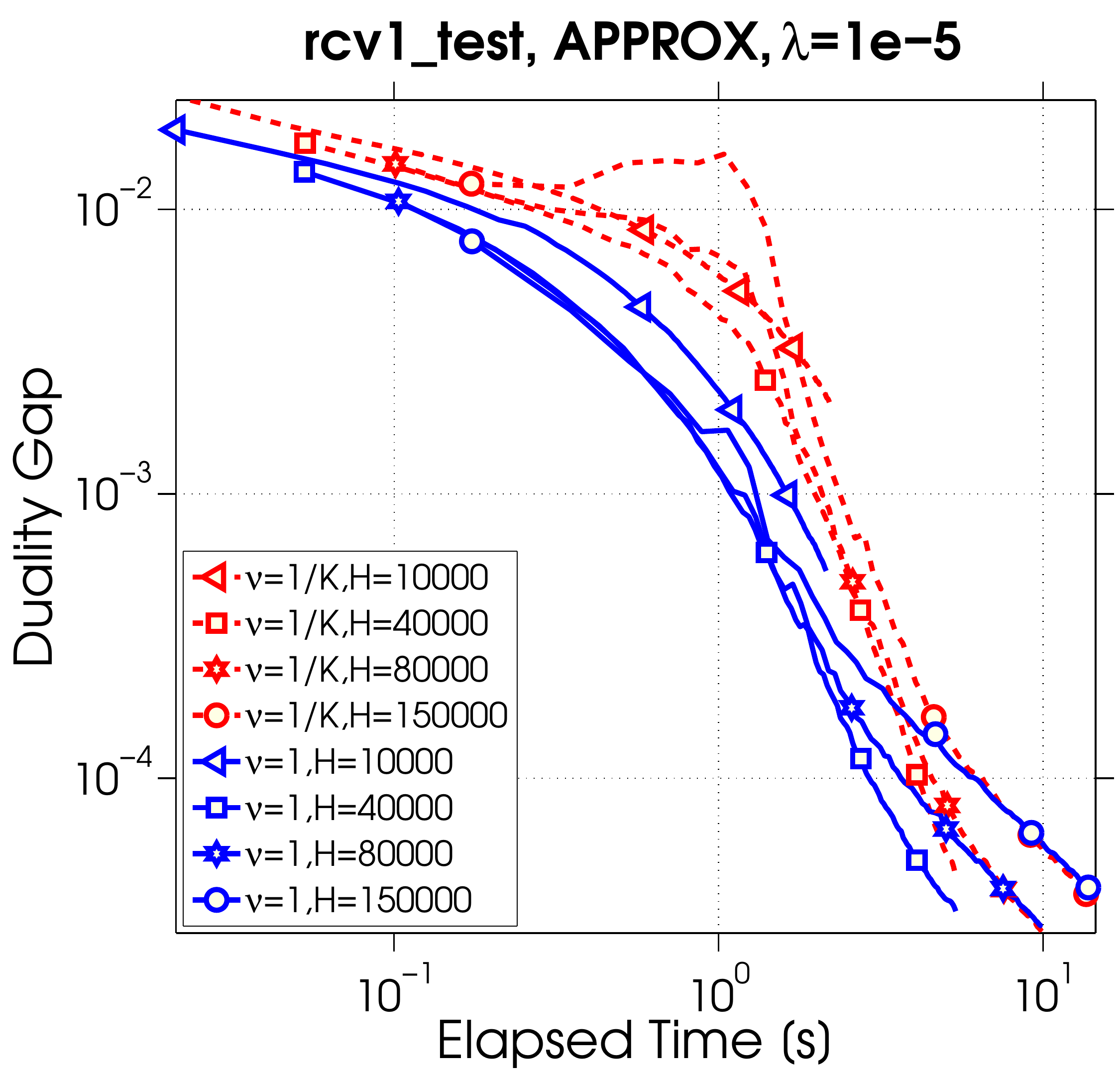}
\caption{Adding (blue solid line) vs Averaging (red dashed line) for APPROX as the local solver.} 
\label{fig:soler2}
\end{figure}

\begin{figure}[H]
\centering
\includegraphics[scale=.19]{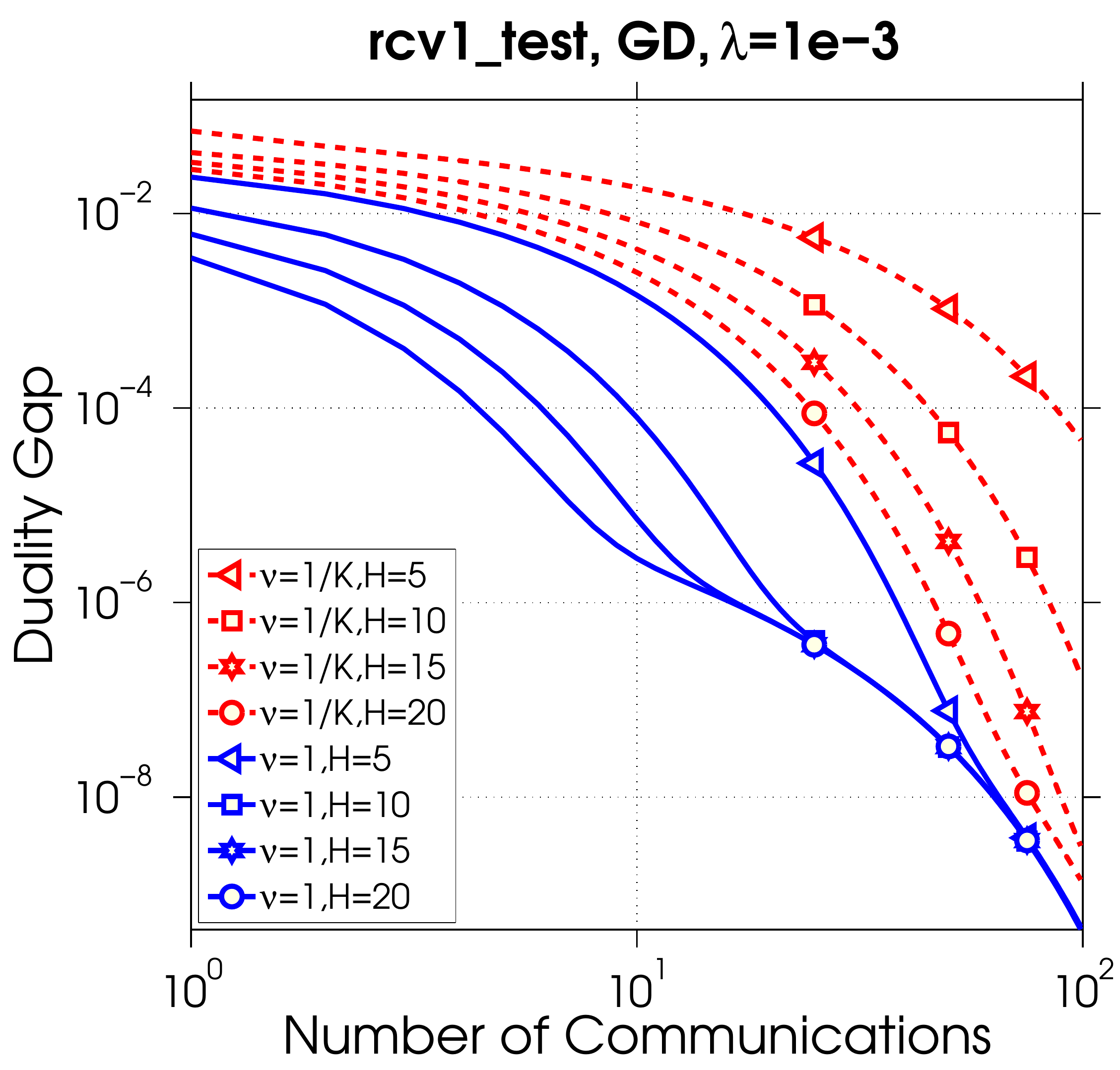}
\includegraphics[scale=.19]{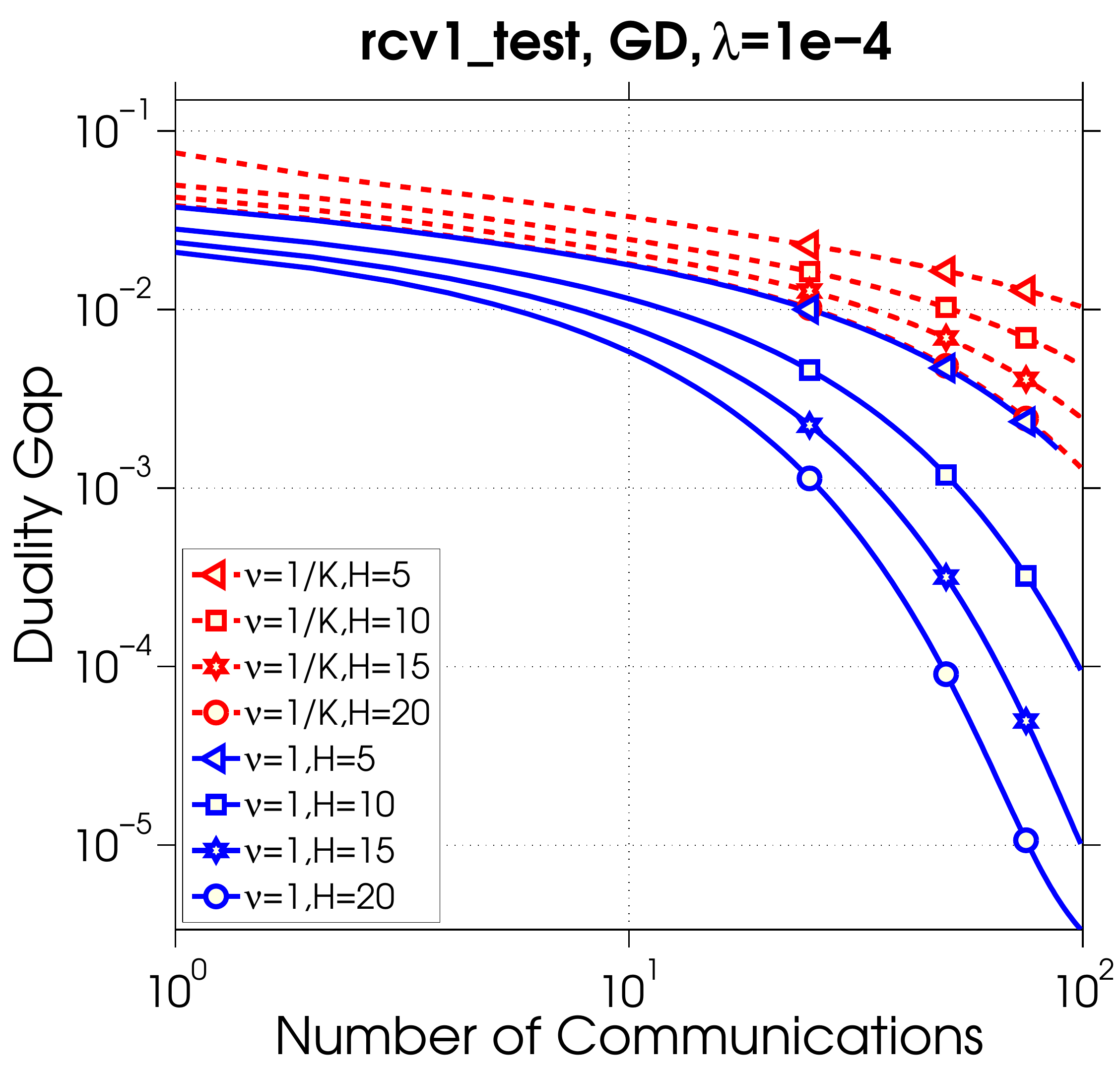}
\includegraphics[scale=.19]{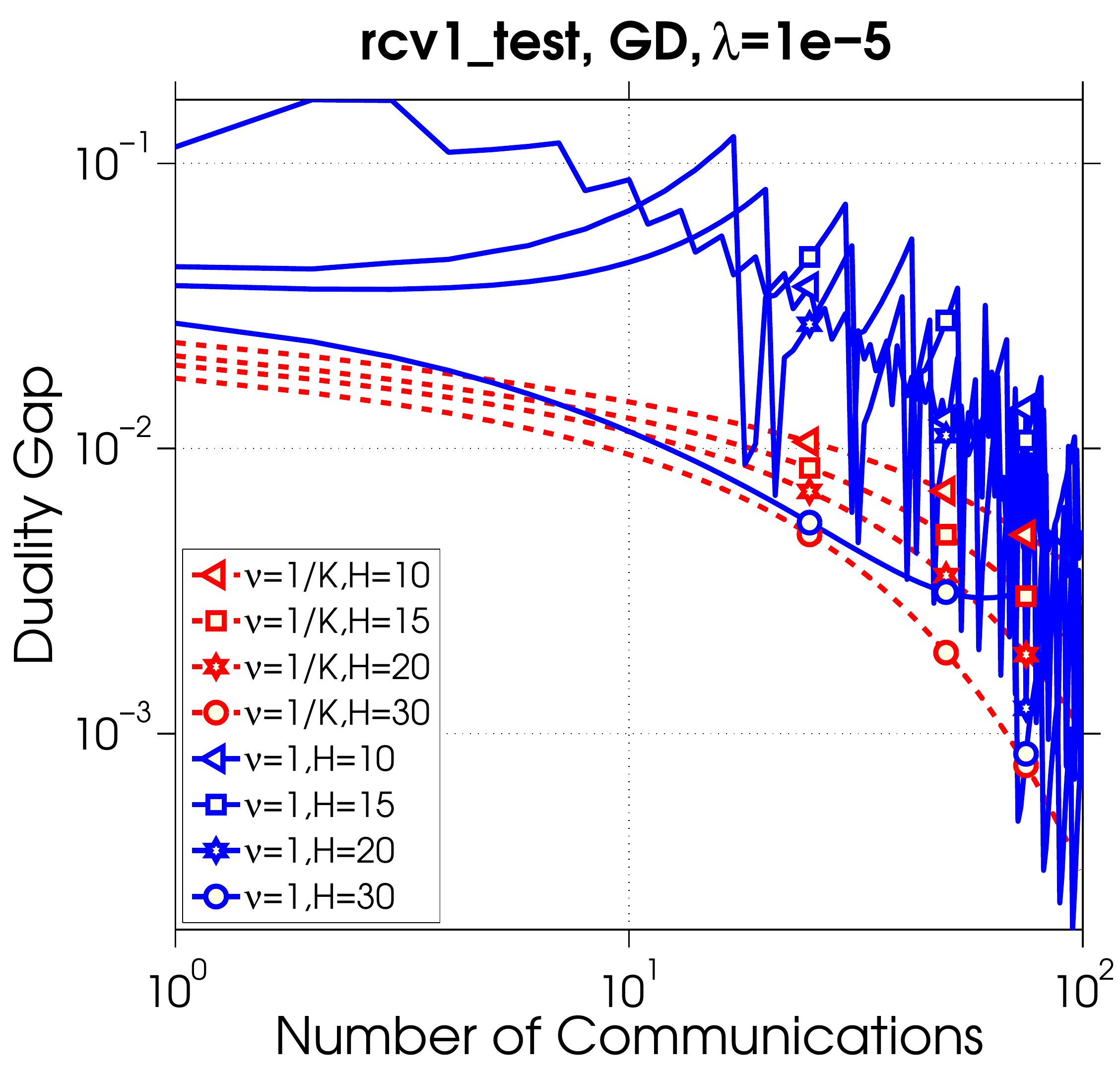}

\includegraphics[scale=.19]{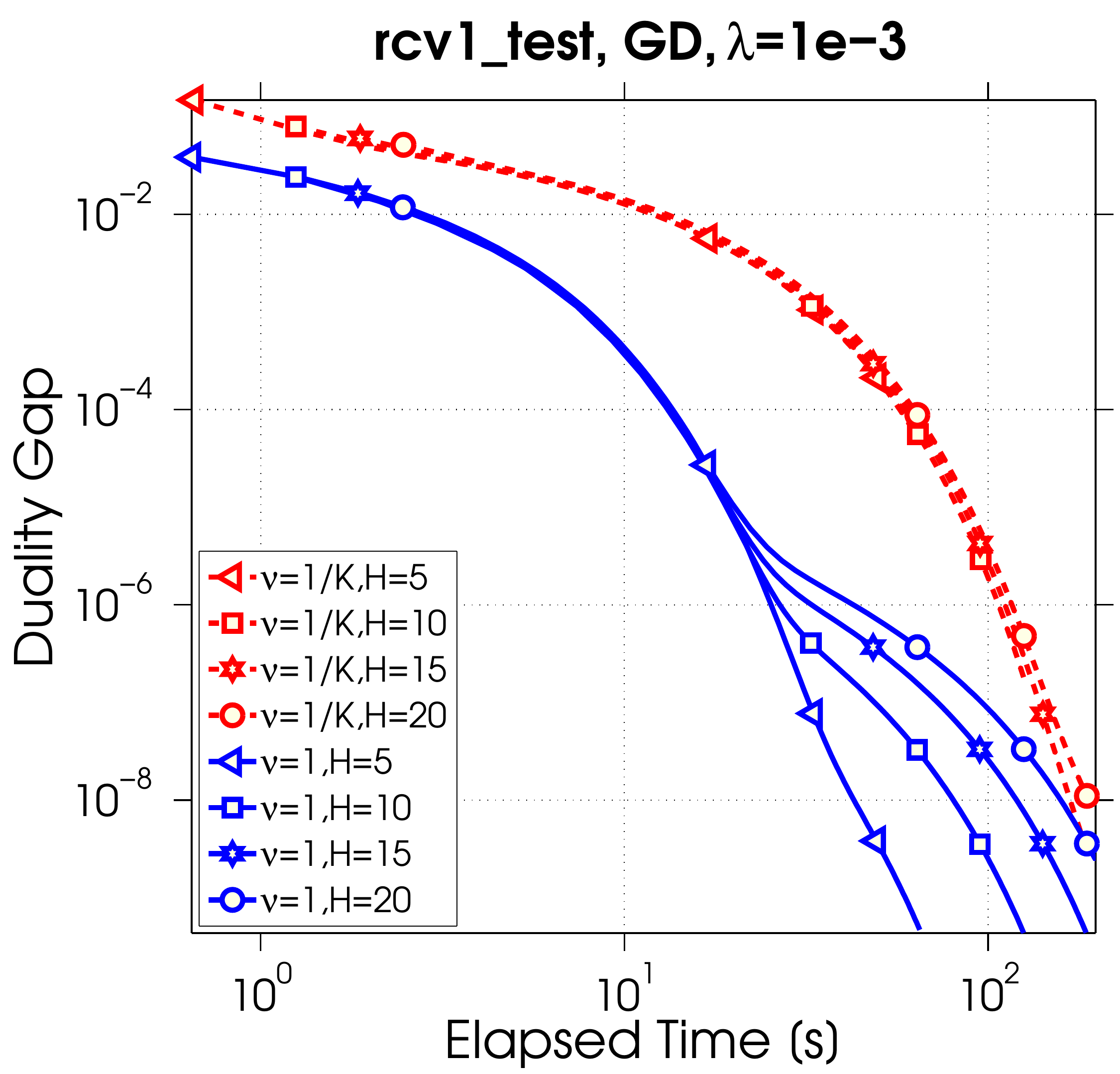}
\includegraphics[scale=.19]{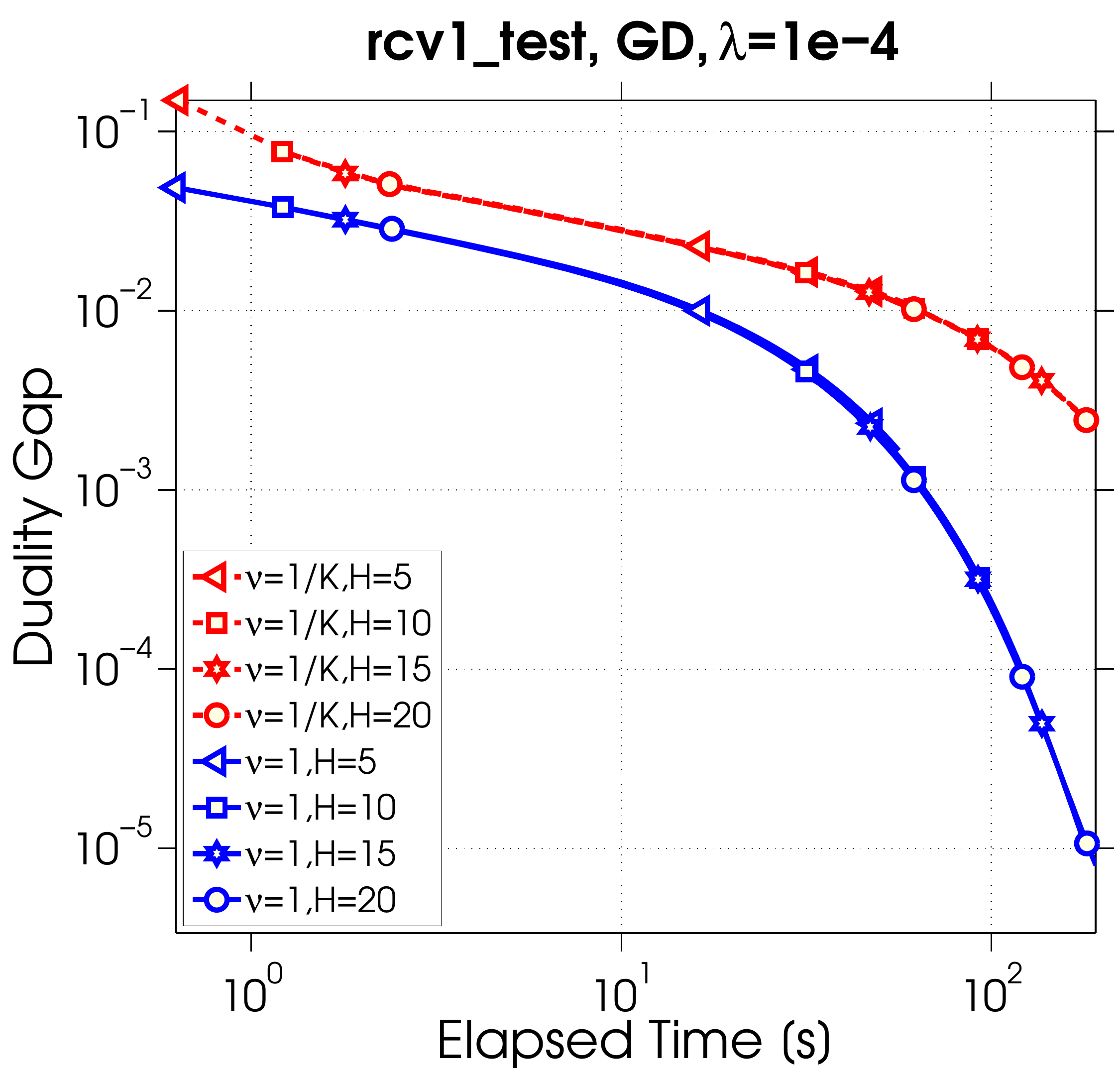}
\includegraphics[scale=.19]{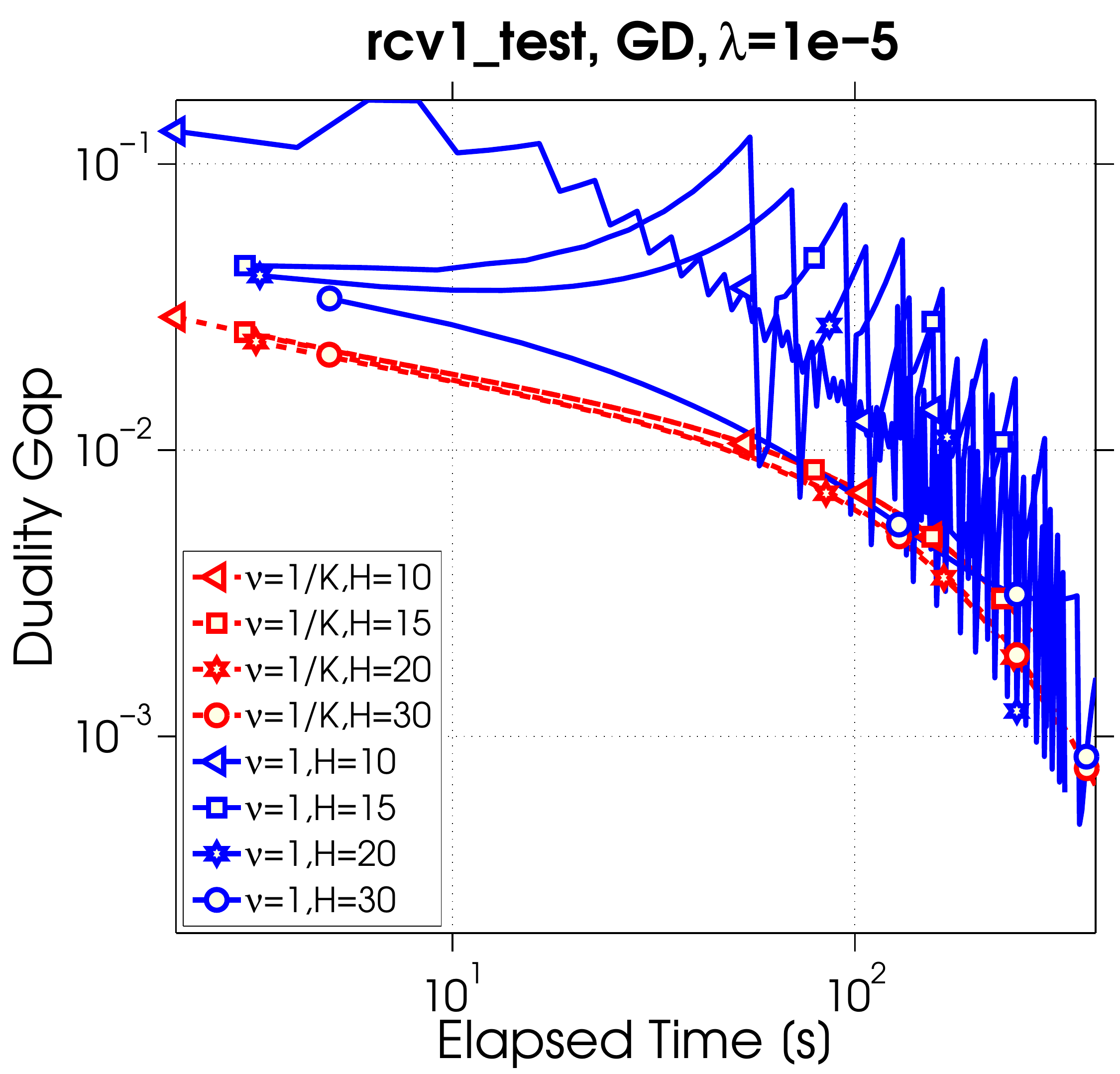}
\caption{Adding (blue solid line) vs Averaging (red dashed line) for Gradient Descent as the local solver.} 
\label{fig:soler3}
\end{figure}

\begin{figure}[H]
\centering
\includegraphics[scale=.19]{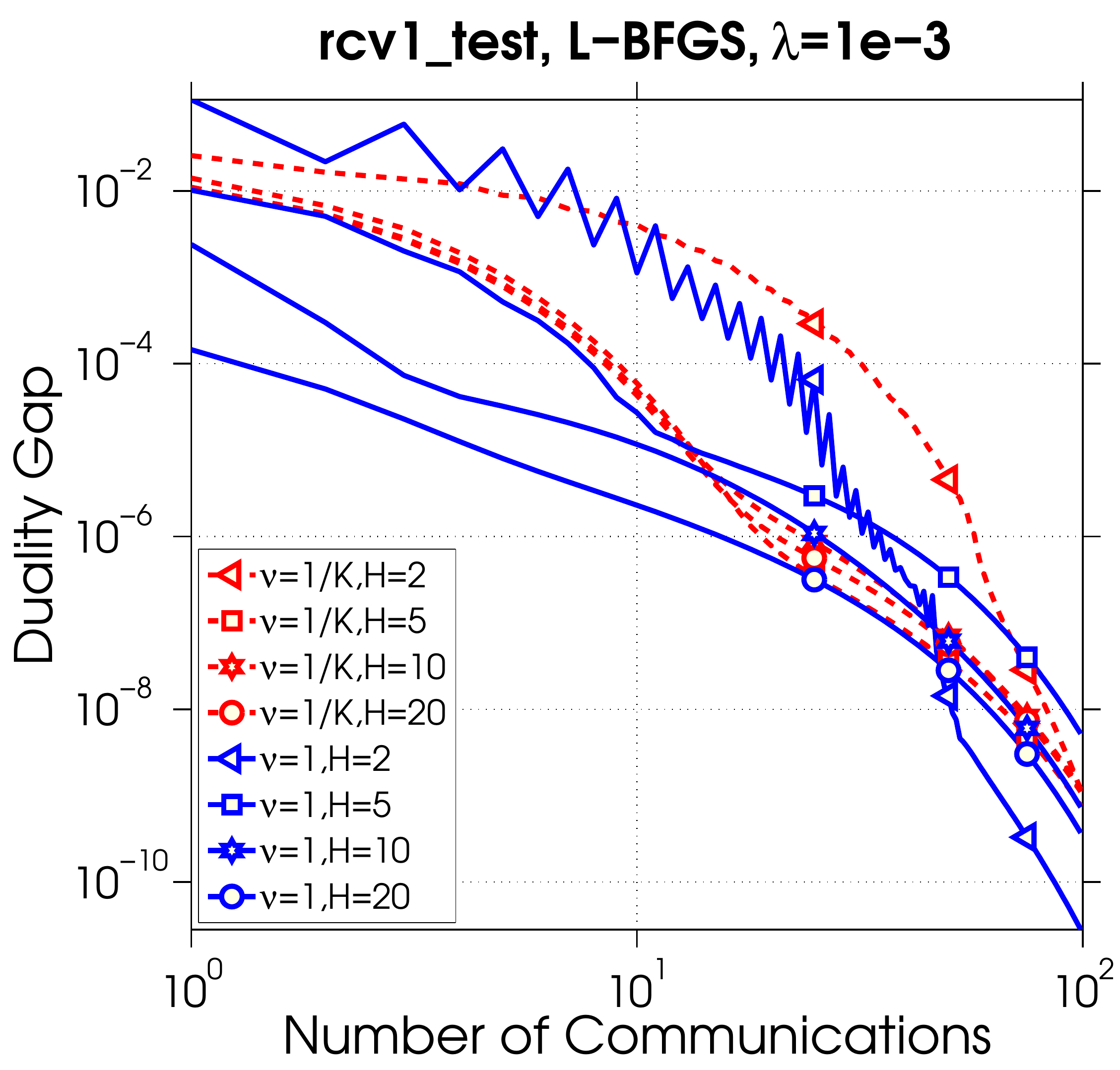}
\includegraphics[scale=.19]{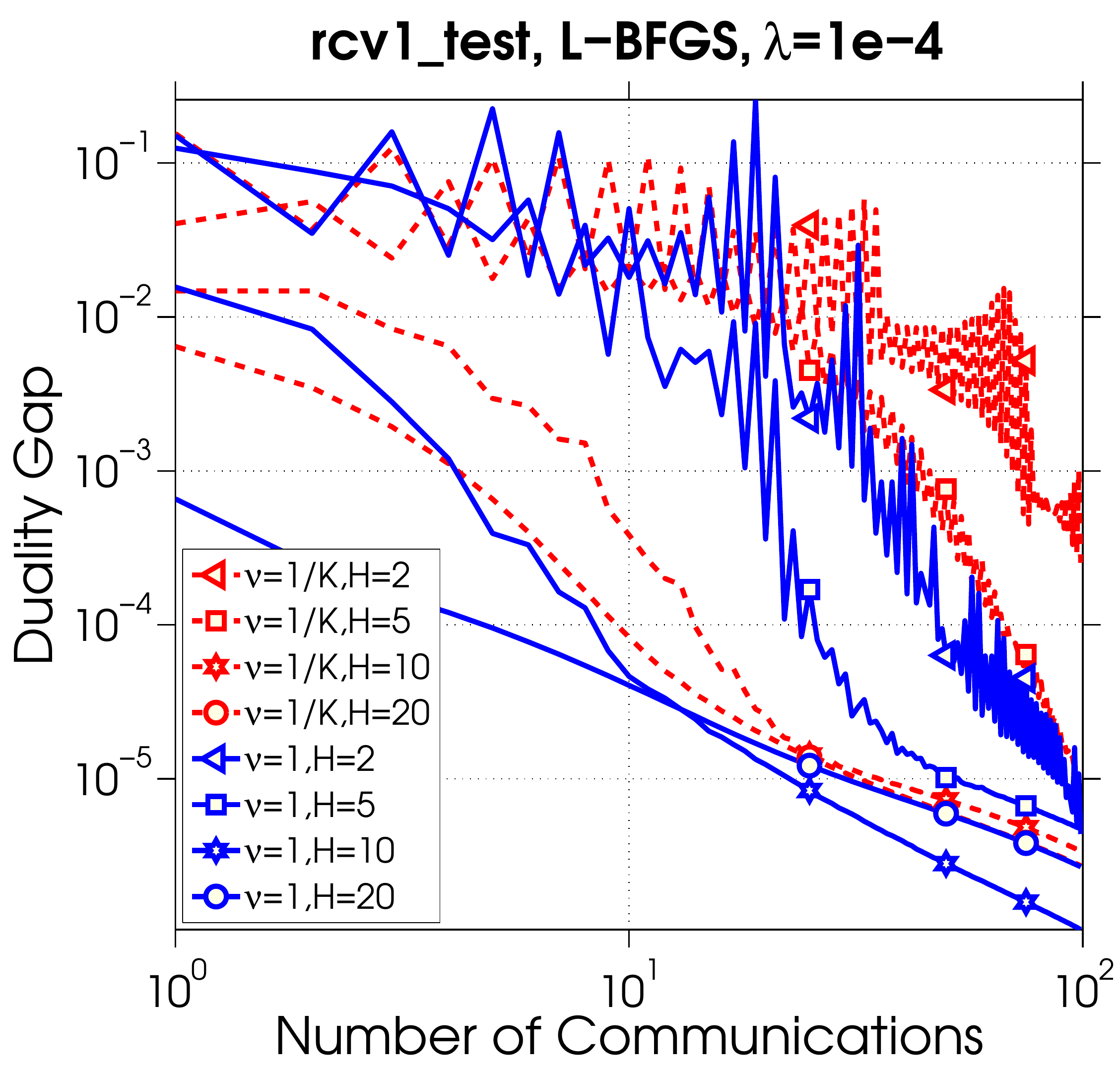}
\includegraphics[scale=.19]{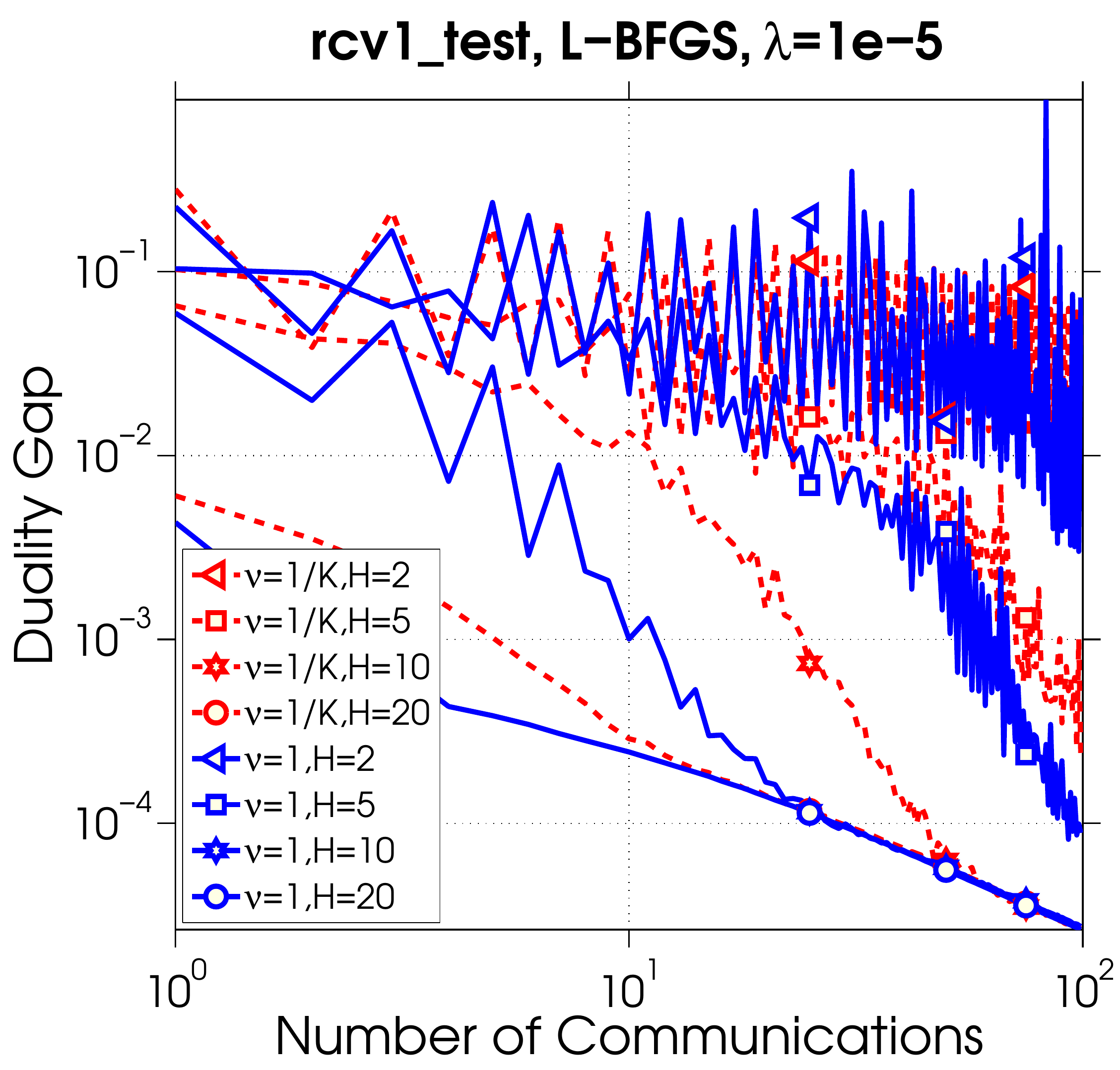}

\includegraphics[scale=.19]{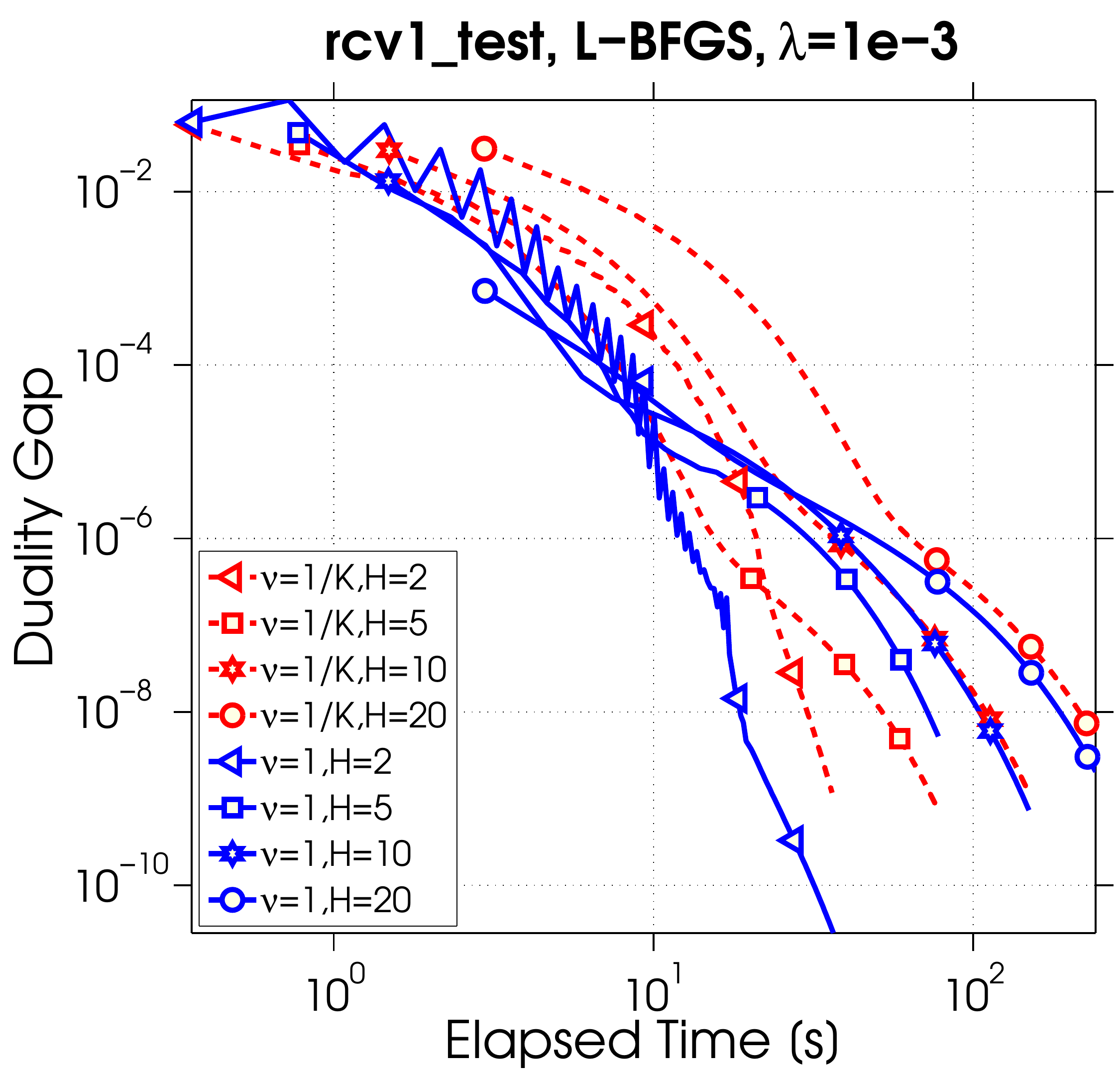}
\includegraphics[scale=.19]{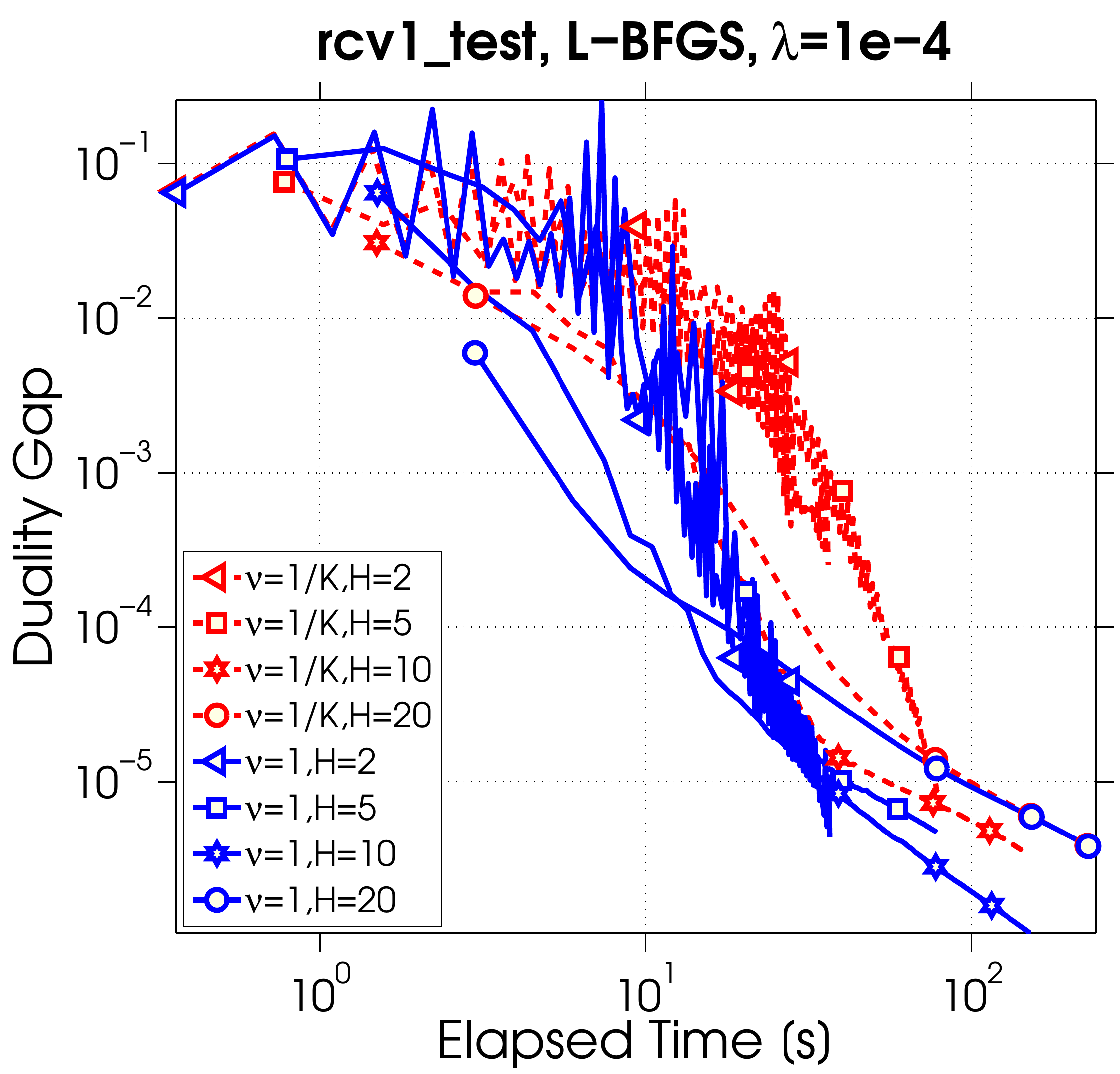}
\includegraphics[scale=.19]{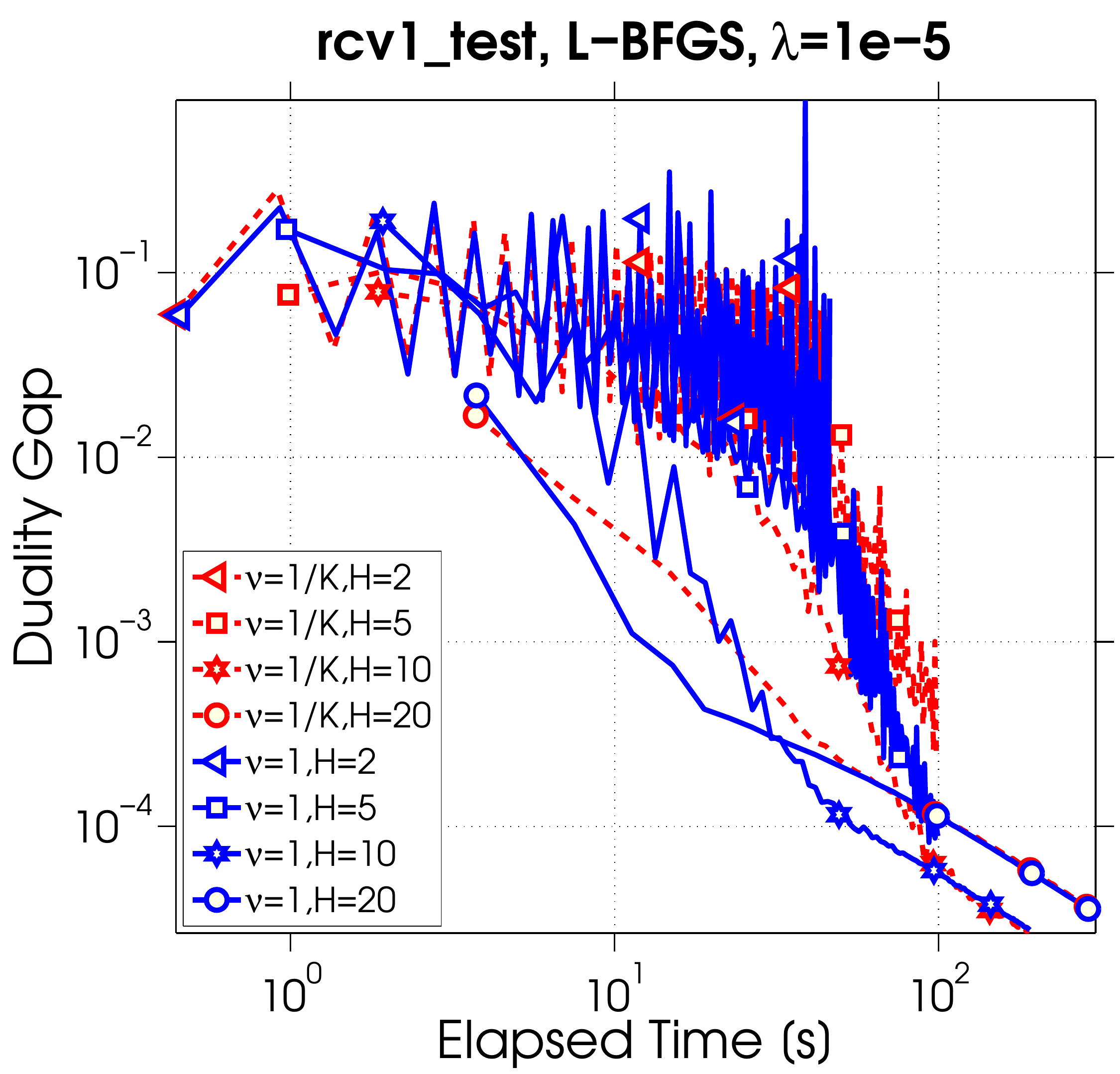}
\caption{Adding (blue solid line) vs Averaging (red dashed line) for L-BFGS as the local solver.} 
\label{fig:soler4}
\end{figure}

\begin{figure}[H]
\centering
\includegraphics[scale=.19]{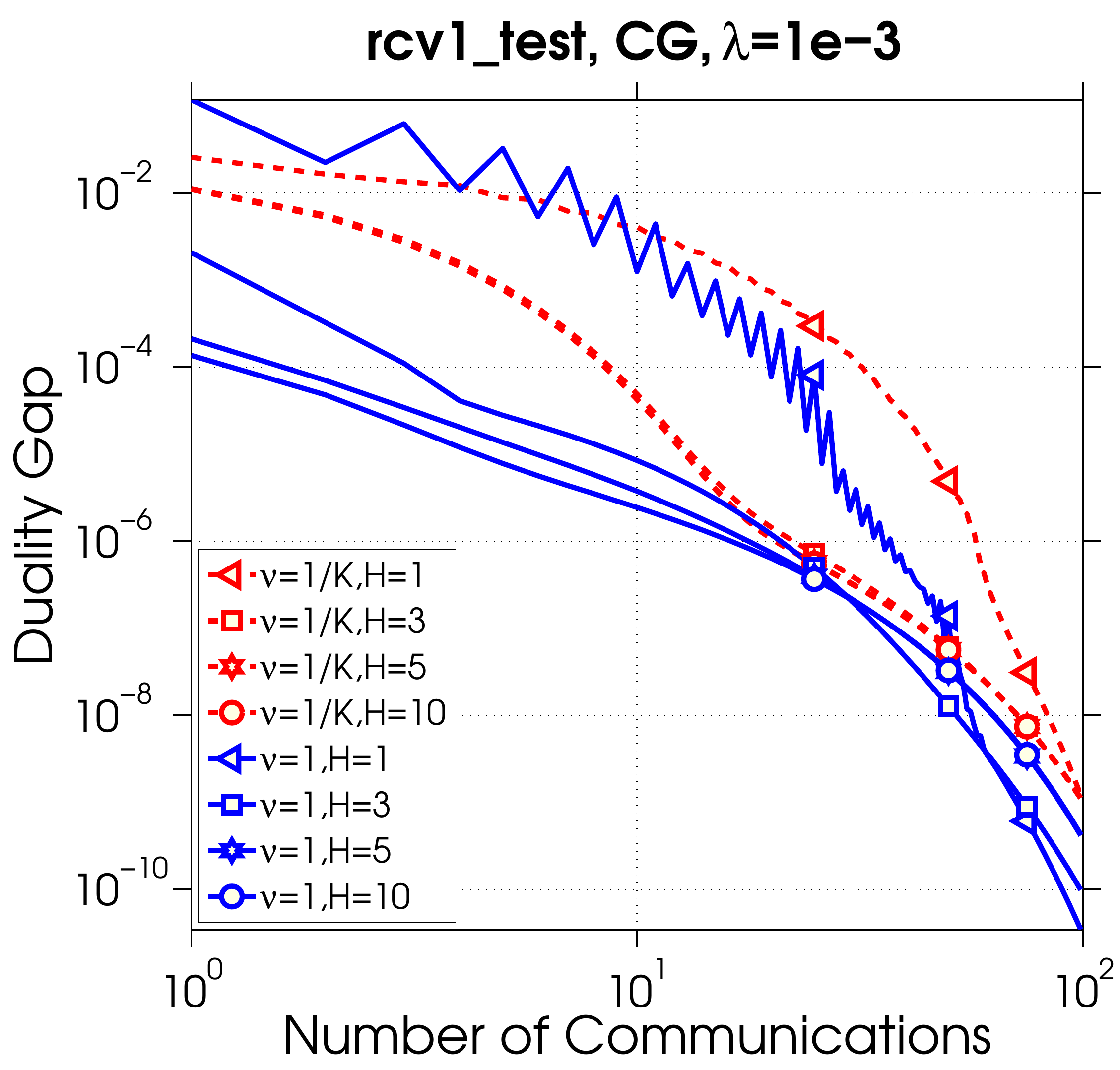}
\includegraphics[scale=.19]{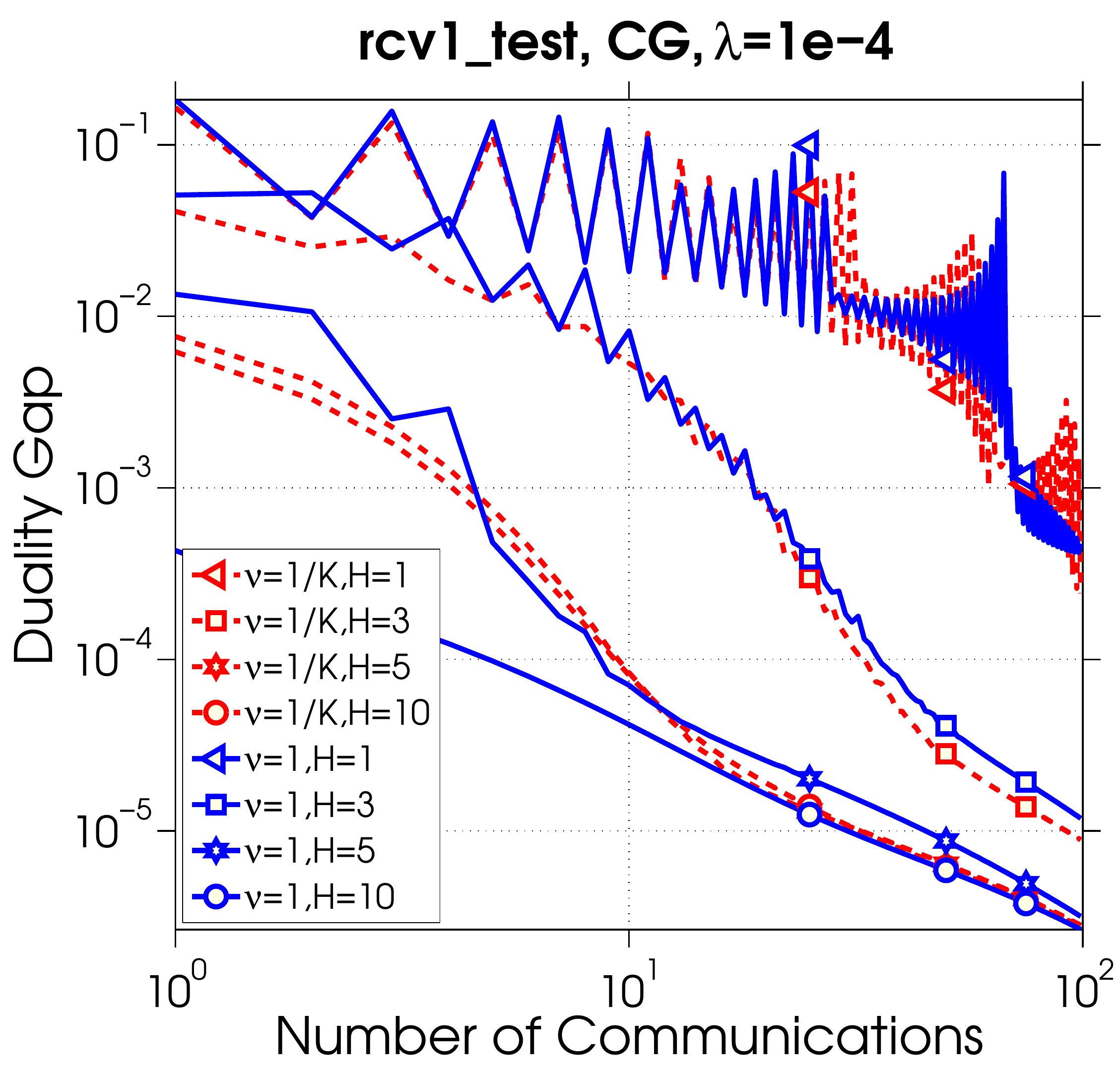}
\includegraphics[scale=.19]{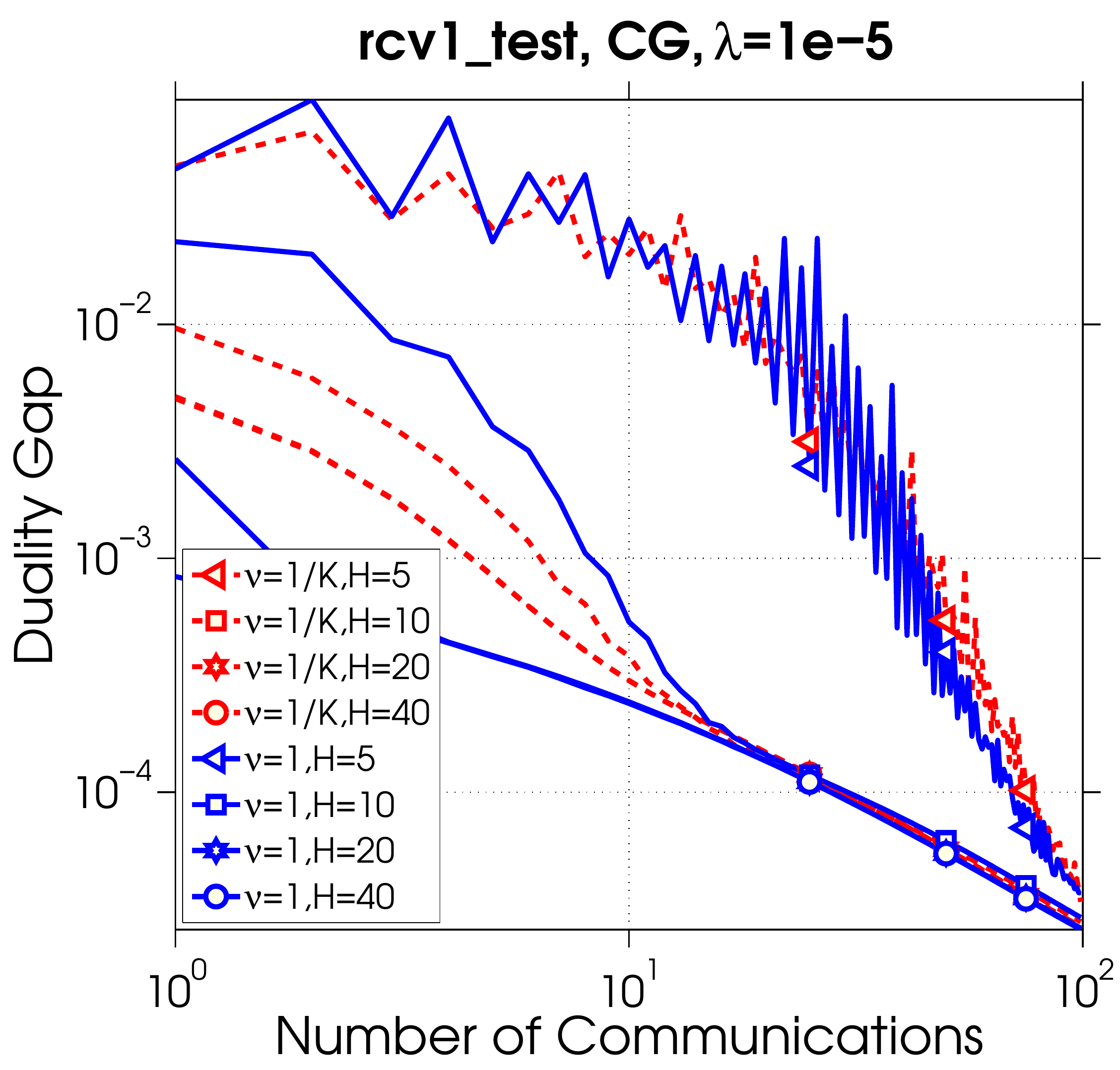}

\includegraphics[scale=.19]{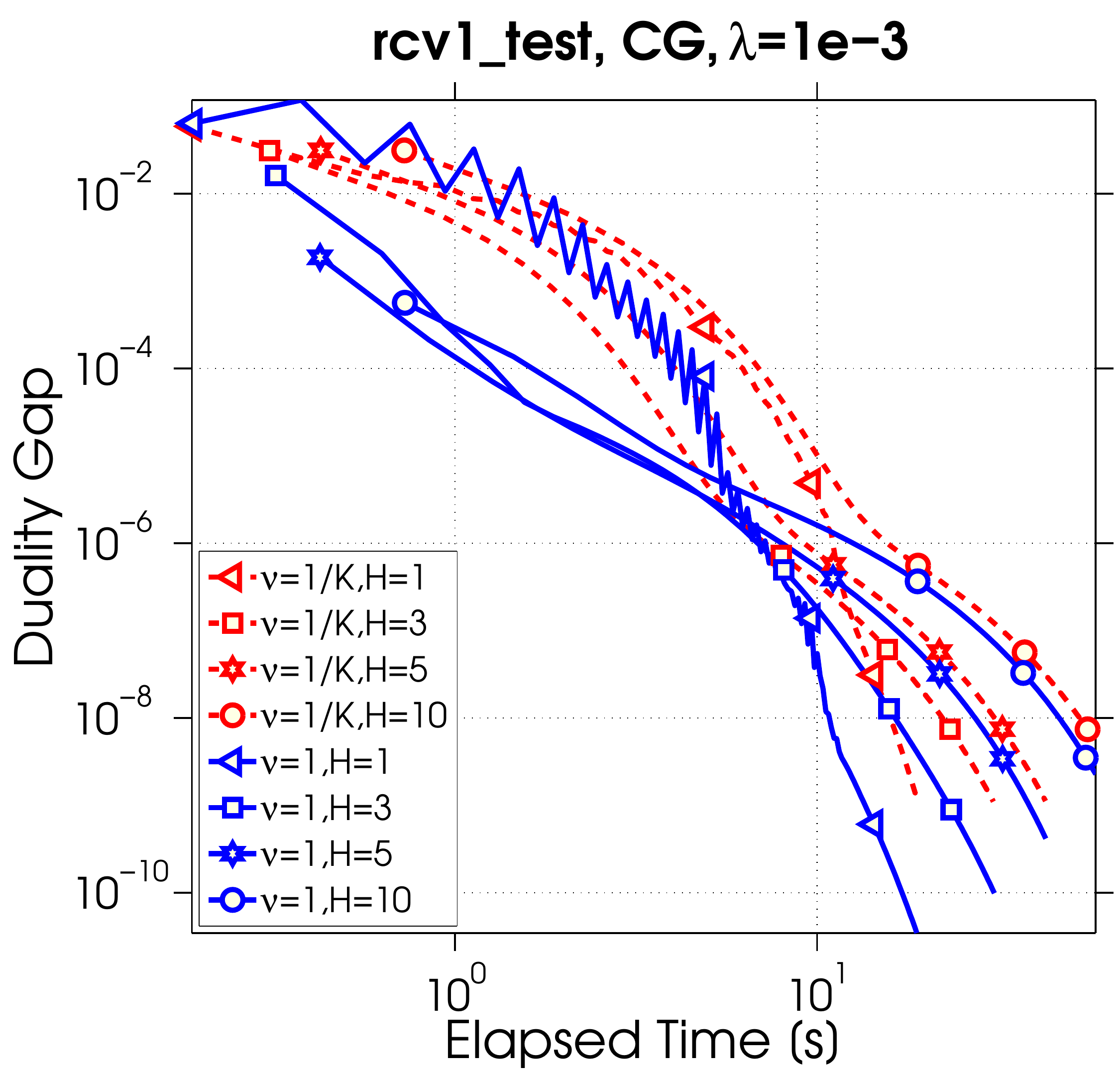}
\includegraphics[scale=.19]{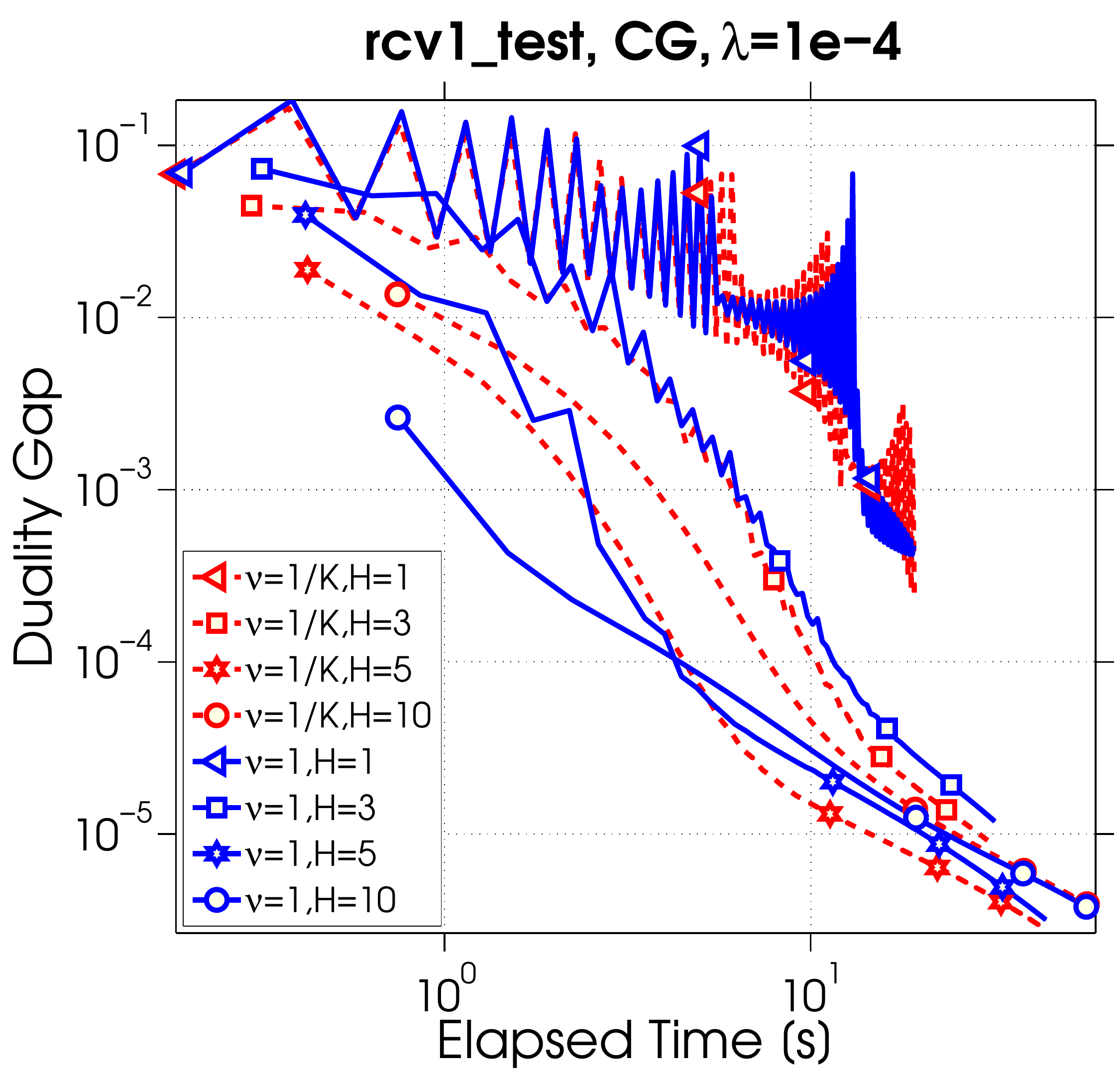}
\includegraphics[scale=.19]{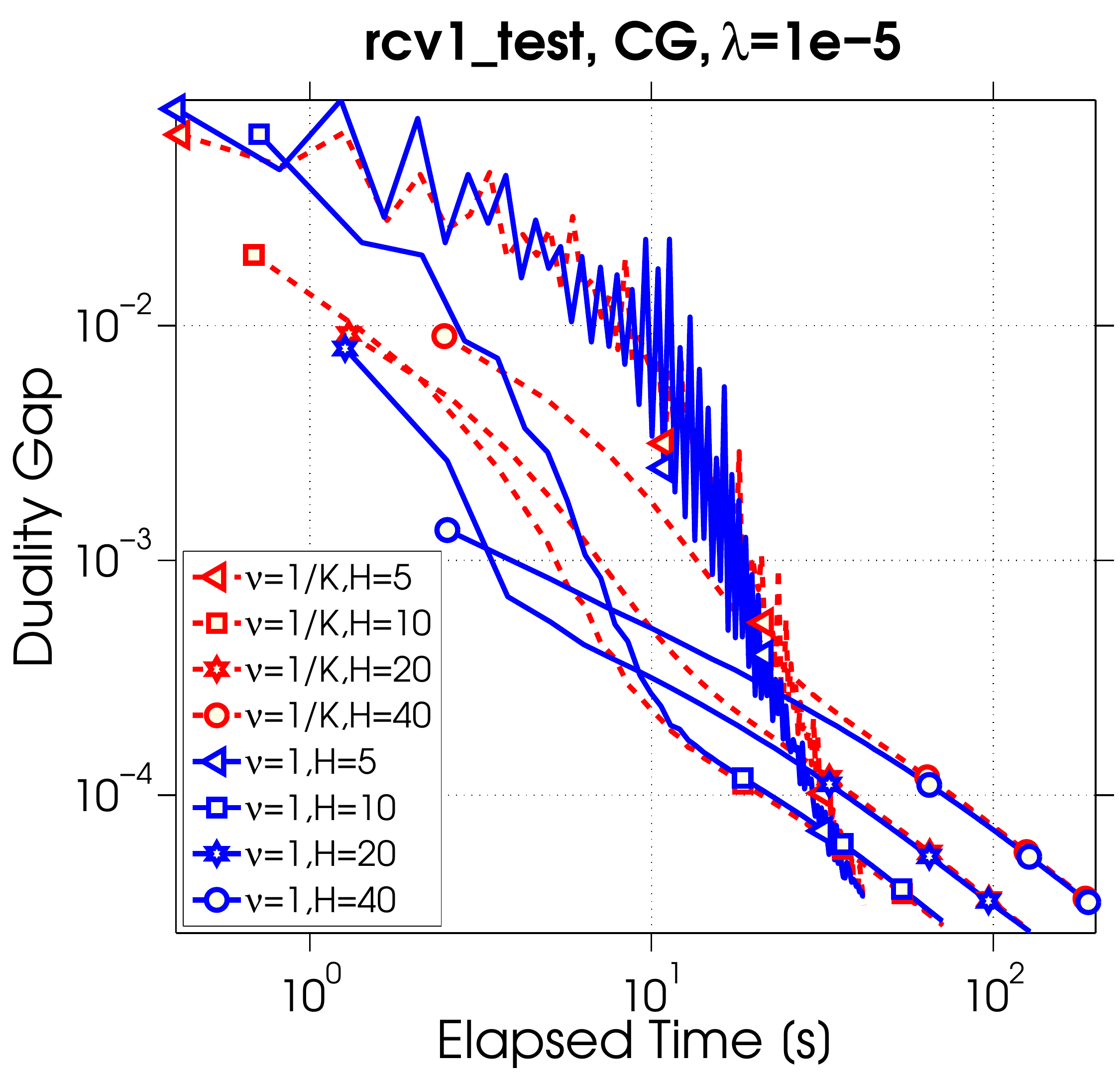}
\caption{Adding (blue solid line) vs Averaging (red dashed line) for Conjugate Gradient Method as the local solver.} 
\label{fig:soler5}
\end{figure}

\begin{figure}[H]
\centering
\includegraphics[scale=.19]{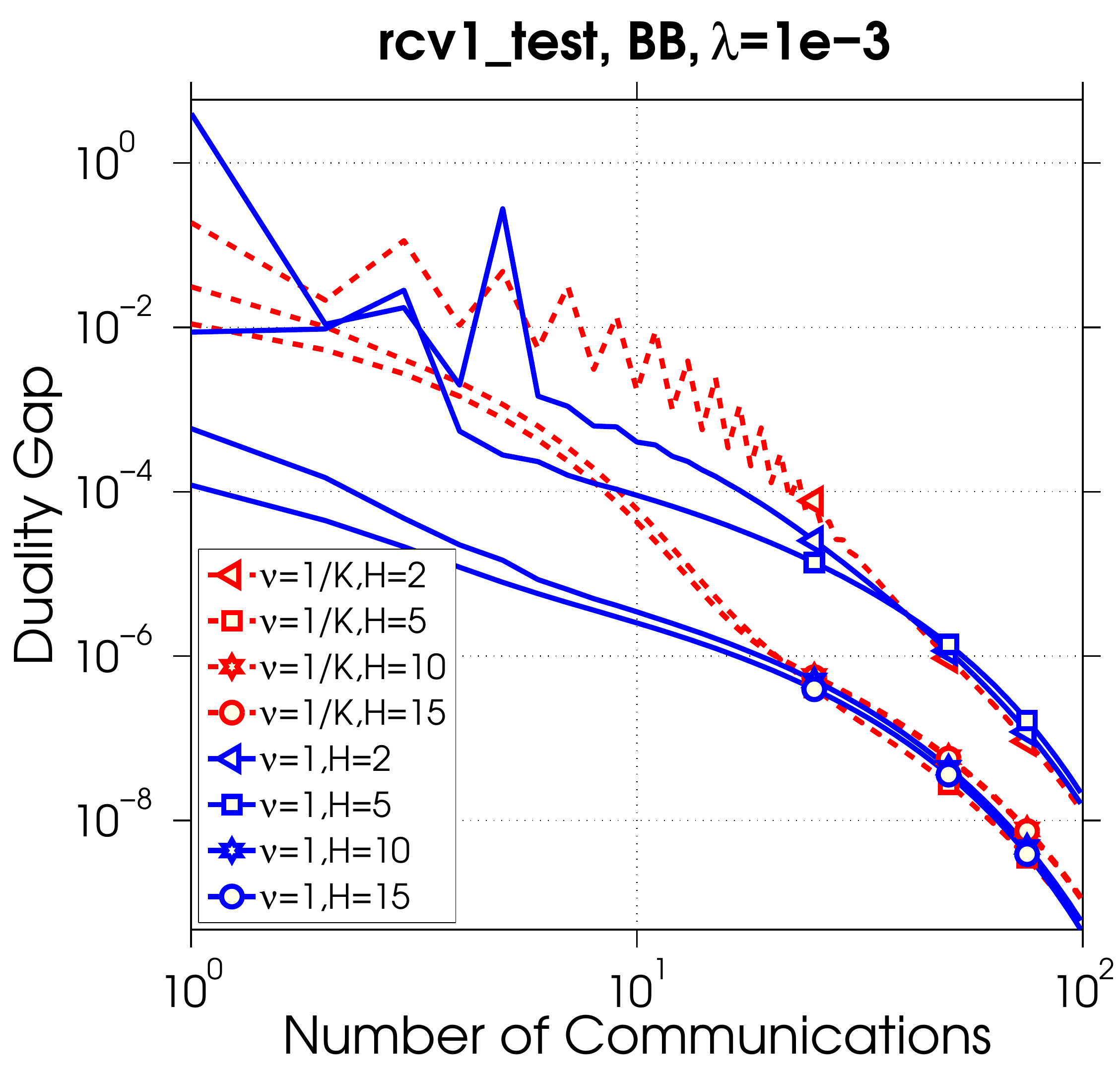}
\includegraphics[scale=.19]{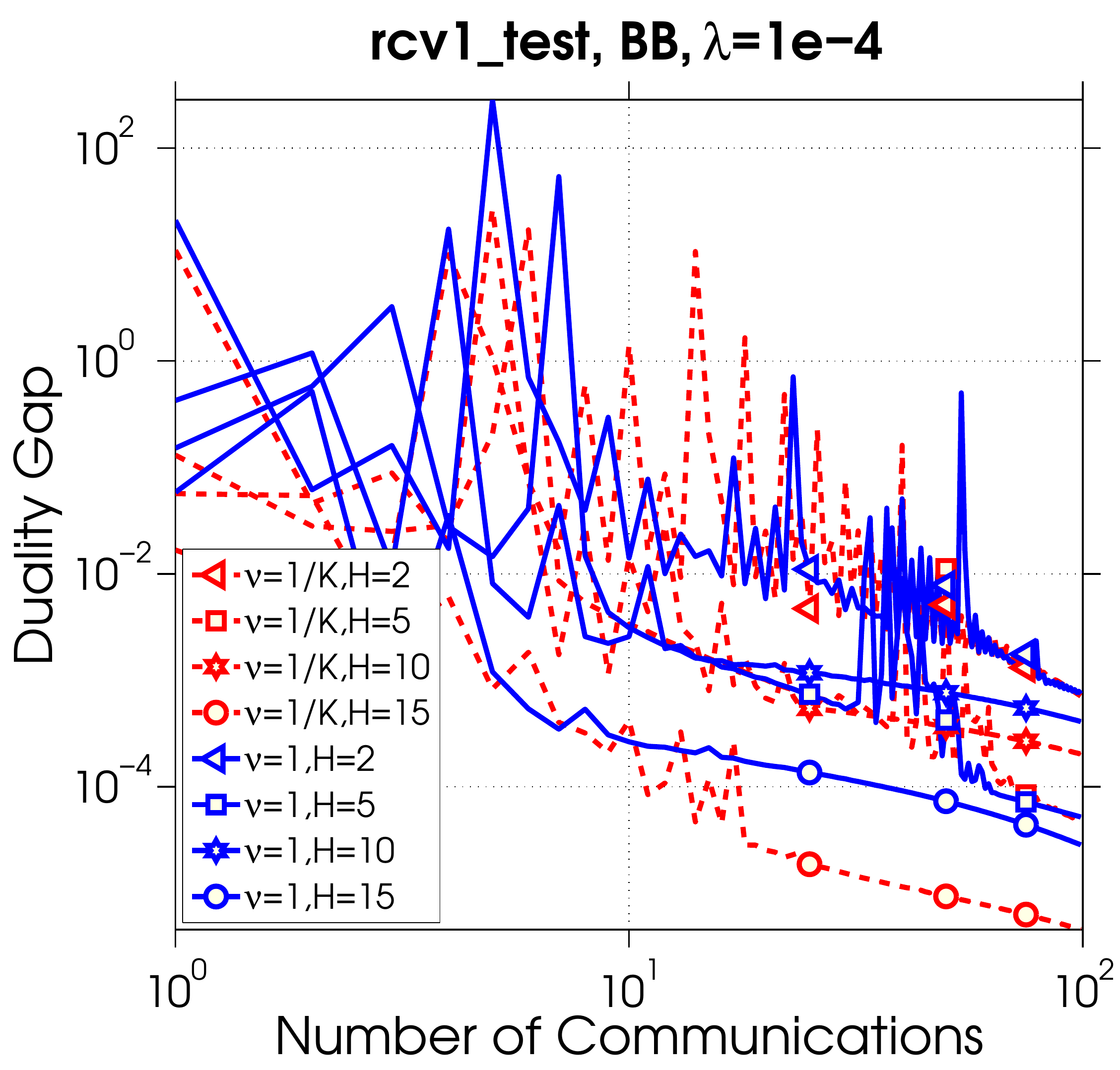}
\includegraphics[scale=.19]{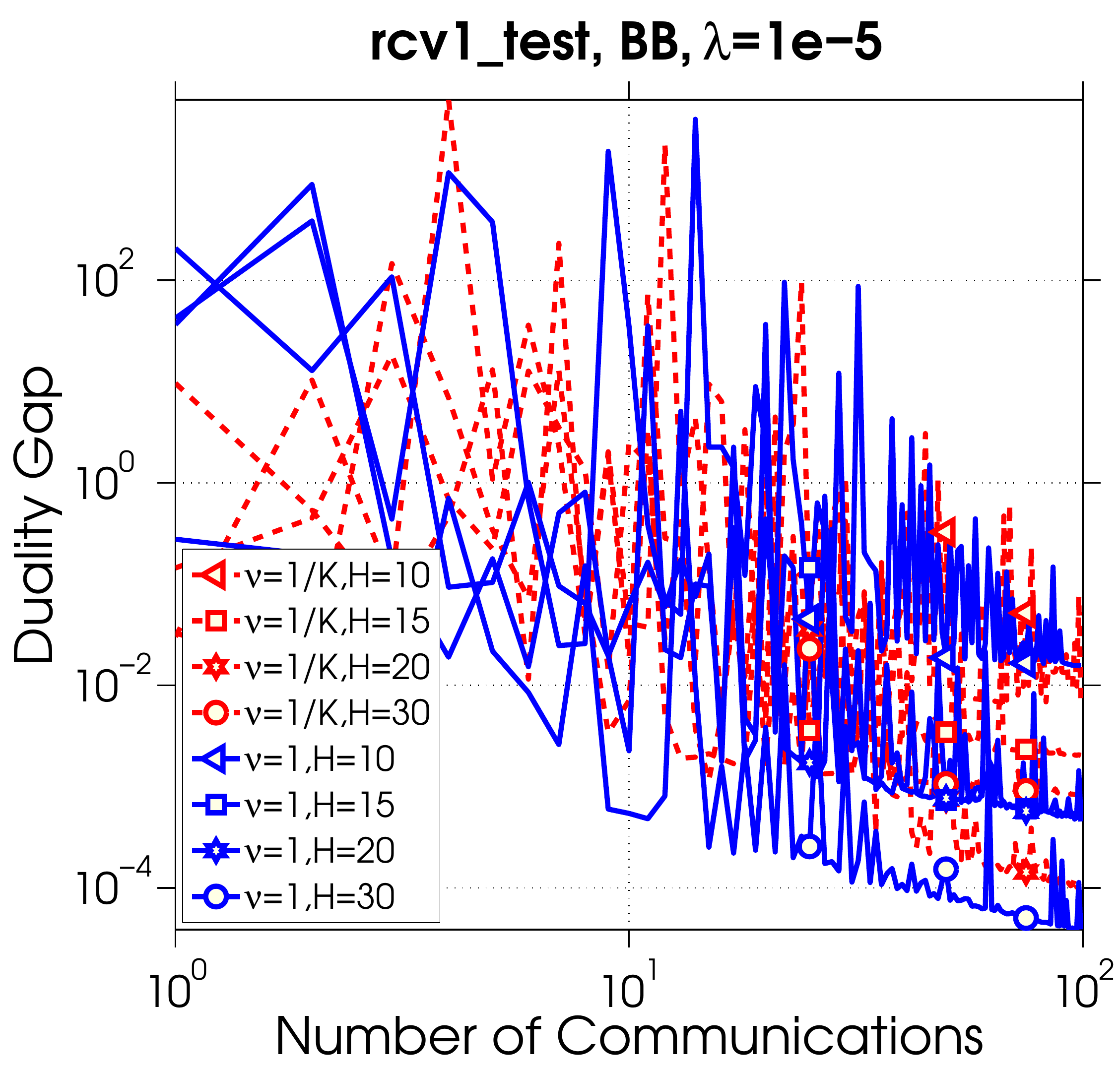}

\includegraphics[scale=.19]{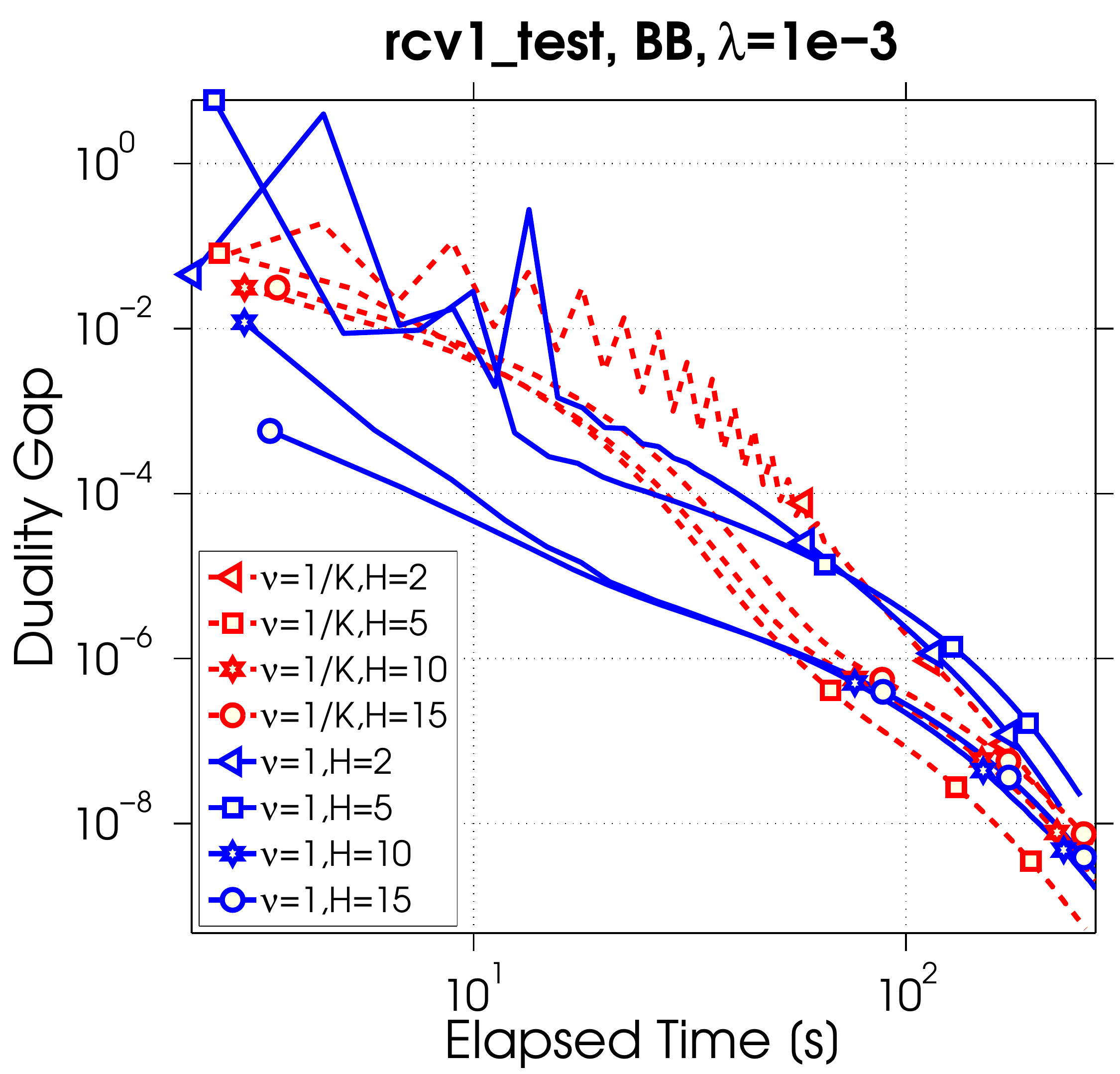}
\includegraphics[scale=.19]{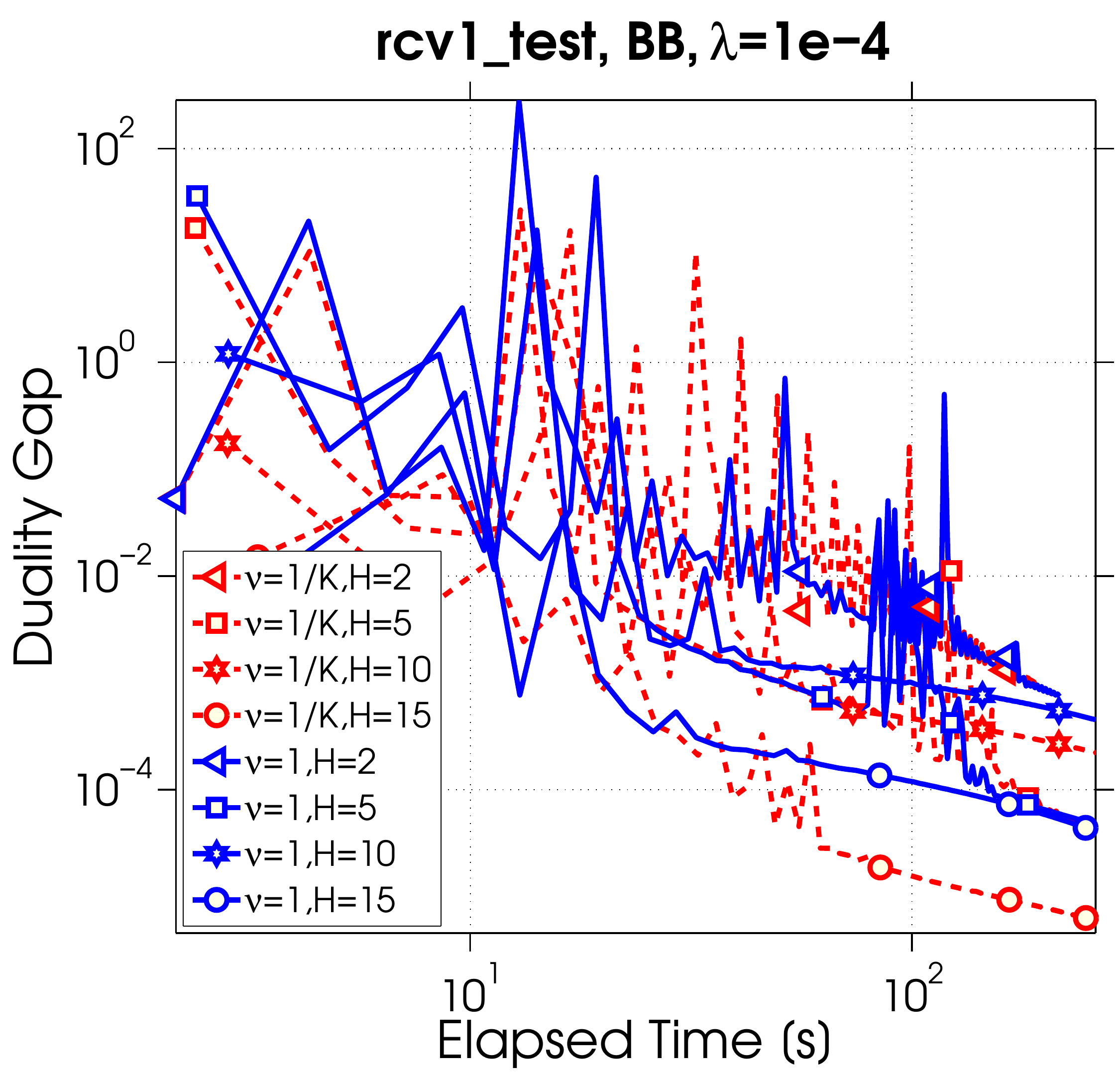}
\includegraphics[scale=.19]{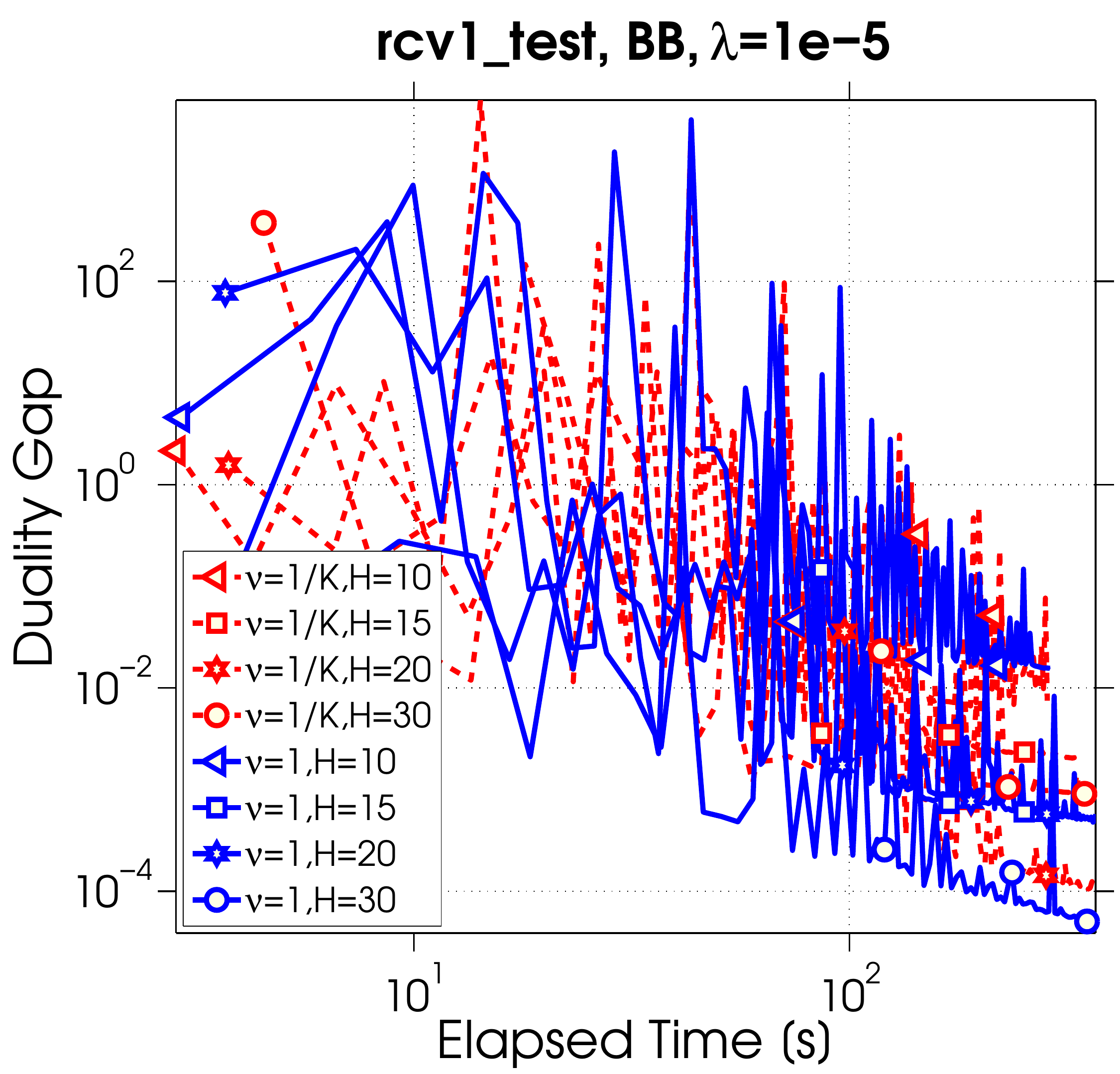}
\caption{Adding (blue solid line) vs Averaging (red dashed line) for BB as the local solver.} 
\label{fig:soler6}
\end{figure}

\begin{figure}[H]
\centering
\includegraphics[scale=.19]{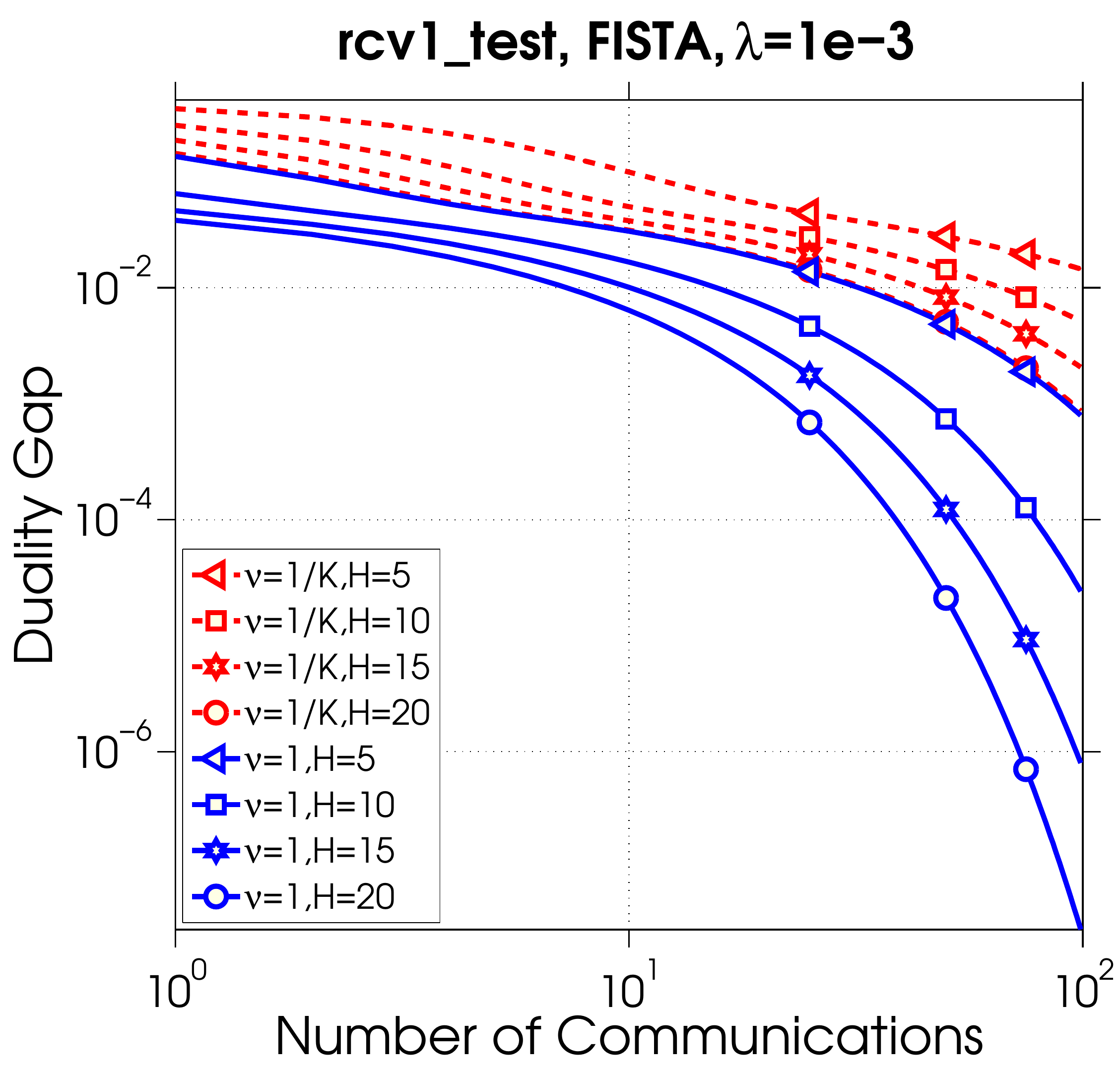}
\includegraphics[scale=.19]{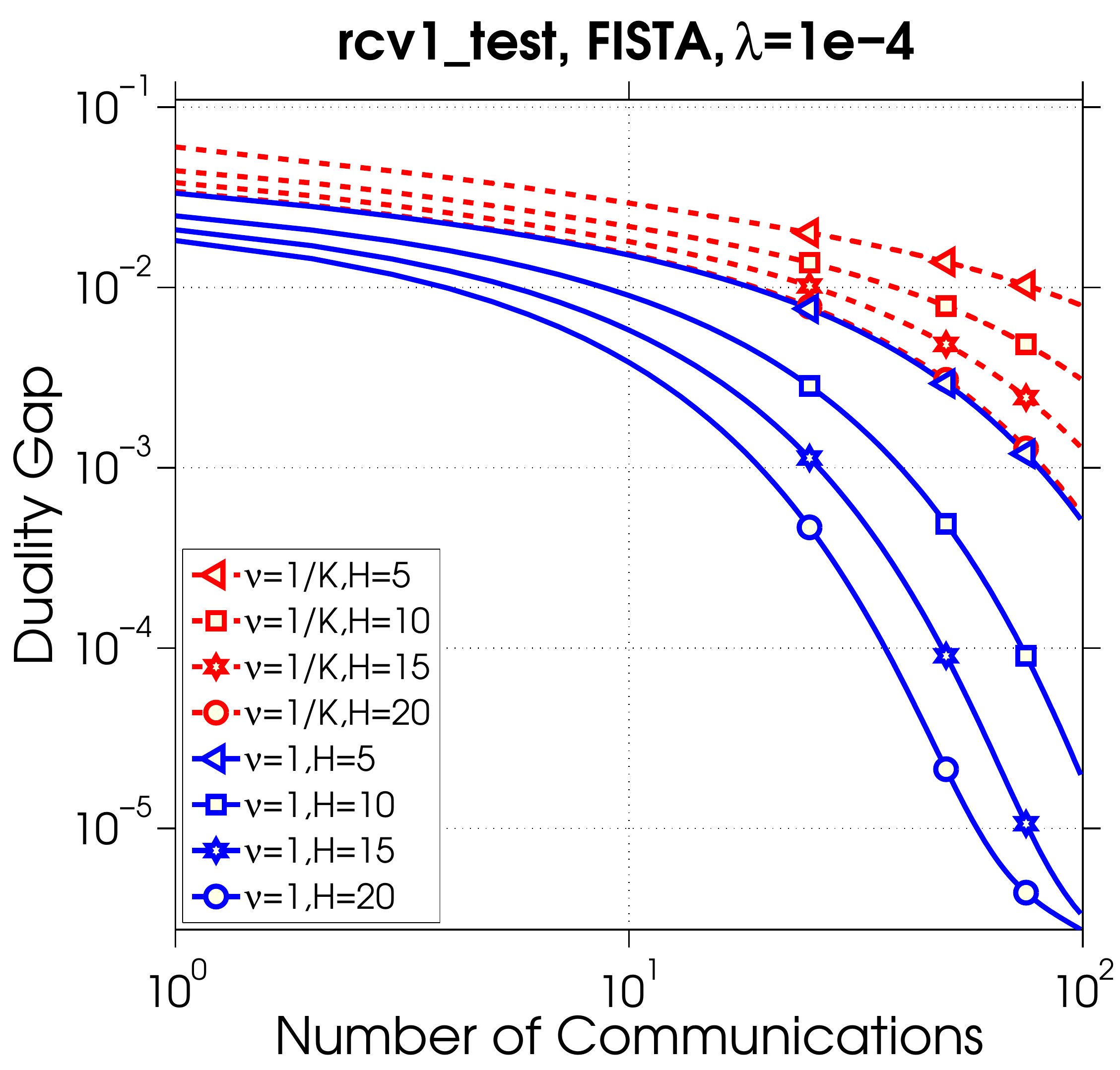}
\includegraphics[scale=.19]{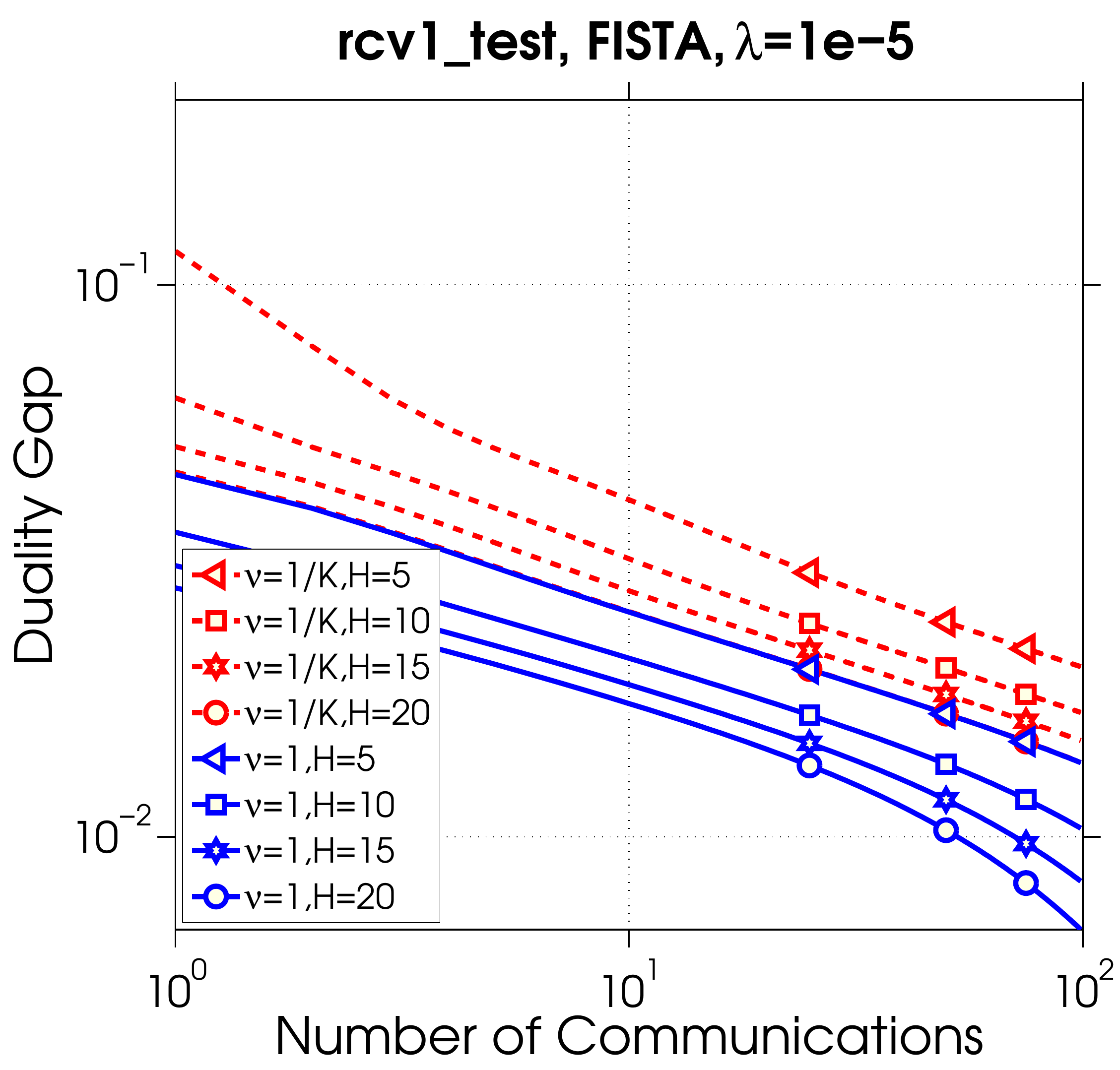}

\includegraphics[scale=.19]{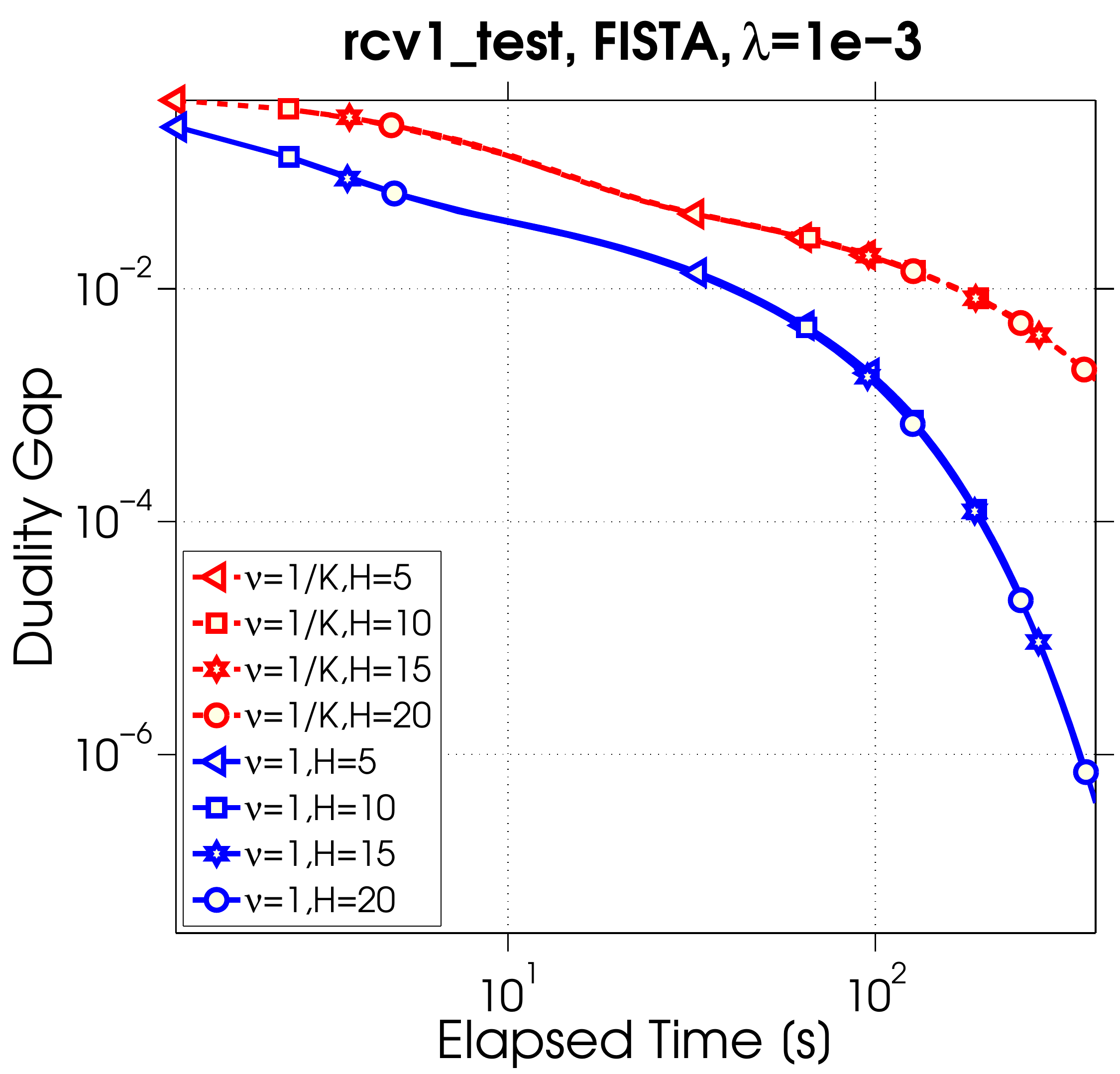}
\includegraphics[scale=.19]{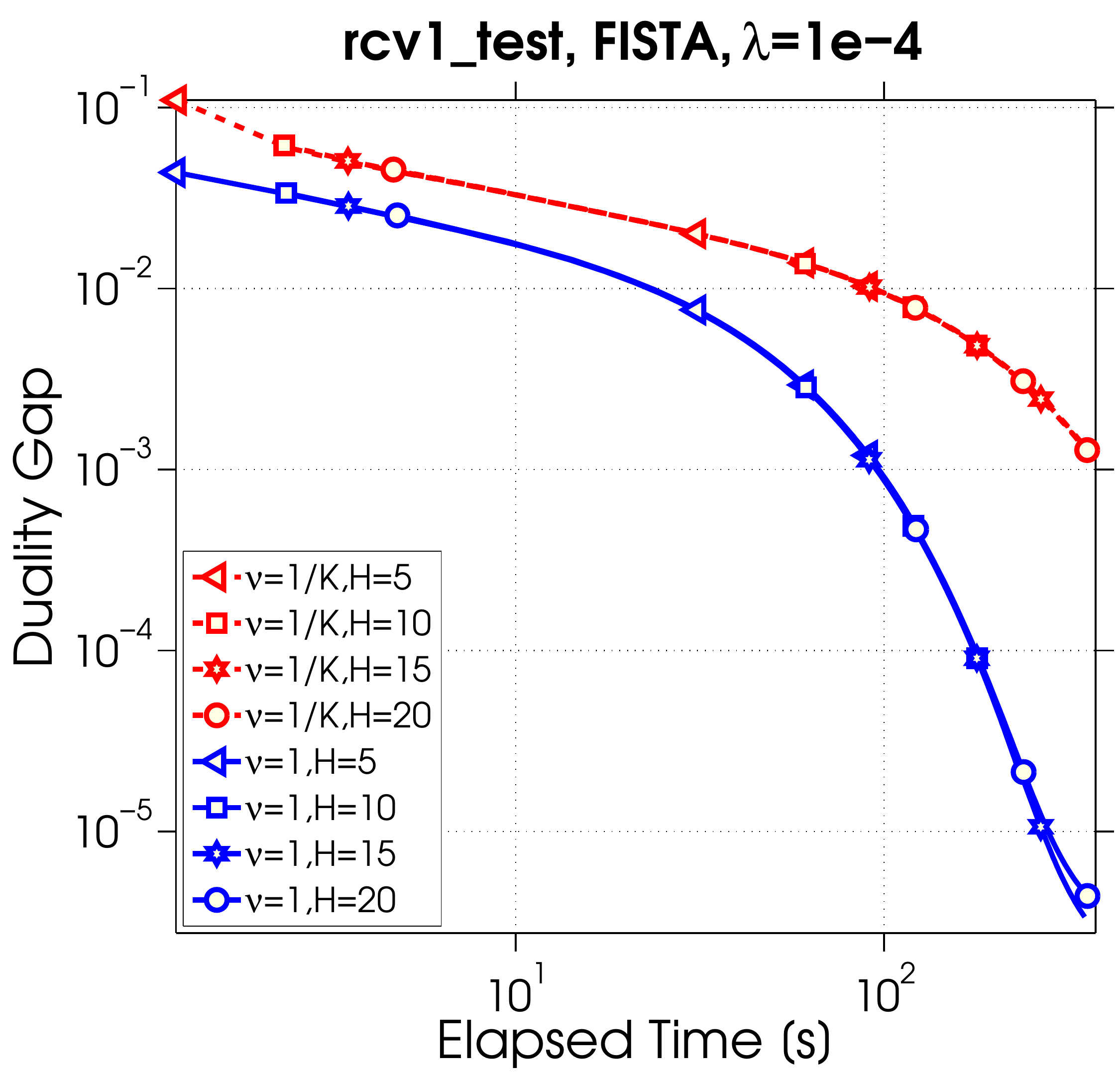}
\includegraphics[scale=.19]{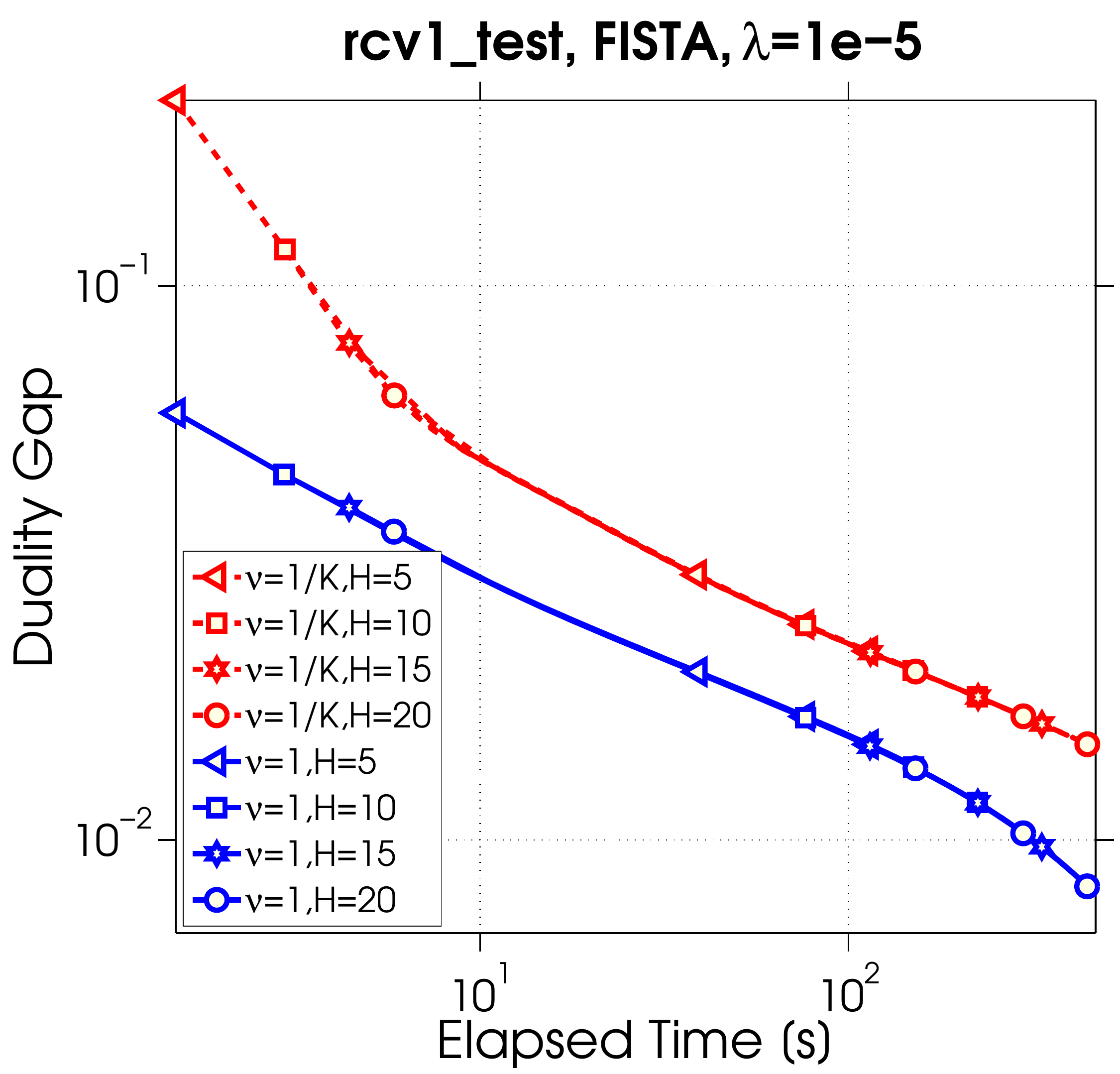}
\caption{Adding (blue solid line) vs Averaging (red dashed line) for FISTA as the local solver.} 
\label{fig:soler7}
\end{figure}

\subsection{The Effect of the Subproblem Parameter $\sigma'$}
\label{sec:subproblemParamExps}

In this section we consider the effect of the choice of the subproblem parameter on convergence (Figure \ref{fig:sigmaTest}). We plot duality gap over the number of communications for RCV and epsilon datasets with quadratic loss and set $K = 8$, $\lambda = 10^{-5}$. For $\aggpar=1$ (adding the local updates), we consider several different values of~$\sigma'$, ranging from $1$ to $8$. The value $\sigma'=8$ represents the safe upper bound of $\aggpar K$, as given in Lemma \ref{lem:sigmaPrimeNotBad}. 

Decreasing $\sigma'$ improves performance in terms of communication until a certain point, after which the algorithm diverges.  For the rcvtest dataset, the optimal convergence occurs around $\sigma'=5$, and diverges fast for $\sigma' \le 3$. For epsilon dataset, $\sigma'$ around $6$ is the best choice and the algorithm will not converge to optimal solution if $\sigma'\leq 5.$ However, more importantly, the ``safe'' upper bound of $\sigma':=\aggpar K=8$ has only slightly worse performance than the practically best (but ``un-safe'') value of~$\sigma'$. 

\begin{figure}[H]
\centering
\includegraphics[scale=.22]{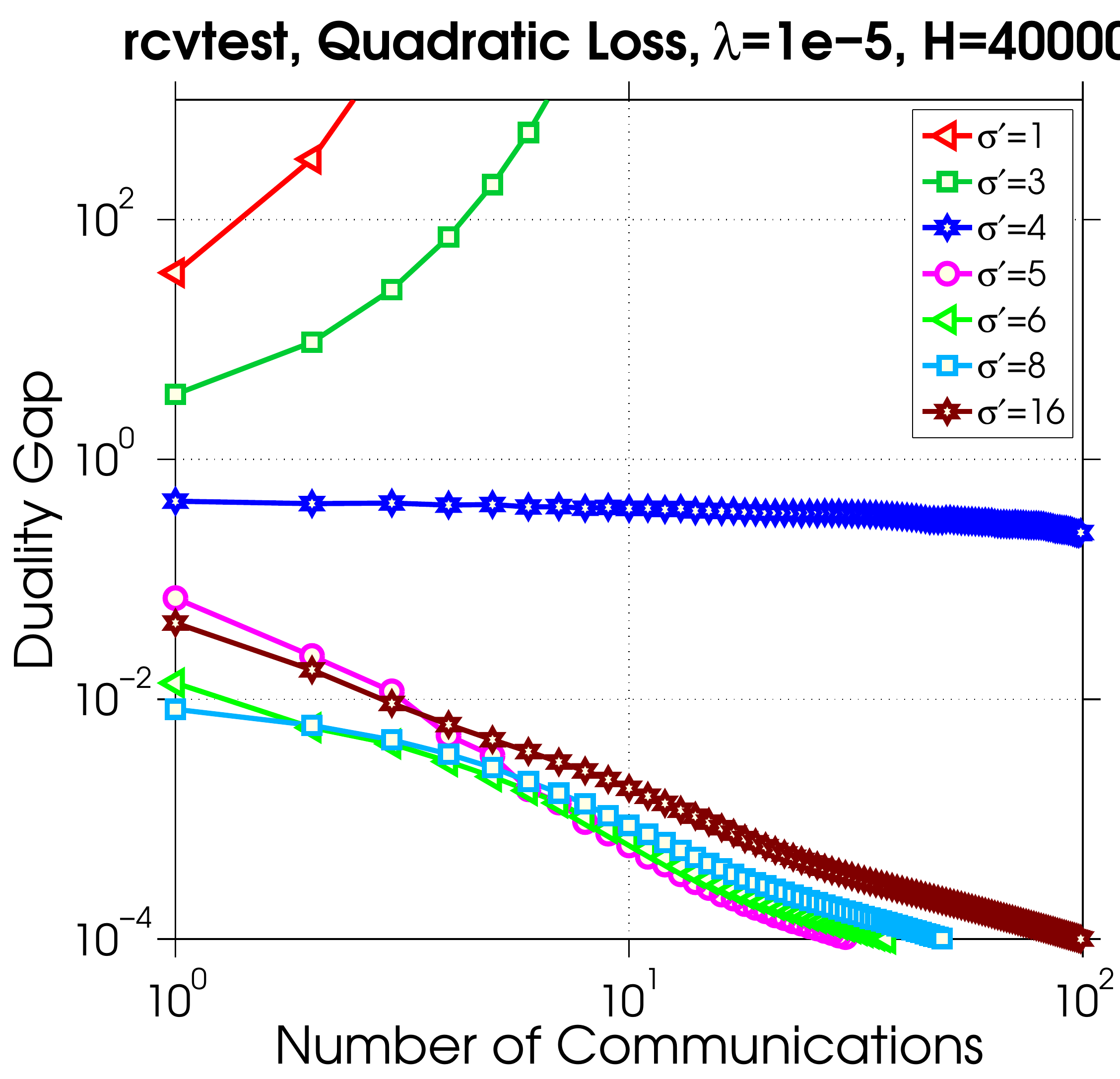}
\includegraphics[scale=.22]{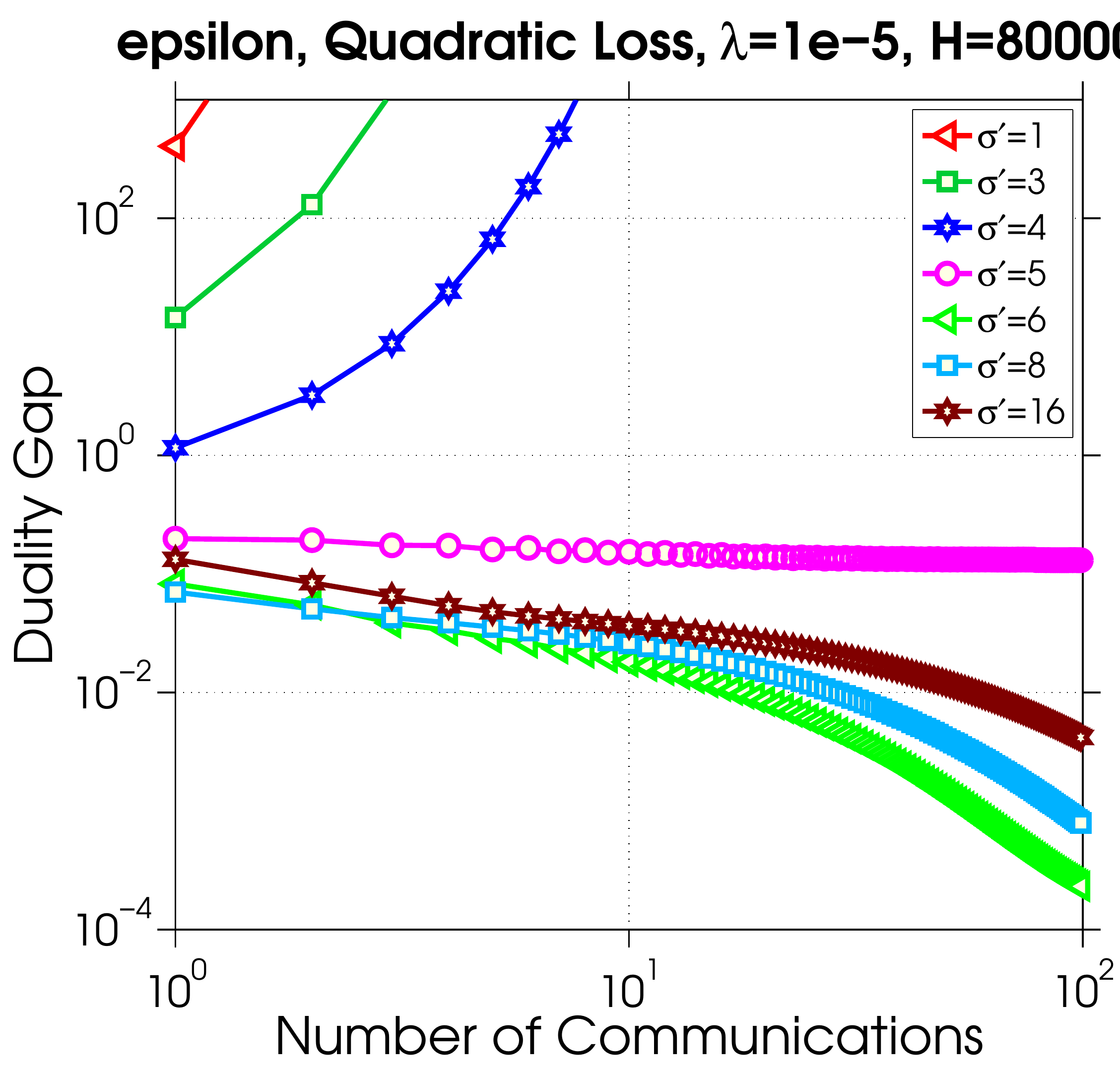}
\caption{The effect of $\sigma'$ on convergence for the rcvtest and epsilon datasets distributed across 8 machines.} 
\label{fig:sigmaTest}
\end{figure}

\subsection{Scaling Property}

\begin{figure}[H]
\centering
\includegraphics[scale=.19]{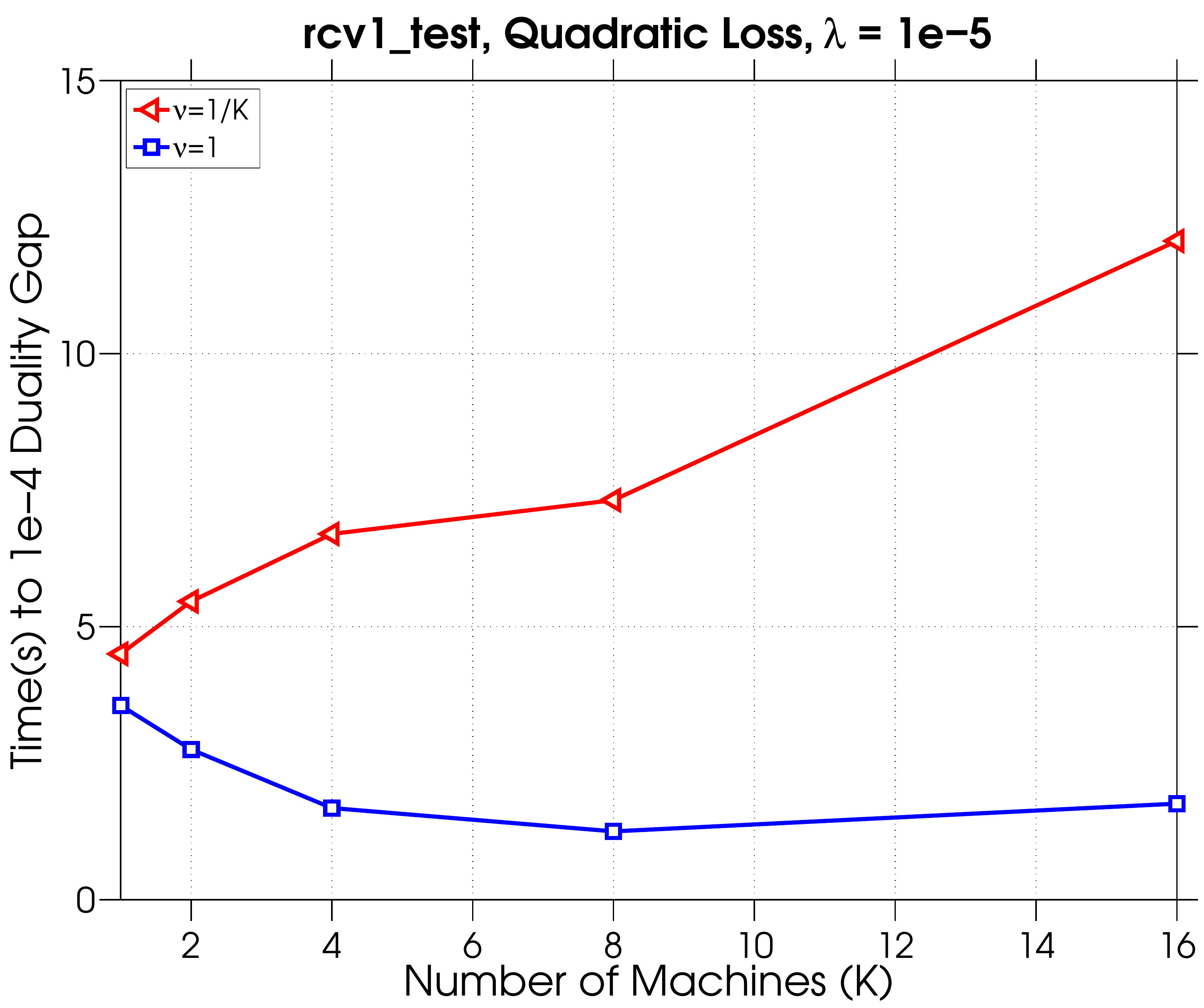}
\includegraphics[scale=.19]{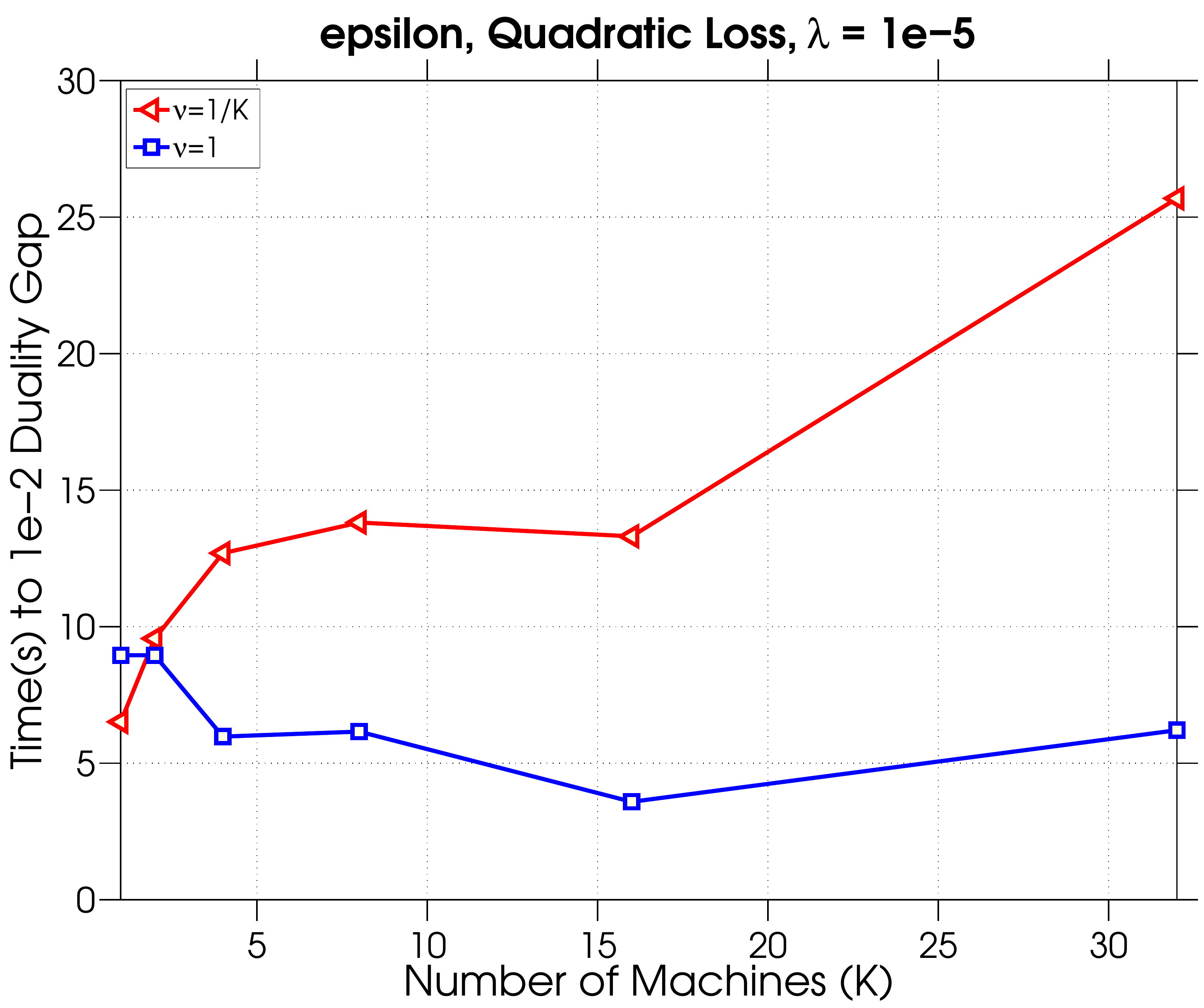}
\caption{The effect of increasing the number of machines $K$ on the time (s) to reach a solution with expected duality gap.} 
\label{fig:scaletest1}
\end{figure}

\mcx{Here we demonstrate the ability of our framework to scale with  $K$ (number of machines). We compare the run time to reach a specific tolerance on duality gap ($10^{-4}$ and $10^{-2}$) for two choices of $\aggpar$. Looking at Figure~\ref{fig:scaletest1}, we see that when choosing $\aggpar=1$, the performance improves as the number of machines increases. However, when $\aggpar = \frac{1}{K}$, the algorithm slows down as $K$ increases. The observations support our analysis in Section 4.  }

\subsection{Performance on a Big Dataset}
\label{sec:hugeDatasetExp}
As shown in Figure~\ref{fig:hugedata}, we test the algorithm on the \emph{splice-site.t} dataset, whose size is about 280 GB. 
We show experiments for three different loss functions $\ell$, namely logistic loss, hinge loss and least squares loss (see Table \ref{tbl:differentLossFunctions}).
We set $\lambda = 10^{-6}$ for the squared norm regularizer. 
The dataset is distributed  across $K=4$ machines and we use \mcx{CD} as the local solver with $H = 50,000$. In all the cases, an optimal solution can be reached in about 20 minutes and again, we observe that setting the aggregation parameter $\nu:=1$ leads to faster convergence than $\nu:=\frac 1 K$ (averaging). 

Also, the number of communication rounds for the three different loss functions are almost the same if we set all the other parameters to be same. However, the patterns of  duality gap decrease for the three loss functions are different.

\begin{figure}[H]
\centering
\includegraphics[scale=.19]{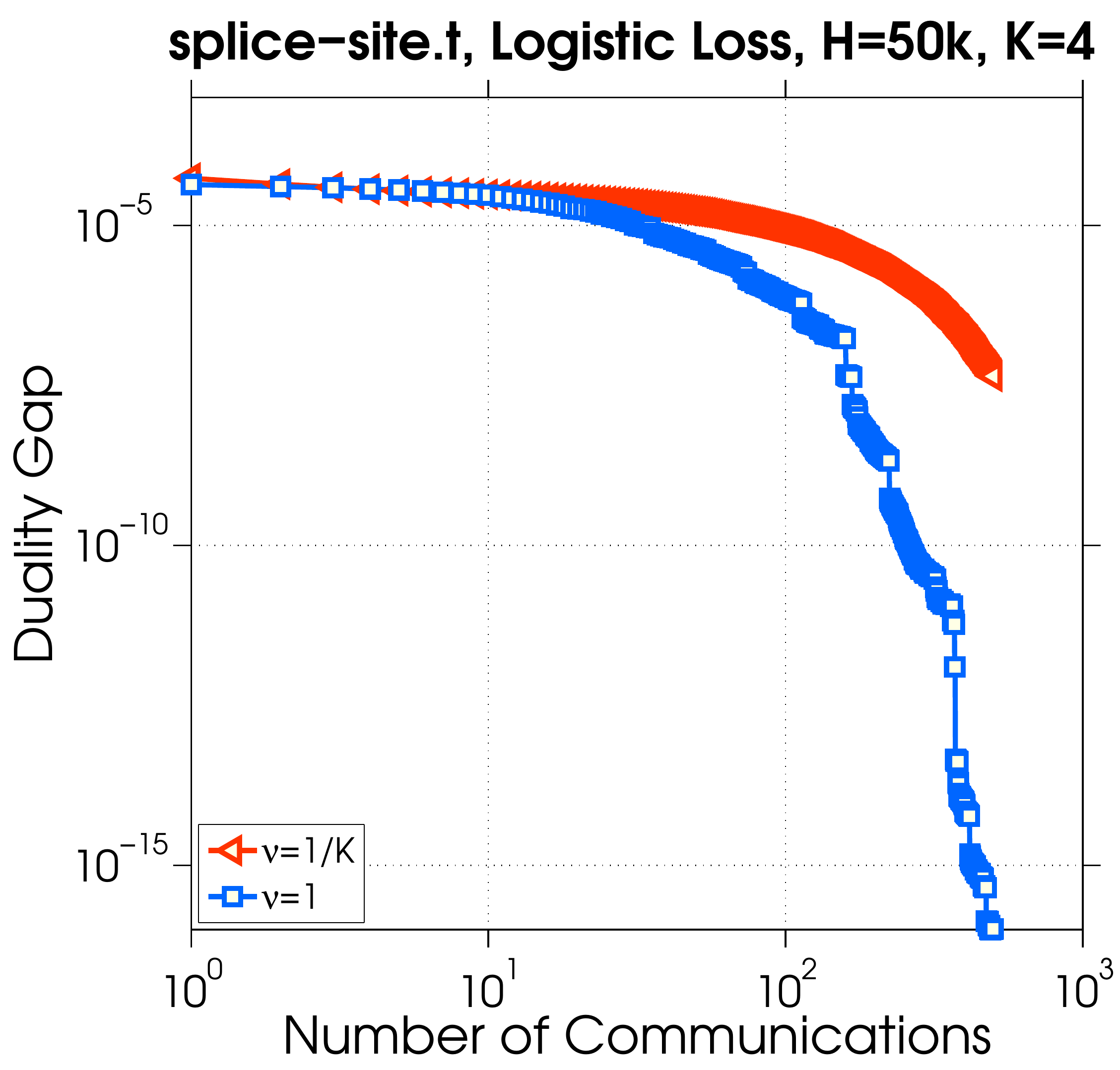}
\includegraphics[scale=.19]{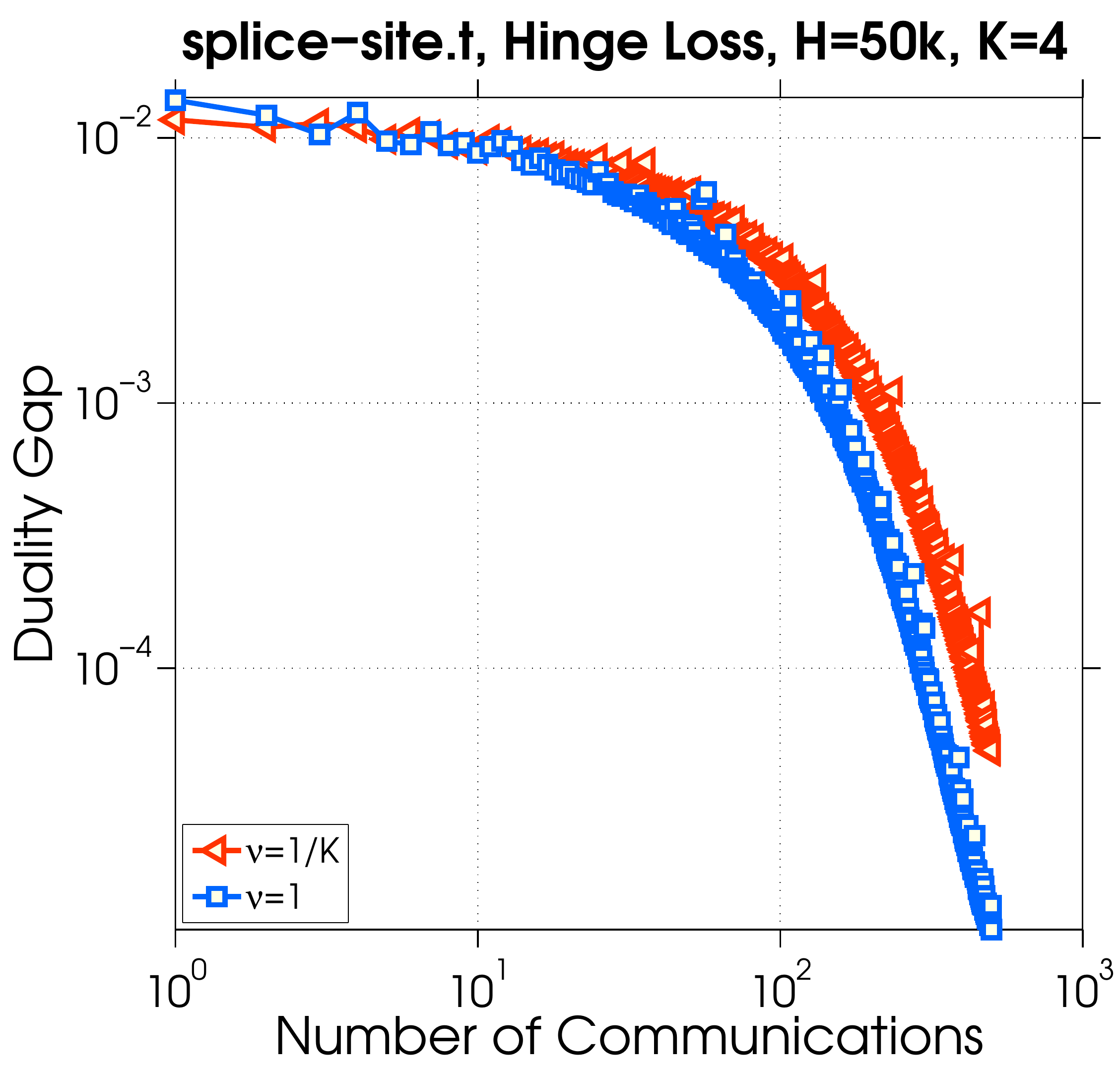}
\includegraphics[scale=.19]{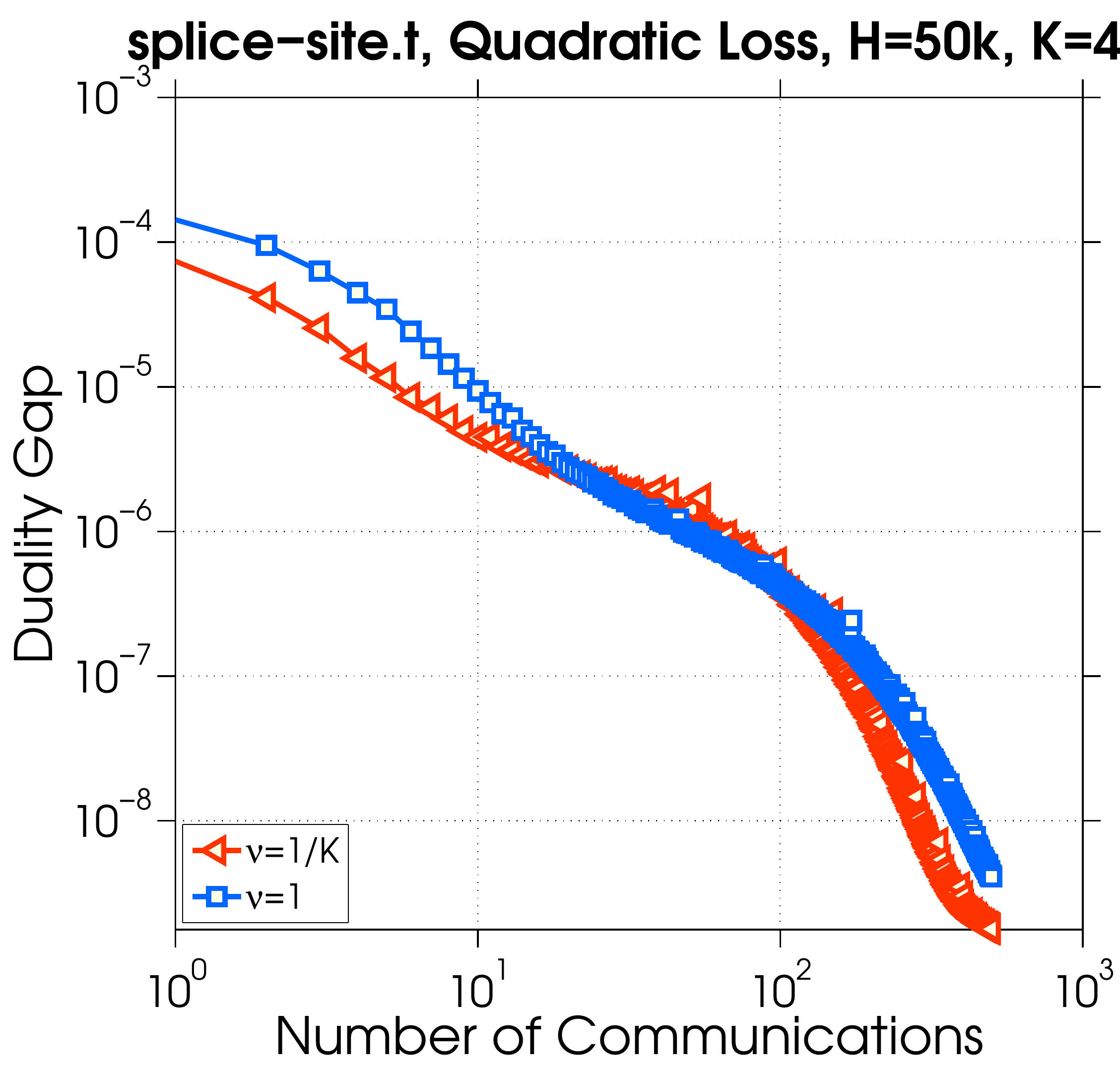}

\includegraphics[scale=.19]{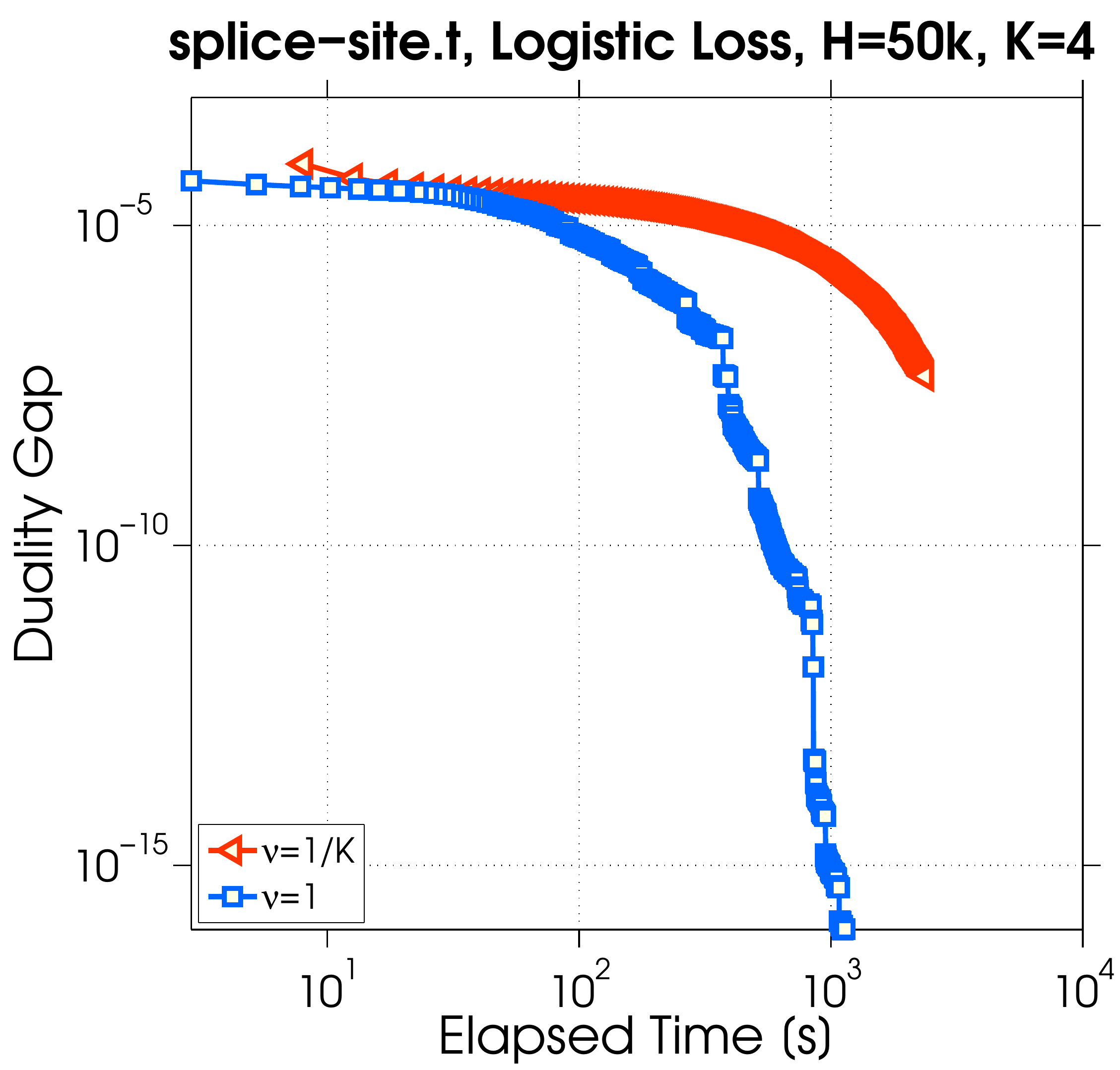}
\includegraphics[scale=.19]{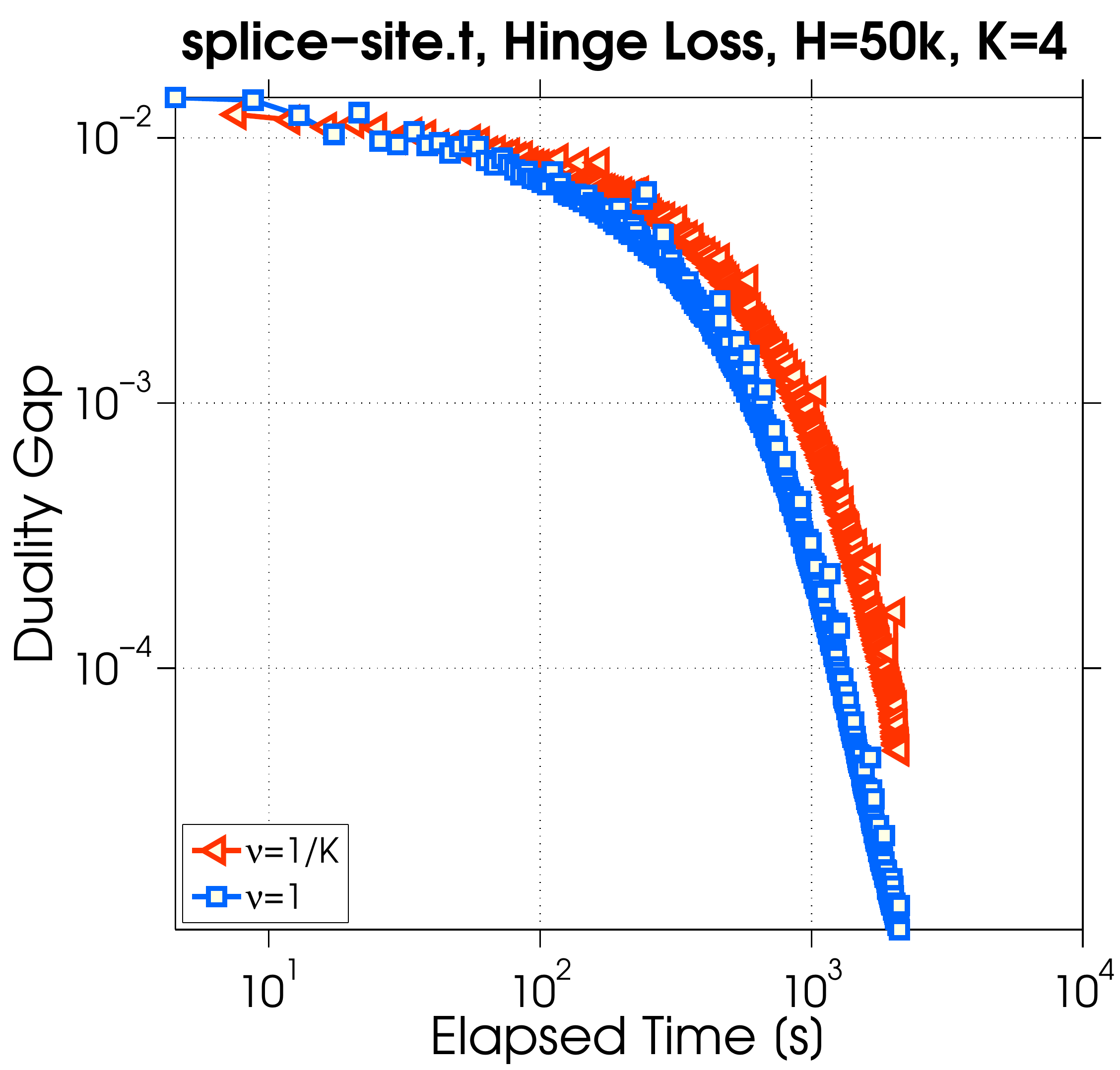}
\includegraphics[scale=.19]{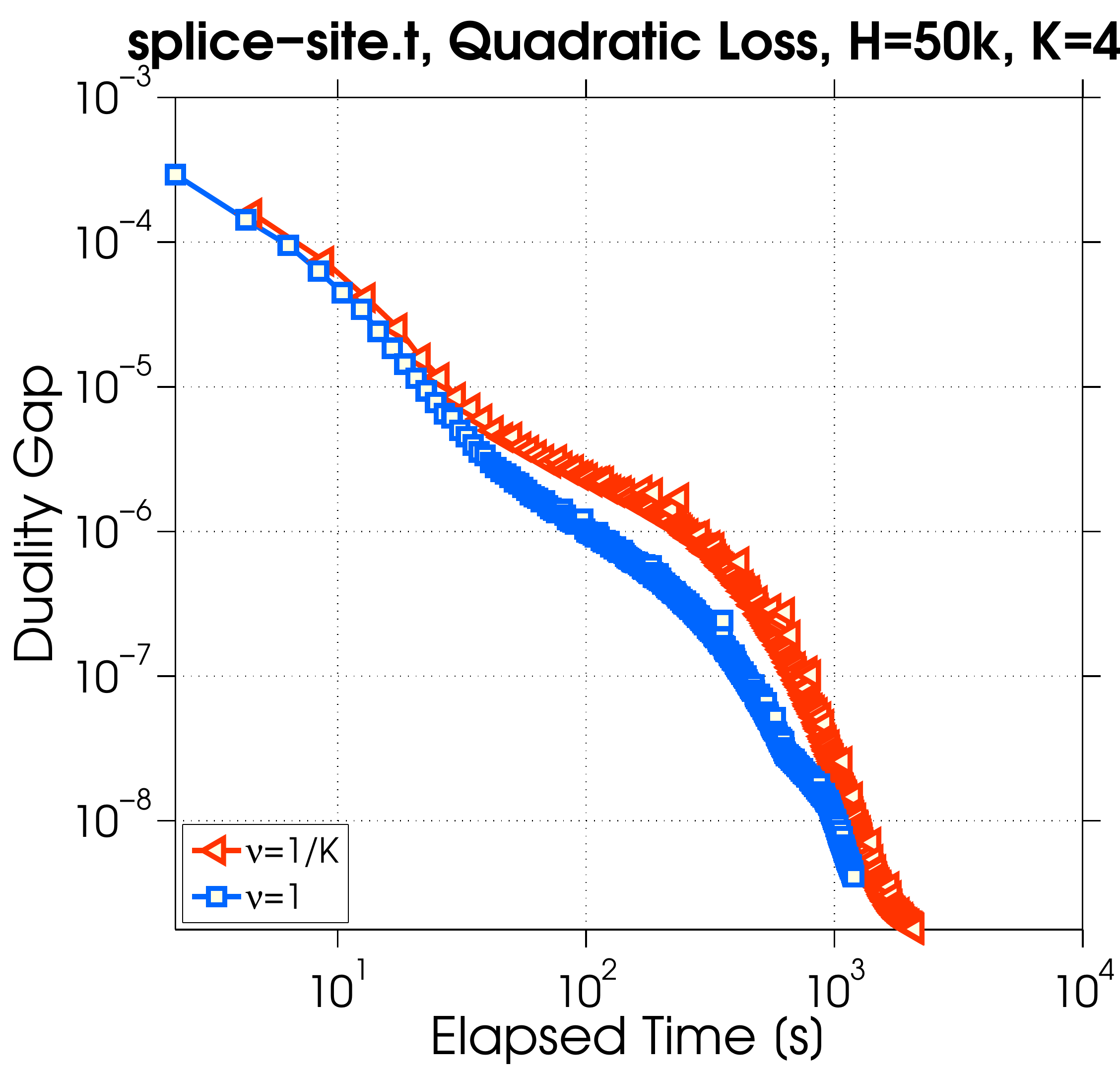}
\caption{Performance of Algorithm \ref{alg:cocoaPractical} on \emph{splice-site.t} dataset, with three different loss functions. } 
\label{fig:hugedata}
\end{figure}

\subsection{Comparison with other distributed methods}

\mcx{Finally, we compare the \cocoap framework with several competing distributed optimization algorithms. The DiSCO algorithm~\cite{Zhang:2015mp} is a Newton-type method, where in each iteration the updates on iterates are computed inexactly using a Preconditioned Conjugate Gradients (PCG) method. As suggested in \cite{Zhang:2015mp}, in our implementation of DISCO we apply the Stochastic Average Gradient (SAG) method~\cite{schmidt2013minimizing}  as the solver to get the initial solutions for each local machine and solve the linear system during PCG.  DiSCO-F~\cite{2016arXiv160305191M}, improves on the computational efficiency of original DiSCO, based on data partitioned across features rather than examples. The DANE algorithm~\cite{Shamir:2013vf} is another distributed Newton-type method, where in each iteration there are two rounds of communications. Also, a subproblem is to be solved in each iteration to obtain updates. For all these algorithms, all the hyper parameters are tuned manually to optimize their performance.

The experiments are conducted on a compute cluster with $K=4$ machines. We run all algorithms using two datasets. Since not all  methods  are primal-based in nature,  it is difficult to continue using duality gap as a measure of optimality. Instead, the norm of the gradient of the primal objective function \eqref{eq:primal} is used to compare the relative quality of the solutions obtained. 

As shown in Figure~\ref{fig:diffmeth}, in terns of number of communications, \cocoap usually converges more rapidly than competing methods during the early  iterations, but tends to get slower later on in the iterative process. This  illustrates that the Newton-type methods can accelerate in the vicinity of the optimal solution, as expected. However, \cocoap can still beat other methods in running time. The main reason for this is the fact that the subproblems in our framework can be solved more efficiently, compared with DiSCO and DANE.}

\begin{figure}[H]
\centering
\includegraphics[scale=.14]{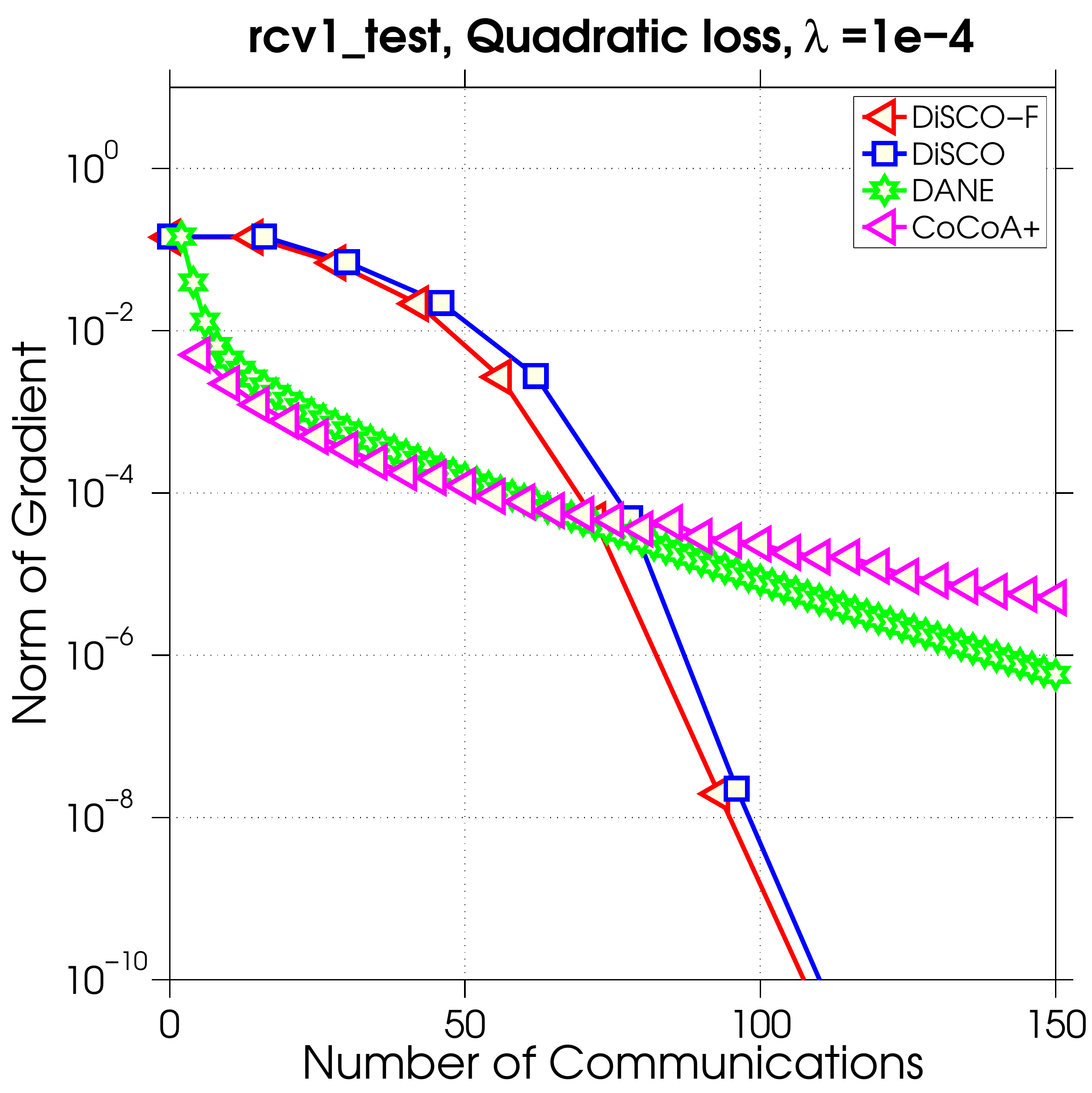}
\includegraphics[scale=.14]{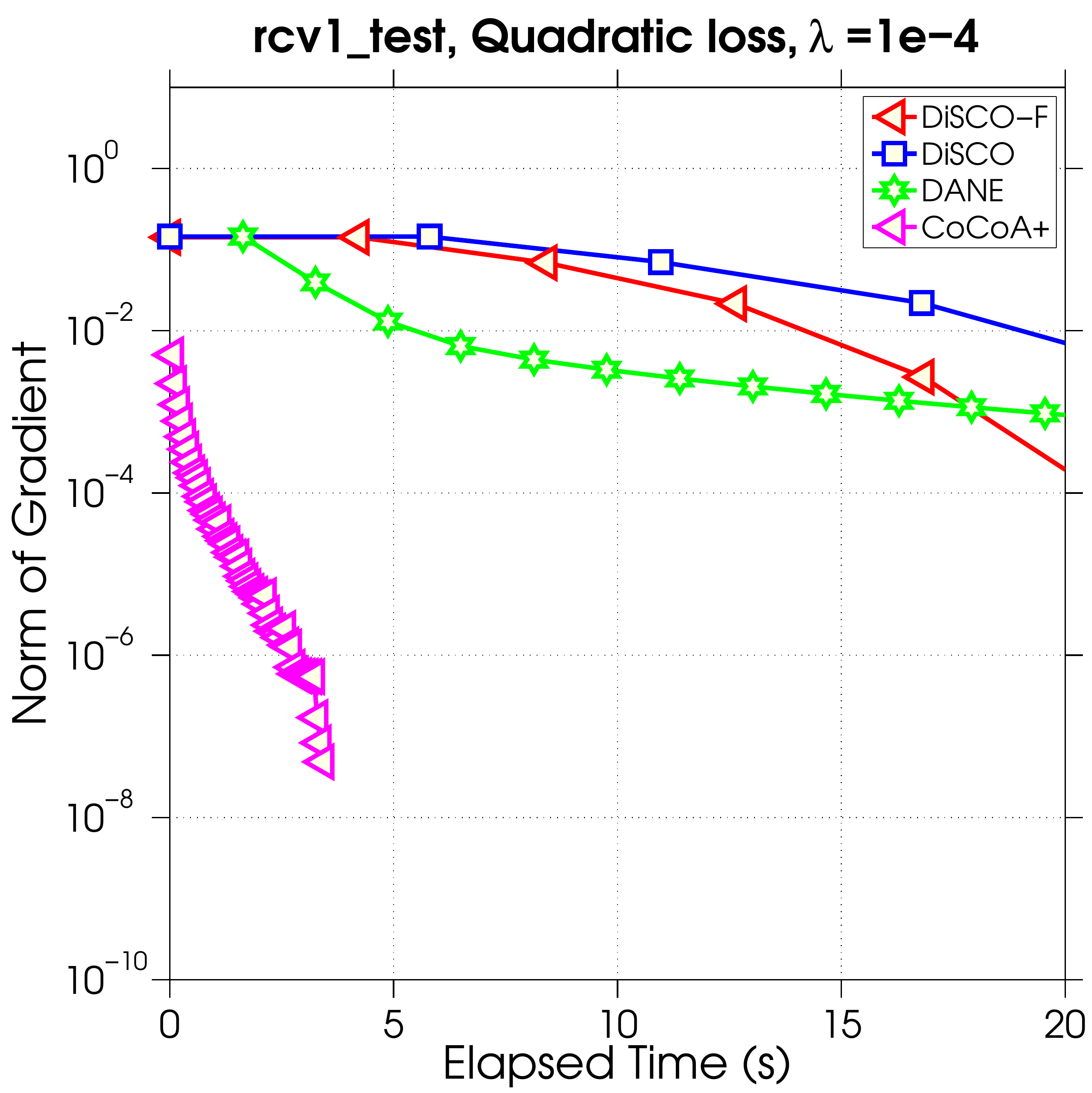}
\includegraphics[scale=.14]{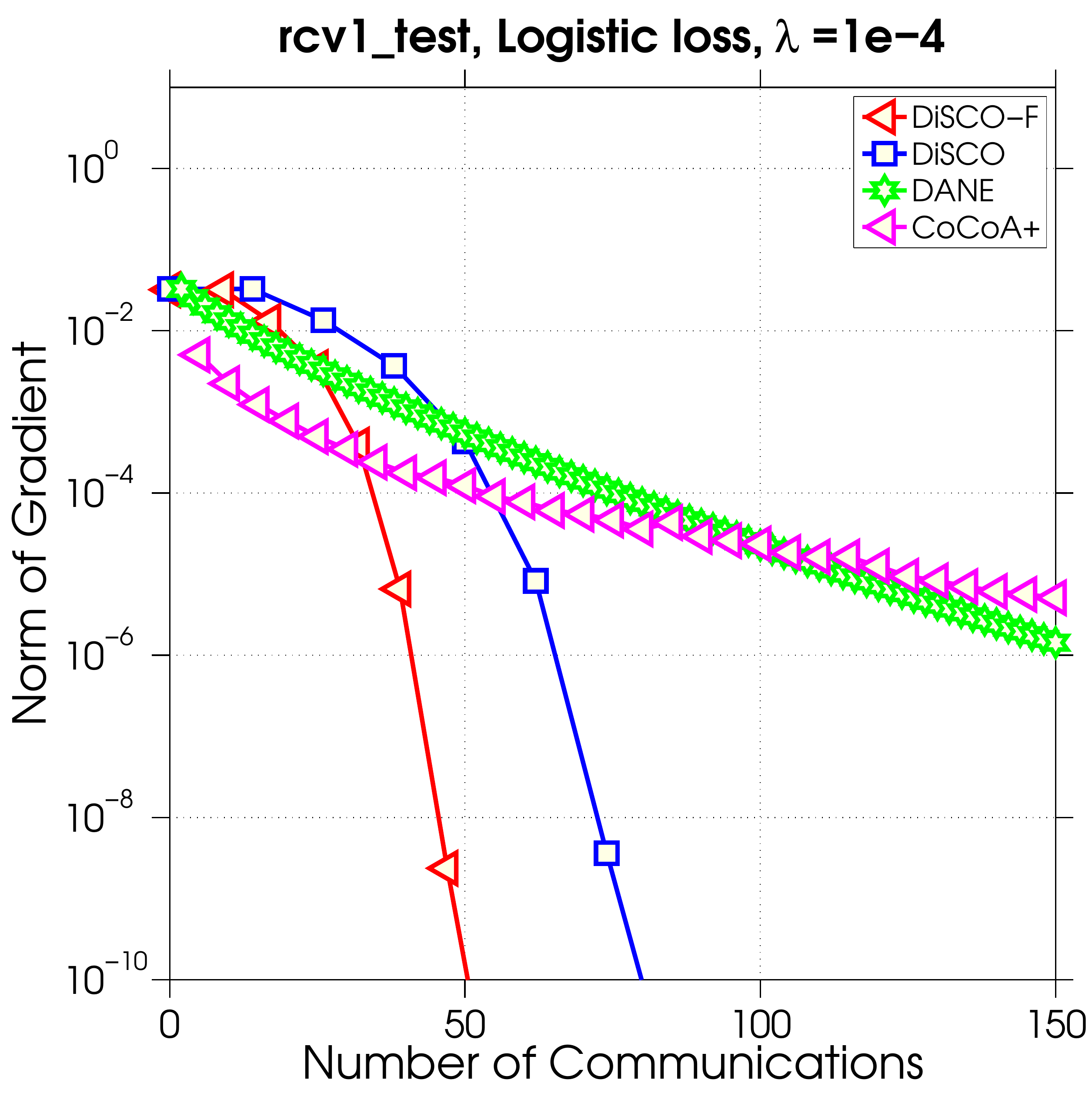}
\includegraphics[scale=.14]{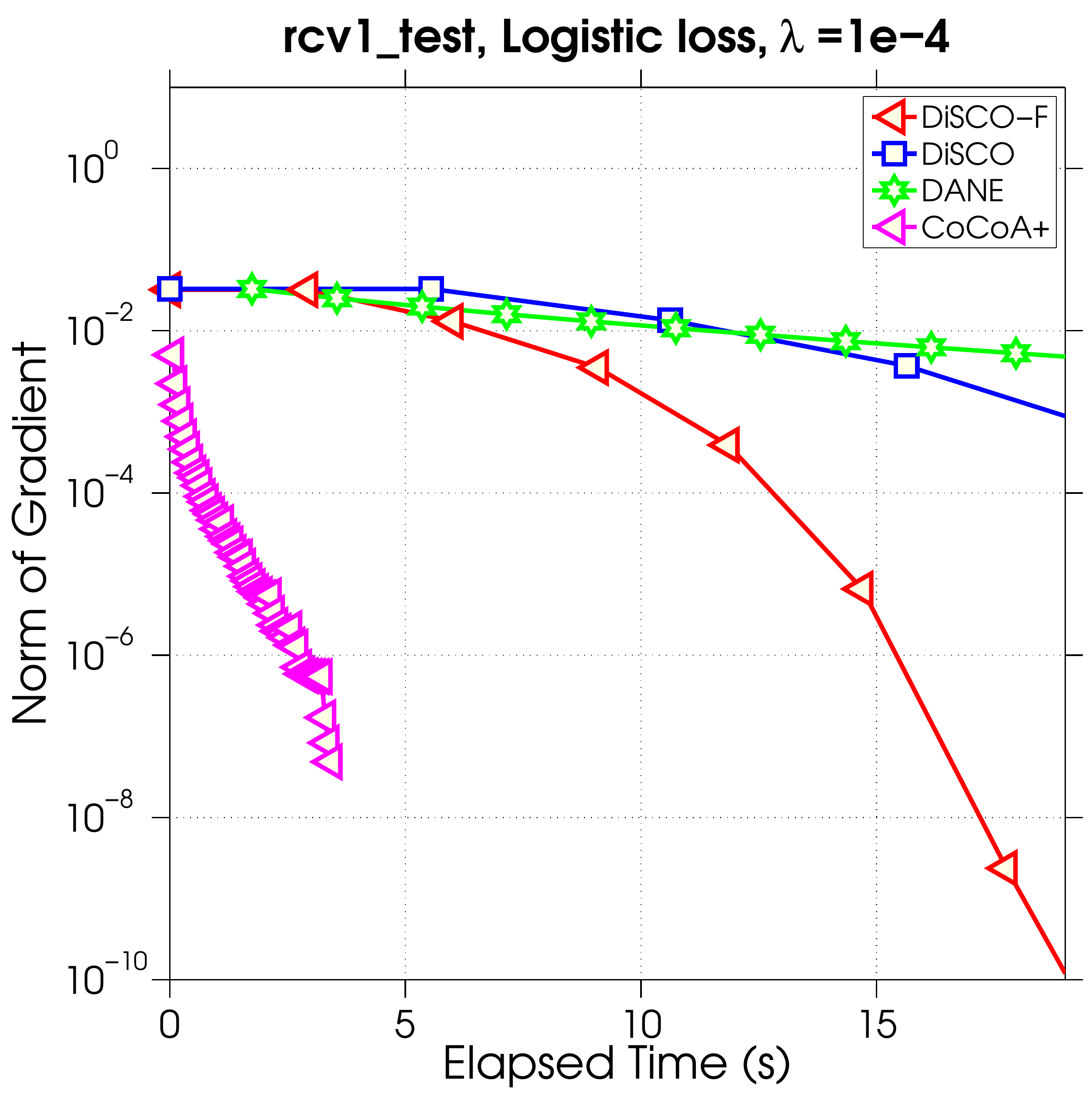}

\includegraphics[scale=.14]{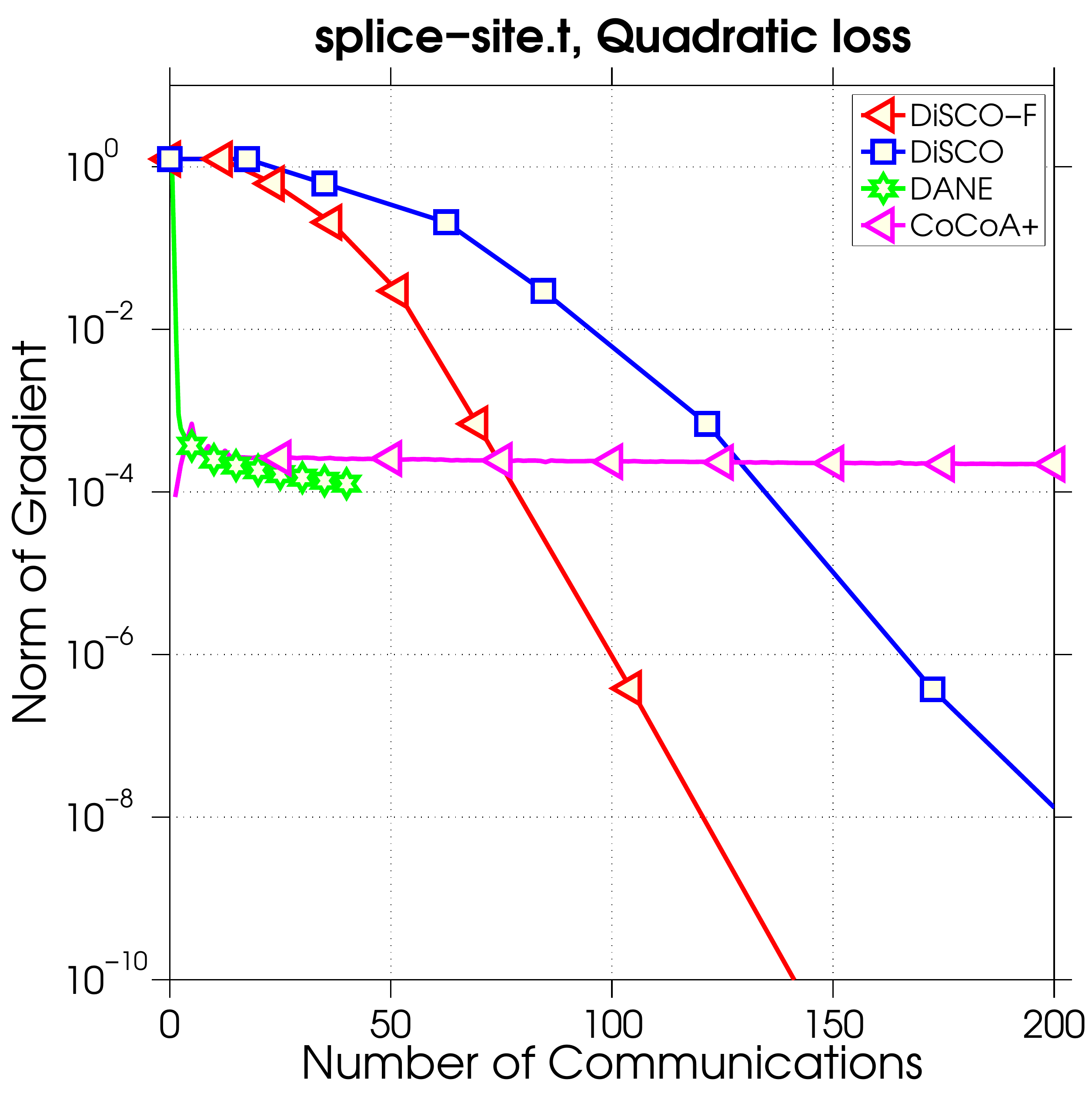}
\includegraphics[scale=.14]{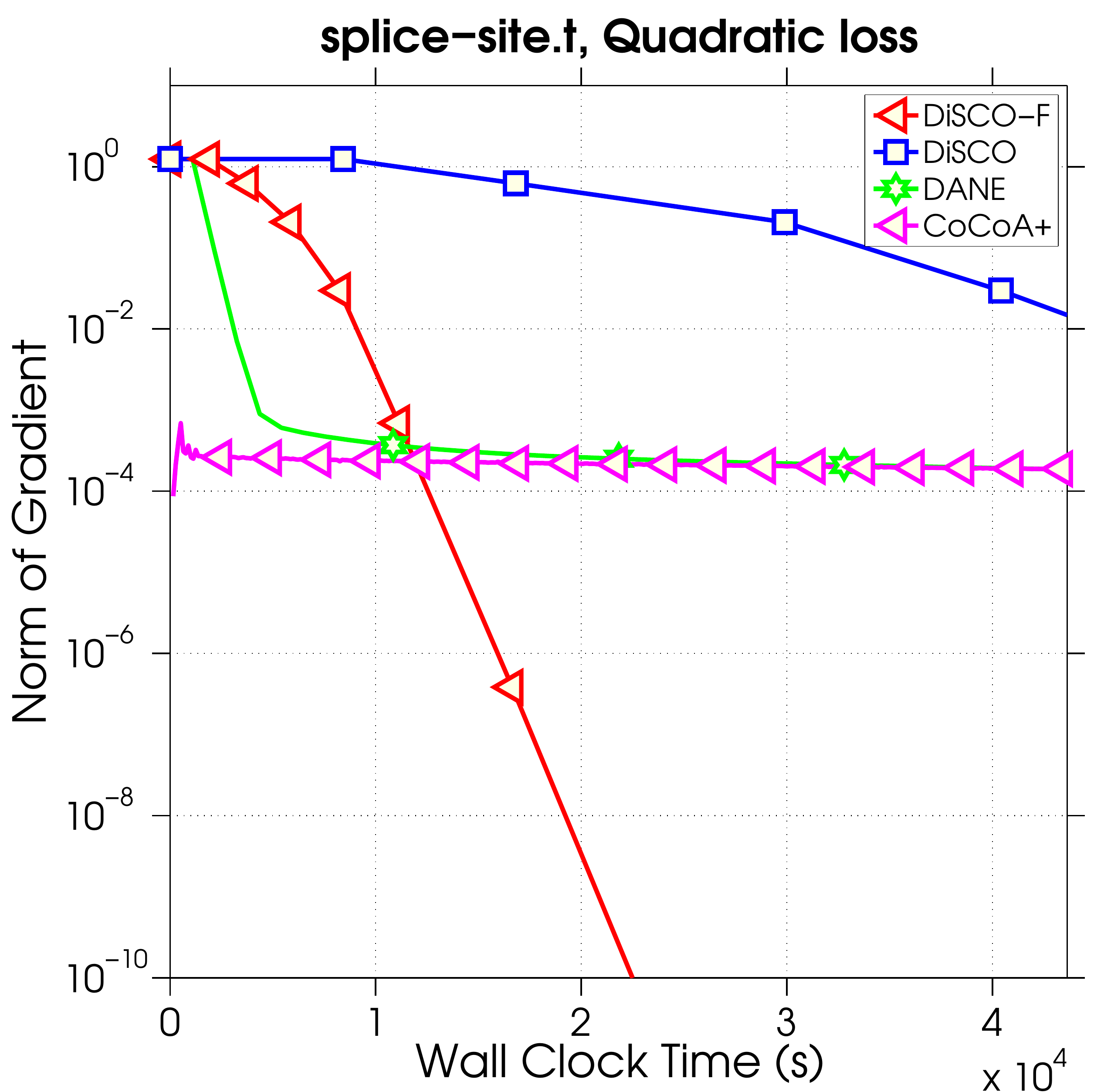}
\includegraphics[scale=.14]{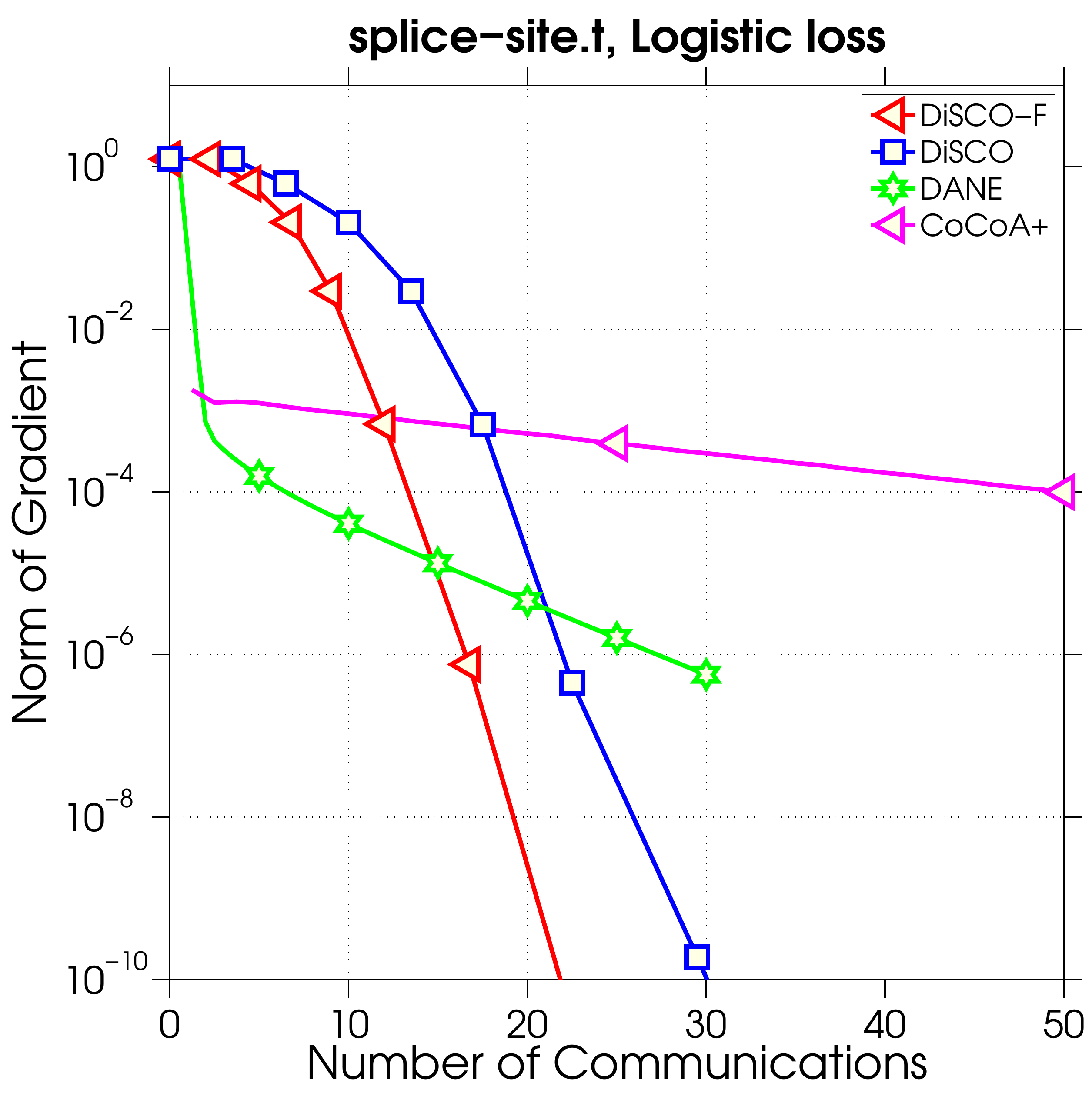}
\includegraphics[scale=.14]{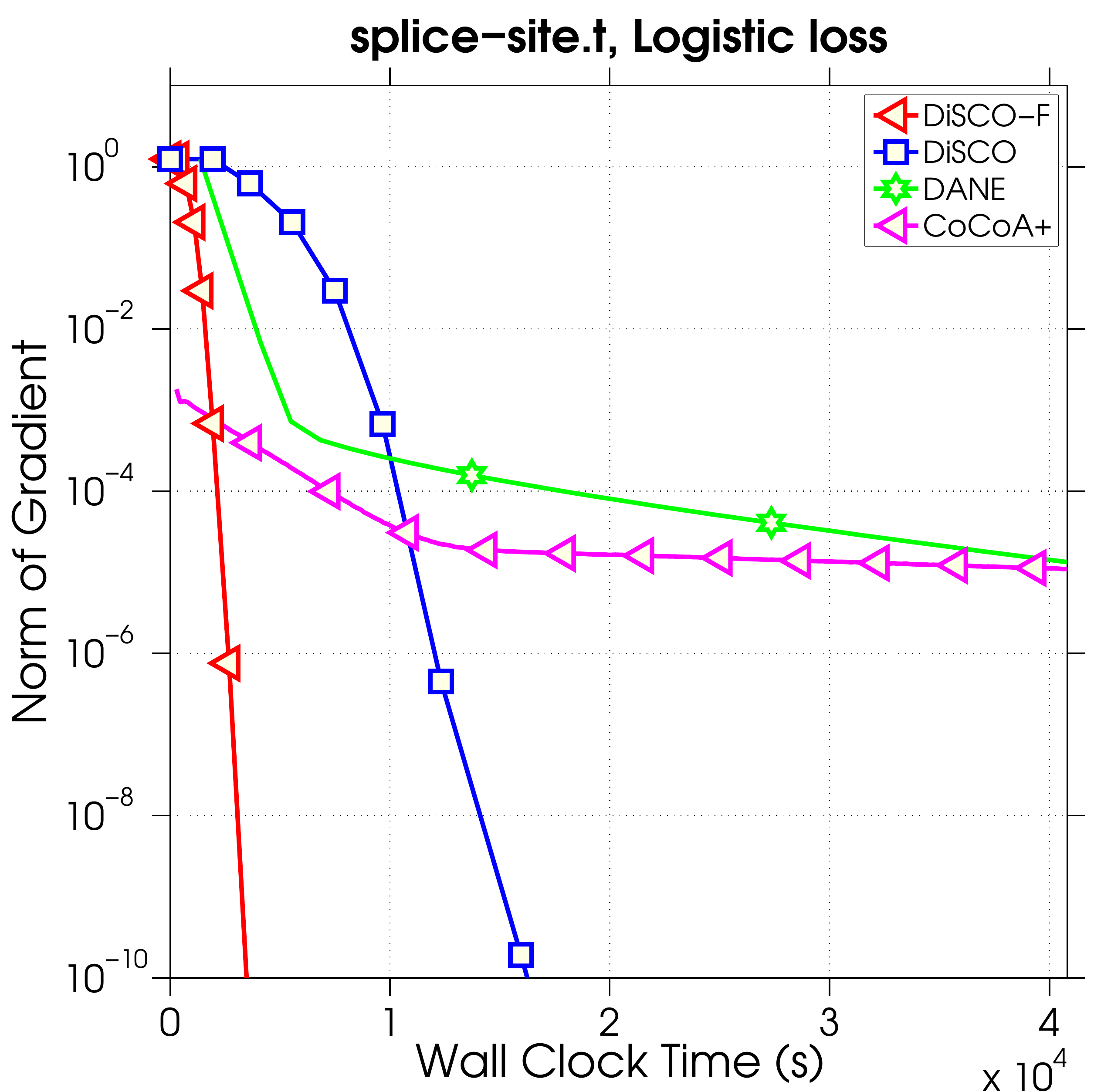}
\caption{Performance of several distributed frameworks on solving \eqref{eq:primal} with different losses on two datasets.} 
\label{fig:diffmeth}
\end{figure}

\section{Conclusion}
We present a novel \cocoap framework for distributed optimization which enables fast and communication-efficient \textit{additive aggregation} in distributed primal-dual optimization.  We analyze the theoretical complexity  of  \cocoap, giving strong 
primal-dual convergence rates with outer iterations scaling independently of  the number of machines. 
We extended the basic theory to allow for non-smooth loss functions, arbitrary strongly convex regularizers, and primal-dual convergence results. Our 
experimental results show significant speedups in terms of runtime over previous methods, including the 
original \cocoa framework as well as other state-of-the-art methods.

\bigskip
 
 \small
\bibliographystyle{gOMS}
\bibliography{literature}
\normalsize

\clearpage

\appendix

\section{Proofs}

\subsection{Proof of Lemma \ref{lem:quartz}}
 
Since $g$ is $1$-strongly convex, $g^*$ is $1$-smooth, and thus we can use \eqref{def:Lsmoothness:gstar} as follows
\begin{align*}
f(\alphav + \hv) &= \lambda g^*\left( \tfrac1{\lambda n} X \alphav + \tfrac1{\lambda n} X \hv \right) \\ 
&\overset{\eqref{def:Lsmoothness:gstar}}{\leq} \lambda \left( g^* \left( \tfrac1{\lambda n} X \alphav \right) + \< \nabla g^* \left( \tfrac1{\lambda n} X \alphav \right) , \tfrac1{\lambda n} X \hv > + \tfrac12 \left\| \tfrac1{\lambda n} X \hv  \right\|^2 \right) \\ 
&= f(\alphav) + \< \nabla f(\alphav), \hv > + \tfrac{1}{2 \lambda n^2} \hv^T X^T X \hv.\qedhere 
\end{align*}

\subsection{Proof of Lemma \ref{lem:step:sigma1}}

Indeed, 
\begin{align*}
D(\alphav+\nu \textstyle{\sum}_{k=1}^K \hk)
&
=D(\alphav+\nu \hv)
\\
&\overset{\eqref{eq:dual}}{=}
 \tfrac1n \textstyle{\sum}_{i=1}^n -\ell_i^*(- \alpha_i-\nu h_i)
 - \lambda g^* \left( \tfrac1{\lambda n} X (\alphav+\nu \hv) \right) 
\\
&\overset{\eqref{eq:fRdefinition}}{=}
 \tfrac1n \textstyle{\sum}_{i=1}^n -\ell_i^*(- \alpha_i-\nu h_i)
 - f(\alphav+\nu \hv)  
\\
&\overset{\eqref{eq:quartz}}{\geq} 
 \tfrac{1-\nu}n \textstyle{\sum}_{i=1}^n -\ell_i^*(- \alpha_i)
+ \nu \tfrac1n \textstyle{\sum}_{i=1}^n -\ell_i^*(- \alpha_i-h_i)
\\& \qquad 
 - f(\alphav) - \nu \< \nabla f(\alphav), \hv> -\nu^2 \tfrac{1}{2 \lambda n^2} \hv^T X^T X \hv
\\
&\overset{\eqref{eq:dual},\eqref{eq:sigmaPrimeSafeDefinition}}{\geq} 
(1-\nu) D(\alphav)  
- \nu   \textstyle{\sum}_{k=1}^K  R_k \left( \alphak + \hk \right)
\\& \qquad 
 -\nu \tfrac1K \textstyle{\sum}_{k=1}^K f(\alphav) - \nu \textstyle{\sum}_{k=1}^K \< \nabla f(\alphav), \hk> -\nu  \sigma' \tfrac{1}{2 \lambda n^2} \hv^T G \hv
\\
&\overset{\eqref{eq:subproblem:sigma1}}{=}
(1-\nu) D(\alphav)  
+ \nu \tfrac1K\Gks(\hk; \alphav),
\end{align*}
where the first inequality follows from Jensen and the last equality also follows from the block diagonal definition of $\nBG$ given in \eqref{eq:nBGDefinition}, i.e.
\begin{equation}
\label{eq:nBG:XTX:relation}
\hv^T \nBG \hv = \textstyle{\sum}_{k = 1}^ K  \hk \Xk^T \Xk \hk.
\end{equation}

\subsection{Proof of Lemma \ref{lem:sigmaPrimeNotBad}}
Considering $\hv\in\R^n$ with zeros in all coordinates except those that belong to the $k$-th block $\mathcal P_k$, we have $\hv^T X^TX \hv = \hv^T \nBG \hv$, and thus $\sigma' \geq \nu$.
Let $\vsub{\hv}{k, l}$ denote $\hk - \vsub{\hv}{l}$. Since $X^TX$ %
is a positive semi-definite matrix, for $k, l \in \{ 1, \dots, K \}, k \neq l$ we have 
\begin{equation}
\label{eq:hklhelper}
0 \leq \vsub{\hv}{k, l}^T X^TX \vsub{\hv}{k, l} = \hk^T X^TX \hk + \vsub{\hv}{l}^T X^TX \vsub{\hv}{l} - 2 \hk^T X^TX \vsub{\hv}{l}.
\end{equation}
By taking any $\hv \in \R^n$ for which $ \hv^T \nBG \hv \leq 1$, in view of \eqref{eq:sigmaPrimeSafeDefinition}, we get
\begin{align*}
\hv^T X^TX \hv &= \sum_{k=1}^K \sum_{l=1}^K \hv^T_{[k]} X^TX \hv^T_{[l]}
= \sum_{k=1}^K \hk^T X^TX \hk^T + \sum_{k\neq l} \hv^T_{[k]} X^TX \hv^T_{[l]}
\\
&\overset{\eqref{eq:hklhelper}}{\leq} \sum_{k=1}^K \hk^T X^TX \hk^T + \sum_{k \neq l} \frac12 \left[ \hk^T X^TX \hk + \vsub{\hv}{l}^T X^TX \vsub{\hv}{l} \right]
\\
&= K \sum_{k=1}^K \hk^T X^TX \hk = K \hv^T \nBG \hv \leq K.
\end{align*}
Therefore we can conclude that $\nu \hv^TX^TX\hv \leq \nu K$ for all $\hv$ included in the definition~\eqref{eq:sigmaPrimeSafeDefinition} of $\sigma'_{\min}$, proving the claim.

\subsection{Proofs of Theorems~\ref{thm:mainResult}~and~\ref{thm:mainResult:gcc}}
\label{sec:analysis}

Before we state the proofs of the main theorems, we will write and prove few crucial lemmas.

\begin{lemma}
\label{lem:basicLemma}
Let $\ell_i^*$ be strongly\footnote{%
Note that the case of weakly convex $\ell_i^*(.)$ is explicitly allowed here as well, as the Lemma holds for the case $\gamma = 0$.
} %
convex with convexity parameter
$\gamma \geq 0$ with respect to the norm $\|\cdot\|$, $\forall i\in[n]$.
Then for all iterations~$t$ of Algorithm~\ref{alg:cocoa} under Assumption~\ref{ass:localImprovement}, and any $s\in [0,1]$, it holds that
\begin{align}
\label{eq:lemma:dualDecrease_VS_dualityGap}
&\Exp[
\bD(\vc{\alphav}{t+1})
-
\bD(\vc{\alphav}{t})
 ]
\geq
\aggpar
(1-\Theta)
 \Big(
 s \gap(\vc{\alphav}{t})
-
\frac{\sigma' s^2}{2\lambda n^2}
\vc{R}{t}
 \Big), \vspace{-2mm}
\end{align}
where\vspace{-2mm}
\begin{align*}
\tagthis \label{eq:defOfR}
\vc{R}{t}&:=
-
\tfrac{ \lambda\gamma n (1-s)}{\sigma' s }
 \|\vc{\uv}{t}-\vc{\alphav}{t}\|^2 
+ 
\textstyle{\sum}_{k=1}^K   
  \| X \vsubset{  (\vc{\uv}{t}  - \vc{\alphav}{t} )}{k}\|^2,
\end{align*}
for $\vc{\uv}{t} \in\R^n$
with 
\begin{equation}
\label{eq:defintionOfUi}
-\vc{u_i}{t}
 \in \partial \ell_i(\wv(\vc{\alphav}{t})^T \xv_i).
\end{equation}
\end{lemma}
\begin{proof}
For sake of notation, 
we will write 
$\alphav$ instead of $\vc{\alphav}{t}$,
$\wv$ instead of $\wv(\vc{\alphav}{t})$
and
$\uv$ instead of $\vc{\uv}{t}$.

Now, let us estimate the expected change of the dual objective. 
Using the definition of the dual update $\vc{\alphav}{t+1} := \vc{\alphav}{t} + \aggpar \, \sum_k \hk$ resulting in Algorithm~\ref{alg:cocoaPractical}, we have
\begin{align*}
&\Exp\big[\bD(\vc{\alphav}{t})
 - \bD(\vc{\alphav}{t+1})\big]
 =
\Exp\Big[\bD(\alphav)
 - \bD(\alphav +
  \aggpar \sum_{k=1}^K
  \hk)\Big]
\\ %
&\overset{\eqref{eq:asfdjalkfjlsaflasdfa}}{\leq}
\Exp\Big[\bD(\alphav)
-(1-\aggpar)\bD(\alphav)
-\aggpar 
 \sum_{k=1}^K 
 \Ggk (\hk^t; \alphav)
\Big]\\
&=
\aggpar
\Exp\Big[
 \bD(\alphav)
- 
 \sum_{k=1}^K 
 \Ggk (\hk^t; \alphav)
\Big]
\\
&
=
\aggpar
\Exp\Big[
 \bD(\alphav)
 -
 \sum_{k=1}^K 
 \Ggk(\hk^\star; \alphav)
 +
 \sum_{k=1}^K 
 \Ggk(\hk^\star; \alphav)
- 
 \sum_{k=1}^K 
 \Ggk (\hk^t; \alphav)
\Big]
\\ 
&\overset{\eqref{eq:localQualityOfImprovement}}{\leq}
\aggpar
\bigg(
 \bD(\alphav)
 -
 \sum_{k=1}^K 
 \Ggk(\hk^\star; \alphav)
 +
 \Theta
 \Big(
 \sum_{k=1}^K  
 \Ggk(\hk^\star; \alphav)
 -
\underbrace{  \sum_{k=1}^K  
 \Ggk({\bf 0}; \alphav)
 }_{\bD(\alphav)}
 \Big)
\bigg)
\\
&=
\aggpar
(1-\Theta)
\Big(
\underbrace{
 \bD(\alphav)
 -
 \sum_{k=1}^K 
 \Ggk(\hk^\star; \alphav)
 }_{C}
\Big).
\tagthis
\label{eq:Afasfwafewaef}
\end{align*} 
Now, let us upper bound 
the $C$ term 
(we will denote by
$\hv^\star 
 = \sum_{k=1}^K \hk^\star$):
\begin{align*}
C&
\overset{\eqref{eq:dual},
\eqref{eq:subproblem:sigma1}}{=}
   \frac1n 
 \sum_{i =1}^n 
 \left(
\ell_i^*(-\alpha_i - h^*_i)
-\ell_i^*(- \alpha_i)
\right)
 +\< \nabla f(\alphav), \hv >
 + \sum_{k=1}^K 
\frac\lambda2
 \sigma'   \Big\|\frac1{\lambda n} X \hk^\star\Big\|^2
\\
&\leq  
   \frac1n 
 \sum_{i =1}^n 
 \left(
\ell_i^*(-\alpha_i - s (u_i - \alpha_i))
-\ell_i^*(- \alpha_i)
\right)
 +
\< \nabla f(\alphav), s (\uv  - \alphav ) > 
 \\ &\qquad + \sum_{k=1}^K 
\frac\lambda2
 \sigma'   \Big\|\frac1{\lambda n}X\vsubset{s (\uv  - \alphav )}{k}\Big\|^2
\\
&\overset{\mbox{Strong conv.}}{\leq} 
   \frac1n 
 \sum_{i =1}^n 
 \left(
s \ell_i^*(-u_i )
+
(1-s)
\ell_i^*(-\alpha_i )
-
\frac{\gamma}{2}
(1-s)s (u_i -\alpha_i)^2
-\ell_i^*(- \alpha_i)
\right)
\\&\quad\quad\quad\quad\quad   +
\< \nabla f(\alphav), s (\uv  - \alphav ) > 
 + \sum_{k=1}^K 
\frac\lambda2
 \sigma'   \Big\|\frac1{\lambda n} X \vsubset{s (\uv  - \alphav )}{k}\Big\|^2 
\\
&=
   \frac1n 
 \sum_{i =1}^n 
 \left(
s \ell_i^*(-u_i )
  -s 
\ell_i^*(-\alpha_i )
-
\frac{\gamma}{2}
(1-s)s (u_i -\alpha_i)^2
\right)
\\&\qquad 
+\< \nabla f(\alphav), s (\uv  - \alphav ) > 
  + \sum_{k=1}^K 
\frac\lambda2
 \sigma'   \Big\|\frac1{\lambda n} X \vsubset{s (\uv  - \alphav )}{k}\Big\|^2.  
\end{align*}
The convex conjugate maximal property implies that
\begin{equation}
\label{eq:adjwofcewa}
\ell_i^*(-u_i)
\overset{\eqref{eq:defintionOfUi}}{=} -u_i \wv(\alphav)^T \xv_i
  -\ell_i(\wv(\alphav)^T \xv_i).
\end{equation}
Moreover, from the definition of the primal and dual optimization problems \eqref{eq:primal},
\eqref{eq:dual}, we can write the duality gap as
\begin{align}
\notag
\gap(\alphav) &:= \bP(\wv(\alphav))-\bD(\alphav)
\\
\notag
&\overset{
\eqref{eq:primal},
\eqref{eq:dual}
}{=}
 \tfrac1{n} 
 \sum_{i=1}^n
 \left(
  \ell_i( \xv_i^T \wv(\alphav)) 
 +  \ell_i^*(- \alpha_i)
  \right)
    +\lambda g(\wv(\alphav)) + \lambda g^*(\vv(\alphav))   
\\
\notag
&= \tfrac1{n} 
 \sum_{i=1}^n
 \left(
  \ell_i( \xv_i^T \wv(\alphav)) 
 +  \ell_i^*(- \alpha_i)
  \right)
    +\lambda g(\nabla g^*(\vv(\alphav))) + \lambda g^*(\vv(\alphav))  
\\
\notag
&= \tfrac1{n} 
 \sum_{i=1}^n
 \left(
  \ell_i( \xv_i^T \wv(\alphav)) 
 +  \ell_i^*(- \alpha_i)
  \right)
    +\lambda \vv(\alphav)^T \wv(\alphav)
\\
\label{eq:asdfjiwjfeojawfa}
&= \tfrac1{n} 
 \sum_{i=1}^n
 \left(
  \ell_i( \xv_i^T \wv(\alphav)) 
 +  \ell_i^*(- \alpha_i)
 +  \wv(\alphav)^T\xv_i \alpha_i
  \right).    
\end{align}
Hence,
\begin{align*}
C
&\overset{
\eqref{eq:adjwofcewa}}
{\leq}
  \frac1n 
 \sum_{i =1}^n 
 \left( 
-s u_i \wv(\alphav)^T \xv_i
  -s\ell_i(\wv(\alphav)^T \xv_i)
  -s 
\ell_i^*(-\alpha_i )
\underbrace{-s \wv(\alphav)^T \xv_i \alpha_i
+s \wv(\alphav)^T \xv_i \alpha_i
}_{0}
\right)
\\&\qquad 
+ \frac1n \sum_{i=1}^n \left(-
\frac{\gamma}{2}
(1-s)s (u_i -\alpha_i)^2
\right)
 +\< \nabla f(\alphav), s (\uv  - \alphav ) > 
 + \sum_{k=1}^K 
\frac\lambda2
 \sigma'   \Big\|\frac1{\lambda n} X \vsubset{s (\uv  - \alphav )}{k}\Big\|^2 
\\
&=
  \frac1n 
 \sum_{i =1}^n 
 \left( 
  -s\ell_i(\wv(\alphav)^T \xv_i)
  -s\ell_i^*(-\alpha_i )
  -s \wv(\alphav)^T \xv_i \alpha_i
\right)
\\ &\qquad +
  \frac1n 
 \sum_{i =1}^n 
 \left(  s \wv(\alphav)^T \xv_i
( \alpha_i-u_i )
 -
\frac{\gamma}{2}
(1-s)s (u_i -\alpha_i)^2
\right)
\\&\qquad  +\frac1n  
\wv(\alphav)^T X  s (\uv  - \alphav )
 + \sum_{k=1}^K 
\frac\lambda2
 \sigma'   \Big\|\frac1{\lambda n} X \vsubset{s (\uv  - \alphav )}{k}\Big\|^2  
\\
&\overset{\eqref{eq:asdfjiwjfeojawfa}}{=}
 -s \gap(\alphav)
-
\frac{\gamma}{2}
(1-s)s 
  \frac1n 
 \sum_{i =1}^n 
 \|\uv-\alphav\|^2 
 + 
\frac{\sigma' s^2}{2\lambda n^2}
\sum_{k=1}^K   
  \| X \vsubset{  (\uv  - \alphav )}{k}\|^2.
  \tagthis 
  \label{eq:asdfafdas}
 \end{align*}
Now, the claimed improvement bound
\eqref{eq:lemma:dualDecrease_VS_dualityGap}
follows
by plugging 
\eqref{eq:asdfafdas}
into \eqref{eq:Afasfwafewaef}.
\end{proof}

\begin{lemma}
\label{lemma:BoundOnR}
If $\ell_i$ are $L$-Lipschitz 
continuous for all $i\in [n]$, then\vspace{-3mm}
\begin{equation}
\label{eq:asfjoewjofa}
\forall t: 
\vc{R}{t}
\leq 4L^2 
\underbrace{\sum _{k=1}^K 
\sigma_k  |\mathcal{P}_k|}_{=: \sigma}, \vspace{-2mm} %
\end{equation}
where\vspace{-1mm}
\begin{equation}
\label{eq:definitionOfSigmaK}
\sigma_k \eqdef
 \max_{\vsubset{\alphav}{k} \in \R^n}
 \frac{\|\Xk \vsubset{\alphav}{k}\|^2}{
 \|\vsubset{\alphav}{k}\|^2}.
\end{equation}
\end{lemma}
\begin{proof}
For general convex functions, the strong convexity parameter is 
$\gamma=0$, and hence the definition of $\vc{R}{t}$ becomes
\begin{align*} 
\vc{R}{t}
\overset{\eqref{eq:defOfR}}{=}
  \sum _{k=1}^K   
  \| X \vsubset{  (\vc{\uv} {t} - \vc{\alphav}{t} )}{k}\|^2
\overset{\eqref{eq:definitionOfSigmaK}}{\leq}   
\sum _{k=1}^K 
\sigma_k  
  \|   \vsubset{  (\vc{\uv} {t} - \vc{\alphav}{t} )}{k}\|^2
\overset{\mbox{Lemma 3 in \cite{ShalevShwartz:2013wl}}}{\leq}   
\sum _{k=1}^K 
\sigma_k  |\mathcal{P}_k| 4L^2. 
\end{align*}
\end{proof}

\subsubsection{Proof of Theorem \ref{thm:convergenceNonsmooth}}

At first let us estimate expected change of dual feasibility. By using the main Lemma \ref{lem:basicLemma}, we have
\begin{align*} 
 &\Exp[\bD(\alphav^\star)-\bD(\vc{\alphav}{t+1})]
 =
\Exp[\bD(\alphav^\star)-\bD(\vc{\alphav}{t+1})+\bD(\vc{\alphav}{t})-\bD(\vc{\alphav}{t})]
\\
&
\overset{\eqref{eq:lemma:dualDecrease_VS_dualityGap}
}{=}
\bD(\alphav^\star)-\bD(\vc{\alphav}{t})
-\aggpar
(1-\Theta)  
 s \gap(\vc{\alphav}{t})
+
\aggpar
(1-\Theta)
\tfrac{\sigma'}{2\lambda }
(\frac s{  n})^2
\vc{R}{t}
\\
&
\overset{\eqref{eq:gap}
}{=}
\bD(\alphav^\star)-\bD(\vc{\alphav}{t})
-\aggpar
(1-\Theta)
   s  (\bP(\wv(\vc{\alphav}{t}))-\bD(\vc{\alphav}{t}))
+
\aggpar
(1-\Theta)  \tfrac{\sigma'}{2\lambda }
(\frac s{  n})^2
\vc{R}{t} 
\\
&\leq
\bD(\alphav^\star)-\bD(\vc{\alphav}{t})
-\aggpar
(1-\Theta)
 s  (\bD(\alphav^\star )-\bD(\vc{\alphav}{t}) )
+
\aggpar
(1-\Theta) 
\tfrac{\sigma'}{2\lambda }
(\frac s{  n})^2
\vc{R}{t} \\
&
\overset{\eqref{eq:asfjoewjofa}}{\leq} 
\left( 
 1-\aggpar
(1-\Theta)
   s
\right) 
   (\bD(\alphav^\star )-\bD(\vc{\alphav}{t}))
+
\aggpar
(1-\Theta) 
\tfrac{\sigma'}{2\lambda }
(\frac s{  n})^2
4L^2  \sigma.
\tagthis 
\label{eq:asoifejwofa}
\end{align*}
 Using
\eqref{eq:asoifejwofa}
recursively we have 
 \begin{align*} 
 &\Exp[\bD(\alphav^\star)-\bD(\vc{\alphav}{t})]
 =
\left( 
 1-\aggpar
(1-\Theta)
   s
\right)^t 
   (\bD(\alphav^\star )-\bD(\vc{\alphav}{0}))
\\&\qquad \qquad\qquad \qquad +
\aggpar
(1-\Theta) 
\tfrac{\sigma'}{2\lambda }
(\frac s{  n})^2
4L^2  \sigma 
\sum_{j=0}^{t-1}
\left( 
 1-\aggpar
(1-\Theta)
   s
\right)^j
\\
&=
\left( 
 1-\aggpar
(1-\Theta)
   s
\right)^t 
   (\bD(\alphav^\star )-\bD(\vc{\alphav}{0}))
+
\aggpar
(1-\Theta) 
\tfrac{\sigma'}{2\lambda }
(\frac s{  n})^2
4L^2  \sigma 
\frac{1-\left( 
 1-\aggpar
(1-\Theta)
   s
\right)^t}
     { 
  \aggpar
(1-\Theta)
   s }
\\
&\leq
\left( 
 1-\aggpar
(1-\Theta)
   s
\right)^t 
   (\bD(\alphav^\star )-\bD(\vc{\alphav}{0}))
+
 s
\frac{4L^2  \sigma   \sigma'}{2\lambda n^2}. 
\tagthis
\label{eq:asfwefcaw}  
 \end{align*}
The choice of 
$s:=1$ and $t= t_0:= \max\{0,\lceil  
\frac1{\aggpar (1-\Theta)}
\log(
 2\lambda n^2 (\bD(\alphav^\star )-\bD(\vc{\alphav}{0}))
  / (4 L^2 \sigma \sigma')
  )
 \rceil\}$
will lead to 
\begin{align}\label{eq:induction_step1}
  \Exp[\bD(\alphav^\star)-\bD(\vc{\alphav}{t})]
 &\leq  
\left( 
 1-\aggpar
(1-\Theta)  
\right)^{t_0}
  (\bD(\alphav^\star )-\bD(\vc{\alphav}{0}))
+ 
\frac{4L^2  \sigma   \sigma'}{2\lambda n^2}
\notag 
\\&\leq 
\frac{4L^2  \sigma   \sigma'}{2\lambda n^2}
+
\frac{4L^2  \sigma   \sigma'}{2\lambda n^2}
=
\frac{4L^2  \sigma   \sigma'}{\lambda n^2}.
\end{align} 
Now, we are going to show that 
\begin{align}
\label{eq:expectationOfDualFeasibility}
\forall t\geq t_0 :  \Exp[\bD(\alphav^\star )-\bD(\vc{\alphav}{t})]
&\leq 
\frac{4L^2  \sigma   \sigma'}{\lambda n^2( 1+ \frac12  \aggpar (1-\Theta)  (t-t_0))}.
\end{align}
Clearly, \eqref{eq:induction_step1} implies that \eqref{eq:expectationOfDualFeasibility} holds for $t=t_0$.
Now imagine that it holds for any $t\geq t_0$ then we show that it also has to hold for $t+1$. 
Indeed, using 
\begin{equation}
\label{eq:asdfjoawjdfas}
s=
\frac{1}
 {1+ \frac12 \aggpar (1-\Theta) (t-t_0)} \in [0,1]
\end{equation} 
  we obtain
\begin{align*}
&\Exp[
\bD(\alphav^\star )-\bD(\vc{\alphav}{t+1})]
\overset{\eqref{eq:asoifejwofa}
}{\leq}
\left( 
 1-\aggpar
(1-\Theta)
   s
\right) 
   (\bD(\alphav^\star )-\bD(\vc{\alphav}{t}))
+
\aggpar
(1-\Theta) 
\tfrac{\sigma'}{2\lambda }
(\frac s{  n})^2
4L^2  \sigma
\\
&\overset{\eqref{eq:expectationOfDualFeasibility}
}{\leq}
\left( 
 1-\aggpar
(1-\Theta)
   s
\right) 
   \frac{4L^2  \sigma   \sigma'}{\lambda n^2( 1+ \frac12  \aggpar (1-\Theta)  (t-t_0))}
+
\aggpar
(1-\Theta) 
\tfrac{\sigma'}{2\lambda }
(\frac s{  n})^2
4L^2  \sigma
\\
&
\overset{\eqref{eq:asdfjoawjdfas}}{=}
\frac{4L^2  \sigma   \sigma'}
     {\lambda n^2}
\left( 
\frac{
1+ \frac12 \aggpar (1-\Theta) (t-t_0)
-\aggpar
(1-\Theta)
+
\aggpar
(1-\Theta) 
\tfrac{1}{2}
}
 {(1+ \frac12 \aggpar (1-\Theta) (t-t_0))^2}
\right)
\\
&=
\frac{4L^2  \sigma   \sigma'}
     {\lambda n^2}
\underbrace{\left( 
\frac{
1+ \frac12 \aggpar (1-\Theta) (t-t_0)
-\frac12 \aggpar
(1-\Theta)
}
 {(1+ \frac12 \aggpar (1-\Theta) (t-t_0))^2}
\right)}_{D}.
\end{align*}
Now, we will upperbound $D$ as follows
\begin{align*}
D&=
\tfrac1
{1+ \frac12 \aggpar (1-\Theta) (t+1-t_0)}
\underbrace{
\tfrac{
(1+ \frac12 \aggpar (1-\Theta) (t+1-t_0))
(1+ \frac12 \aggpar (1-\Theta) (t-1-t_0))
}
 {(1+ \frac12 \aggpar (1-\Theta) (t-t_0))^2}}_{\leq 1}
 \\
&\leq  
\tfrac1
{1+ \frac12 \aggpar (1-\Theta) (t+1-t_0)},
\end{align*}
where in the last inequality we have used the fact that geometric mean
 is less or equal to arithmetic mean. 
 
If $\overline \alphav$ is defined as \eqref{eq:averageOfAlphaDefinition}
then we obtain that
\begin{align*}
\Exp[\gap(\overline\alphav)] &=  
 \Exp\left[\gap\left(\sum_{t=T_0}^{T-1} \tfrac1{T-T_0} \vc{\alphav}{t}\right)\right]
 \leq
  \tfrac1{T-T_0} \Exp\left[\sum_{t=T_0}^{T-1} \gap\left( \vc{\alphav}{t}\right)\right]
\\
&
\overset{
\eqref{eq:lemma:dualDecrease_VS_dualityGap}
,\eqref{eq:asfjoewjofa}
}{\leq}
  \tfrac1{T-T_0} \Exp\left[\sum_{t=T_0}^{T-1} 
\left(
\frac1{\aggpar
(1-\Theta)
 s}
(
\bD(\vc{\alphav}{t+1})
-
\bD(\vc{\alphav}{t})
 )
 +
\tfrac{4L^2 \sigma \sigma' s}{2\lambda n^2 }
\right)  
  \right]
\\  
 &=
\frac1{\aggpar
(1-\Theta)
 s}
   \frac1{T-T_0} 
   \Exp\left[
\bD(\vc{\alphav}{T})
-
\bD(\vc{\alphav}{T_0})
  \right] 
+\tfrac{4L^2 \sigma \sigma' s}{2\lambda n^2 }  
\\  
 &\leq
\frac1{\aggpar
(1-\Theta)
 s}
   \frac1{T-T_0} 
   \Exp\left[
\bD(\alphav^\star)
-
\bD(\vc{\alphav}{T_0})
  \right] 
+\tfrac{4L^2 \sigma \sigma' s}{2\lambda n^2 }.  
\tagthis \label{eq:askjfdsanlfas}
  \end{align*}
Now, if $T\geq \lceil
\frac1{\aggpar (1-\Theta)}\rceil+T_0$ such that $T_0\geq t_0$
we obtain
\begin{align*}
\Exp[\gap(\overline\alphav)] 
&\overset{\eqref{eq:askjfdsanlfas}
,\eqref{eq:expectationOfDualFeasibility}
}{\leq}
\frac1{\aggpar
(1-\Theta)
 s}
   \frac1{T-T_0} 
\left(
\frac{4L^2  \sigma   \sigma'}{\lambda n^2( 1+ \frac12  \aggpar (1-\Theta)  (T_0-t_0))}
\right)
+\frac{4L^2 \sigma \sigma' s}{2\lambda n^2 }
\\
&=
\frac{
4L^2  \sigma   \sigma'}{\lambda n^2}
\left(
\frac1{\aggpar
(1-\Theta)
 s}
   \frac1{T-T_0} 
\frac{1}{ 1+ \frac12  \aggpar (1-\Theta)  (T_0-t_0)}
+\frac{  s}{2 }
\right). 
\tagthis
\label{eq:fawefwafewa}
\end{align*}
Choosing 
\begin{equation}
\label{eq:afskoijewofaw}
s=\frac{1}{(T-T_0) \aggpar (1-\Theta)} \in [0,1]
\end{equation}
gives us
\begin{align*}
\Exp[\gap(\overline\alphav)] 
&
\overset{\eqref{eq:fawefwafewa},
\eqref{eq:afskoijewofaw}}{\leq}
\frac{
4L^2  \sigma   \sigma'}{\lambda n^2}
\left(
\frac{1}{ 1+ \frac12  \aggpar (1-\Theta)  (T_0-t_0)}
+\frac{1}{(T-T_0) \aggpar (1-\Theta)} \frac{  1}{2 }
\right). \tagthis
\label{eq:afsjweofjwafea}
\end{align*}
To have right hand side of
\eqref{eq:afsjweofjwafea}
smaller then 
$\epsilon_\gap$
it is sufficient to choose
$T_0$ and $T$ such that
\begin{eqnarray}
\label{eq:sfadwafeewafa}
\frac{4L^2  \sigma   \sigma'}{\lambda n^2}
\left(
\frac{1}{ 1+ \frac12  \aggpar (1-\Theta)  (T_0-t_0)}
\right)
&\leq & \frac12 \epsilon_\gap,
\\
\label{eq:sfadwafeewafa2}
\frac{4L^2  \sigma   \sigma'}{\lambda n^2}
\left(
\frac{1}{(T-T_0) \aggpar (1-\Theta)} \frac{  1}{2 }
\right)
&\leq & \frac12 \epsilon_\gap.
\end{eqnarray}
Hence of 
if
\begin{eqnarray*}
t_0+
\frac{2}{ \aggpar (1-\Theta) }
\left(
\frac
{8L^2  \sigma   \sigma'}
{\lambda n^2 \epsilon_\gap}
-1
\right)
&\leq & 
 T_0 
,
\\
T_0
+
\frac
{4L^2  \sigma   \sigma'}
{\lambda n^2 \epsilon_\gap
\aggpar (1-\Theta)}
&\leq &  T,  
\end{eqnarray*}
then 
\eqref{eq:sfadwafeewafa}
and
\eqref{eq:sfadwafeewafa2}
are satisfied.

\subsubsection{Proof of Theorem \ref{thm:convergenceSmoothCase}
}
If the function $\ell_i(.)$ is $(1/\gamma)$-smooth then $\ell_i^*(.)$ is $\gamma$-strongly convex with respect to the
$\|\cdot\|$ norm.
From \eqref{eq:defOfR}
we have
\begin{align*}
\vc{R}{t}&
\overset{\eqref{eq:defOfR}}{=}
-
\tfrac{ \lambda\gamma n (1-s)}{\sigma' s }
   \|\vc{\uv}{t}-\vc{\alphav}{t}\|^2 
+ 
 {\sum}_{k=1}^K   
  \| X \vsubset{  (\vc{\uv}{t}  - \vc{\alphav}{t} )}{k}\|^2
\\%
&
\overset{\eqref{eq:definitionOfSigmaK}}{\leq}  
-
\tfrac{ \lambda\gamma n (1-s)}{\sigma' s }
   \|\vc{\uv}{t}-\vc{\alphav}{t}\|^2 
+ 
 {\sum}_{k=1}^K   
 \sigma_k
  \|  \vsubset{   (\vc{\uv}{t}  - \vc{\alphav}{t}  )}{k}\|^2
\\
&\leq
-
\tfrac{ \lambda\gamma n (1-s)}{\sigma' s }
   \|\vc{\uv}{t}-\vc{\alphav}{t}\|^2 
+
\sigma_{\max} 
 {\sum}_{k=1}^K   
  \|  \vsubset{ (  \vc{\uv}{t}  - \vc{\alphav}{t}  )}{k}\|^2
\\
&=
\left(
-
\tfrac{ \lambda\gamma n (1-s)}{\sigma' s }
+\sigma_{\max}
\right)
   \|\vc{\uv}{t}-\vc{\alphav}{t}\|^2.\tagthis
   \label{eq:afjfocjwfcea} 
\end{align*}
 If we plug 
 \begin{equation}
 s=
  \frac{ \lambda\gamma n }
      {\lambda\gamma n+
\sigma_{\max} \sigma'}\in [0,1]
\label{eq:fajoejfojew}
\end{equation} 
into
\eqref{eq:afjfocjwfcea}
we obtain that
$\forall t: \vc{R}{t}\leq 0$.
Putting the  same $s$
into
\eqref{eq:lemma:dualDecrease_VS_dualityGap}
will give us
\begin{align*}
\Exp[
\bD(\vc{\alphav}{t+1})
-
\bD(\vc{\alphav}{t})
 ]
&\overset{\eqref{eq:lemma:dualDecrease_VS_dualityGap}
,\eqref{eq:fajoejfojew}}{\geq}
\aggpar
(1-\Theta)
 \frac{ \lambda\gamma n }
      {\lambda\gamma n+
\sigma_{\max} \sigma'} \gap(\vc{\alphav}{t})
\notag
\\&\geq
\aggpar
(1-\Theta)
 \frac{ \lambda\gamma n }
      {\lambda\gamma n+
\sigma_{\max} \sigma'} \bD(\alphav^\star)-\bD(\vc{\alphav}{t}).
\tagthis
\label{eq:fasfawfwaf}
\end{align*}
Using the fact that
$\Exp[\bD(\vc{\alphav}{t+1})-\bD(\vc{\alphav}{t})]
=\Exp[\bD(\vc{\alphav}{t+1})-\bD(\alphav^\star)]
+\bD(\alphav^\star)-\bD(\vc{\alphav}{t})
$
we have 
\begin{align*}
\Exp[\bD(\vc{\alphav}{t+1})-\bD(\alphav^\star)]
+\bD(\alphav^\star)-\bD(\vc{\alphav}{t})
\overset{
\eqref{eq:fasfawfwaf}}
{
\geq
}
\aggpar
(1-\Theta)
 \frac{ \lambda\gamma n }
      {\lambda\gamma n+
\sigma_{\max} \sigma'} \bD(\alphav^\star)-\bD(\vc{\alphav}{t})
\end{align*}
which is equivalent with
\begin{align*}
\Exp[\bD(\alphav^\star)-\bD(\vc{\alphav}{t+1})]
\leq 
\left(
1-\aggpar
(1-\Theta)
 \frac{ \lambda\gamma n }
      {\lambda\gamma n+
\sigma_{\max} \sigma'}\right)
\bD(\alphav^\star)-\bD(\vc{\alphav}{t}).
\tagthis \label{eq:affpja}
\end{align*}
Therefore if we denote by $\vc{\epsilon_\bD}{t} = \bD(\alphav^\star)-\bD(\vc{\alphav}{t})$
we have that
\begin{align*}
 \Exp[\vc{\epsilon_\bD}{t}] 
 &\overset{\eqref{eq:affpja}}{\leq}   \left(
 1-\aggpar
(1-\Theta)
 \frac{ \lambda\gamma n }
      {\lambda\gamma n+
\sigma_{\max} \sigma'}
   \right)^t \vc{\epsilon_\bD}{0}
\leq 
\left(
 1-\aggpar
(1-\Theta)
 \frac{ \lambda\gamma n }
      {\lambda\gamma n+
\sigma_{\max} \sigma'}
   \right)^t
\\&\leq \exp\left(-t \aggpar
(1-\Theta)
 \frac{ \lambda\gamma n }
      {\lambda\gamma n+
\sigma_{\max} \sigma'}
     \right).
\end{align*}
The right hand side will be smaller than some $\epsilon_\bD$ if 
$$
 t   
    \geq 
\frac{1}
   {\aggpar
(1-\Theta)}
\frac
{\lambda\gamma n+
\sigma_{\max} \sigma'}
{ \lambda\gamma n }
    \log \frac1{\epsilon_\bD}.
$$
Moreover, to bound the duality gap, we have
\begin{align*}
\aggpar
(1-\Theta)
 \frac{ \lambda\gamma n }
      {\lambda\gamma n+
\sigma_{\max} \sigma'} \gap(\vc{\alphav}{t})
&
\overset{
\eqref{eq:fasfawfwaf}
}{\leq}
\Exp[
\bD(\vc{\alphav}{t+1})
-
\bD(\vc{\alphav}{t})
 ]
\leq 
\Exp[
\bD(\alphav^\star)
-
\bD(\vc{\alphav}{t})
 ]. 
\end{align*}
Therefore  $\gap(\vc{\alphav}{t})\leq 
\frac1{
\aggpar
(1-\Theta)}
 \frac      {\lambda\gamma n+
\sigma_{\max} \sigma'} 
{ \lambda\gamma n }    \vc{\epsilon_\bD}{t}$.  
Hence if $\epsilon_\bD \leq 
\aggpar
(1-\Theta)
 \frac{ \lambda\gamma n }
      {\lambda\gamma n+
\sigma_{\max} \sigma'} 
 \epsilon_\gap $
then $\gap(\vc{\alphav}{t})\leq \epsilon_\gap$.
Therefore
after 
$$
 t   
    \geq 
\frac{1}
   {\aggpar
(1-\Theta)}
\frac
{\lambda\gamma n+
\sigma_{\max} \sigma'}
{ \lambda\gamma n }
    \log 
\left(
\frac{1}
   {\aggpar
(1-\Theta)}
\frac
{\lambda\gamma n+
\sigma_{\max} \sigma'}
{ \lambda\gamma n }
    \frac1{\epsilon_\gap}
    \right) 
$$
iterations we have obtained a duality gap less than $\epsilon_\gap$.

\end{document}